\documentclass[onefignum,onetabnum]{siamonline171218}

\usepackage{booktabs} % For formal tables
\usepackage{amsmath}
\usepackage{amssymb}
\usepackage{multirow}
\usepackage{multicol}
\usepackage{soul}
\usepackage{amsfonts}
\usepackage{bm}
\usepackage{xcolor}
\usepackage[toc,page,header]{appendix}
\usepackage{thmtools}
\usepackage{thm-restate}
\usepackage{mathtools}
\usepackage[normalem]{ulem}

\usepackage{wrapfig}

\newtheorem{remark}{Remark}

\def\sigmamin{\sigma_{\text{min}}}

\NewDocumentCommand{\TT}{s m}{\textcolor{cyan}{\IfBooleanF{#1}{\textbf{Taos:}~}#2}}

\NewDocumentCommand{\MathVariable}{o m}{%
    {\IfValueTF{#1}{#1{#2}}{#2}}%
}

\NewDocumentCommand{\RefUpToThree}{m D<>{} D<>{} o o o}{%
    #1~#2\ref{#4}#3%
    \IfValueT{#5}{%
        \IfValueTF{#6}{%
            , #2\ref{#5}#3, and #2\ref{#6}#3%
        }{%
            and #2\ref{#5}#3%
        }%
    }%
}
\NewDocumentCommand{\RefEq}{s o o o}{%
    \RefUpToThree{\IfBooleanTF{#1}{Equation}{equation}}<><>[#2][#3][#4]%
}
\NewDocumentCommand{\RefAppendix}{s o o o}{%
    \RefUpToThree{\IfBooleanTF{#1}{Appendix}{appendix}}[#2][#3][#4]%
}
\NewDocumentCommand{\RefSec}{s o o o}{%
    \RefUpToThree{\IfBooleanTF{#1}{Section}{Section}}[#2][#3][#4]%
}
\NewDocumentCommand{\RefFig}{s o o o}{%
    \RefUpToThree{\IfBooleanTF{#1}{Figure}{Figure}}[#2][#3][#4]%
}
\NewDocumentCommand{\RefTable}{s o o o}{%
    \RefUpToThree{\IfBooleanTF{#1}{Table}{Table}}[#2][#3][#4]%
}

\DeclarePairedDelimiter{\set}{\{}{\}}
\DeclarePairedDelimiter{\abs}{\lvert}{\rvert}
\DeclarePairedDelimiter{\pn}{(}{)}
\DeclarePairedDelimiter{\spn}{[}{]}
\DeclarePairedDelimiter{\hrpn}{(}{]}
\DeclarePairedDelimiter{\norm}{\lVert}{\rVert}

\NewDocumentCommand{\vect}{}{\MathVariable[\bm]}
\NewDocumentCommand{\mat}{}{\MathVariable[\bm]}
\NewDocumentCommand{\dif}{}{\mathrm{d}}
\NewDocumentCommand{\wrt}{m}{\,\dif#1}
\NewDocumentCommand{\diffrac}{O{} m}{\frac{\dif#1}{\dif#2}}
\NewDocumentCommand{\delfrac}{O{} m}{\frac{\partial#1}{\partial#2}}
\NewDocumentCommand{\Gaussian}{}{\mathcal{N}}
\NewDocumentCommand{\Transpose}{}{\mathsf{T}}
\NewDocumentCommand{\Reals}{}{\mathbb{R}}

\def\eqref#1{equation~\ref{#1}}
\def\gL{{\mathcal{L}}}
\def\gN{{\mathcal{N}}}
\def\gO{{\mathcal{O}}}
\def\mC{{\bm{C}}}
\def\mI{{\bm{I}}}
\def\sZ{{\mathbb{Z}}}
\def\vb{{\bm{b}}}
\def\vf{{\bm{f}}}
\def\vg{{\bm{g}}}
\def\vm{{\bm{m}}}
\def\vs{{\bm{s}}}
\def\vu{{\bm{u}}}
\def\vv{{\bm{v}}}
\def\vw{{\bm{w}}}
\def\vx{{\bm{x}}}
\def\vy{{\bm{y}}}
\def\vz{{\bm{z}}}
\newcommand{\E}{\mathbb{E}}
\newcommand{\R}{\mathbb{R}}
\let\argmin\undefined
\DeclareMathOperator*{\argmin}{arg\,min}

\newcommand{\rev}[1]{#1}
\NewDocumentCommand{\EDIT}{m}{#1}

\usepackage{lipsum}
\usepackage{amsfonts}
\usepackage{graphicx}
\usepackage{epstopdf}

\usepackage{algorithm}
\usepackage[noend]{algpseudocode}

\ifpdf
  \DeclareGraphicsExtensions{.eps,.pdf,.png,.jpg}
\else
  \DeclareGraphicsExtensions{.eps}
\fi

\usepackage{enumitem}
\setlist[enumerate]{leftmargin=.5in}
\setlist[itemize]{leftmargin=.5in}

\newsiamremark{hypothesis}{Hypothesis}
\crefname{hypothesis}{Hypothesis}{Hypotheses}
\newsiamthm{claim}{Claim}

\headers{Flow Matching for Efficient and Scalable Data Assimilation}{T. Transue and et al.}

\title{Flow Matching
for Efficient and Scalable Data Assimilation\thanks{Submitted to the editors \today{}.
\funding{This material is based on research sponsored by NSF grants DMS-2208361, DMS-2219956, and DMS-2436344, and DOE grants DE-SC0023490, DE-SC0025589, and DE-SC0025801.}
}}

\author{
Taos Transue*%\thanks{Department of Mathematics and Scientific Computing and Imaging Institute, University of Utah, Salt Lake City, UT, 84112}
\and
Bohan Chen*\thanks{T. Transue and B. Chen are co-first authors.}
\and
So Takao\thanks{B. Chen and S. Takao are with The Computing and Mathematical Sciences Department, California Institute of Technology, 1200 E California Blvd, Pasadena, CA 91125.}
\and
Bao Wang\thanks{T. Transue and B. Wang are with the Department of Mathematics and Scientific Computing and Imaging Institute, University of Utah, Salt Lake City, UT, 84112.}
}

\usepackage{amsopn}

\ifpdf
\hypersetup{
  pdftitle={Flow Matching
  for Efficient and Scalable Data Assimilation},
  pdfauthor={Taos Transue* and Bohan Chen* and So Takao and Bao Wang}
}
\fi

\begin{document}

\maketitle

\begin{abstract}
Data assimilation (DA) estimates a dynamical system's state from noisy observations. Recent generative models like the ensemble score filter (EnSF) improve DA in high-dimensional nonlinear settings but are computationally expensive. We introduce the ensemble flow filter (EnFF), a training-free, flow matching (FM)-based framework that accelerates sampling and offers flexibility in flow design. EnFF uses Monte Carlo estimators for the marginal flow field, localized guidance for observation assimilation, and utilizes a novel flow path that exploits the Bayesian DA formulation.
It generalizes classical filters such as the bootstrap particle filter and ensemble Kalman filter. Experiments on high-dimensional benchmarks demonstrate EnFF's improved cost-accuracy tradeoffs and scalability, highlighting FM's potential for efficient, scalable DA. Code is available at \url{https://github.com/Utah-Math-Data-Science/Data-Assimilation-Flow-Matching}.
\end{abstract}



\begin{wrapfigure}{r}{0.49\textwidth}
\vspace{-1.3cm}
\centering
\resizebox{0.45\textwidth}{!}{%
    \centering
    \begin{tabular}{cc}
        \multicolumn{2}{c}
     {\includegraphics[width=0.99\linewidth]{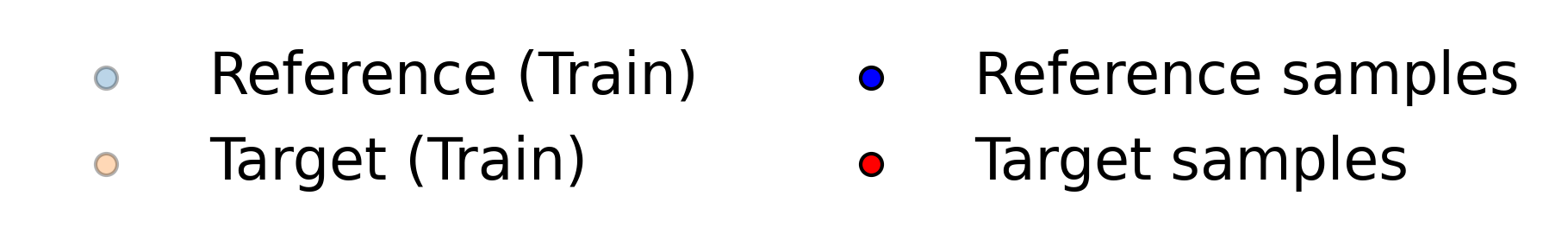}} \\
        \includegraphics[width=0.5\linewidth]{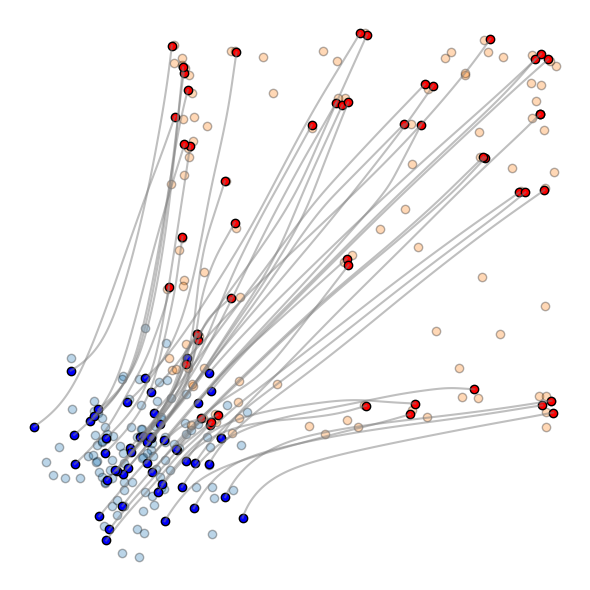}
        &
        \includegraphics[width=0.5\linewidth]{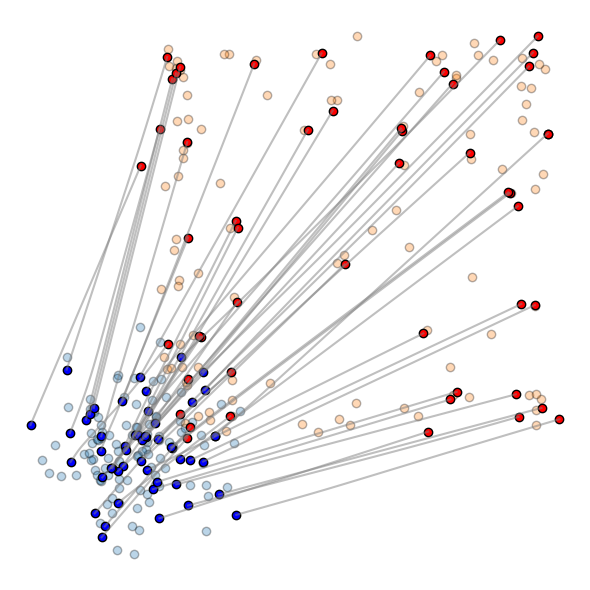} \\ [-2pt]
        \scriptsize (a) OT flow &
        \scriptsize (b) F2P flow (ours)
    \end{tabular}
    }\vspace{-0.3cm}
\caption{
Trajectories of the ODE using (a) OT flow and (b) F2P flow. F2P trajectories are straighter due to coupling between reference and target measures.
    }
    \label{fig:ot-vs-f2p}\vspace{-0.2cm}
\end{wrapfigure}

\section{Introduction}\label{section:introduction}\vspace{-0.3cm}

Data assimilation (DA) estimates \EDIT{a} dynamical system's state from noisy observations \cite{law2015data}, with {\it filtering} referring to sequential inference.
While Bayesian DA is principled, it is often intractable, requiring simplifications.
For example, the Kalman filter (KF) provides a closed-form solution for linear Gaussian systems \cite{Kalman1960new,welch1995introduction}, while the particle filter (PF) offers an asymptotically exact solution via particle approximations \cite{chopin2020introduction}.

DA plays a key role in numerical weather prediction (NWP), but the high dimensionality and nonlinear processes of NWP pose major computational challenges. Modern NWP models have $\mathcal{O}(10^9)$ state dimensions and ingest millions of observations daily, with both dynamics and observation operators being nonlinear. Under such settings, KFs and their nonlinear variants (e.g., extended and unscented KFs \cite{julier2004unscented}) become intractable due to their cubic complexity in dimension, while PFs such as the bootstrap PF (BPF) \cite{gordon1993novel,kitagawa1996monte} suffer from mode collapse.
The ensemble KF (EnKF) assumes a near-Gaussian predictive distribution \cite{calvello2024accuracy,bach2025learning} and reduces KF costs by using a small ensemble for state estimation \cite{burgers1998analysis,evensen2003ensemble}. While effective in NWP \cite{kunii2012estimating,buehner2017ensemble}, small ensembles introduce spurious correlations and underestimated variances, risking instability \cite{bannister2017review}. Localization \cite{hunt2007efficient} and inflation \cite{anderson1999monte} mitigate these issues, but tuning their hyperparameters for robust performance remains challenging, especially in nonlinear settings, due to their computational cost and their hyperparameters' sensitivity.

To address these challenges, generative models (GMs) such as diffusion models (DMs) have gained attention in DA for their ability to represent high-dimensional, non-Gaussian distributions and solve inverse problems with nonlinear forward operators. For example, score-based DA \cite{rozet2023score} uses a pretrained score function network to infer full trajectories from spatiotemporal observations, while the ensemble score filter (EnSF) \cite{bao2024score,bao2024ensemble} estimates score functions online to assimilate observations sequentially via guidance-based methods. EnSF shows improved robustness over traditional approaches when observation operators are nonlinear.

Despite their promise, DMs suffer from slow sampling due to the high cost of solving stochastic differential equations (SDEs), limiting their use in long-trajectory DA. Several methods address this bottleneck: Denoising diffusion implicit models \cite{song2020denoising} accelerate sampling via non-Markovian processes, while \cite{lu2022dpm,zhao2023unipc} propose higher-order ODE solvers for the probability flow ODE. Alternatively, flow matching (FM) \cite{lipman2022flow,liu2023flow} and related stochastic interpolants \cite{albergo2023stochastic,albergo2023building} offer a more fundamental solution by introducing {\em better couplings} between reference and target measures \cite{tong2023improving,albergo2024stochastic}. This encourages straighter flows for efficient generation.

\subsection{Contributions}
\begin{figure}[t]
    \centering
    \includegraphics[width=0.9\textwidth, trim={0 20 0 0}, clip]{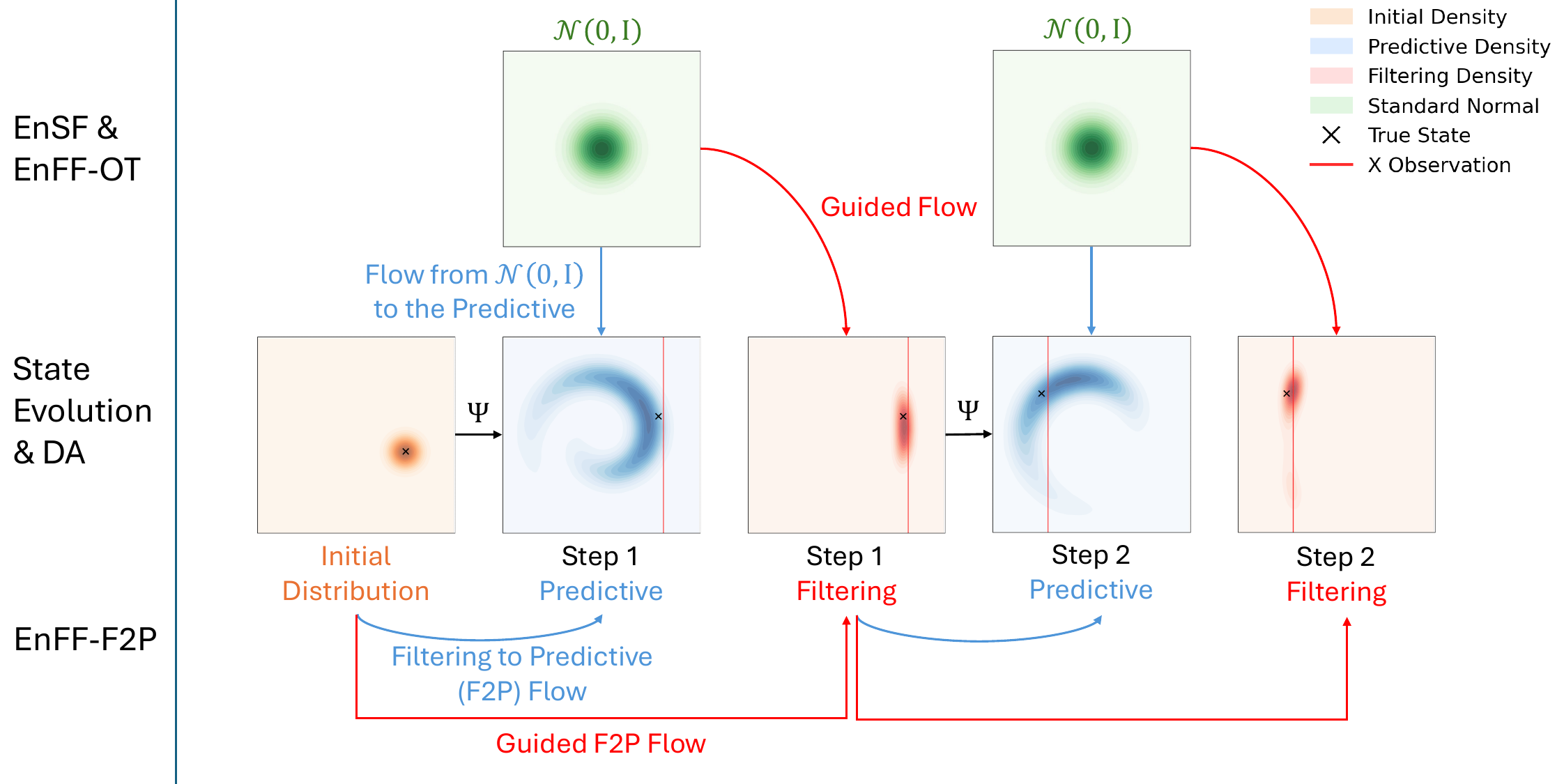}
    \caption{\rev{Comparison of generative model (GM)-based and classical data assimilation frameworks. The flowchart illustrates two consecutive assimilation steps, driven by the dynamical model $\Psi$, comparing \EDIT{standard GM-based approaches} (EnSF and EnFF-OT; top), and our proposed EnFF-F2P (bottom). Classical methods (e.g., EnKF) apply an analysis step on the predictive distribution to yield the filtering distribution. GM-based approaches, such as EnSF and EnFF-OT (top), construct a transport map from a reference distribution $\mathcal{N}(\vect{0}, \mat{I})$ to the predictive distribution, using a guidance term to target the filtering distribution. While utilizing similar guidance, our proposed EnFF-F2P (bottom) completely abandons the $\mathcal{N}(\vect{0}, \mat{I})$ reference, instead using a guided flow from the previous filtering distribution to the predictive distribution. This filtering-to-predictive formulation achieves higher sampling efficiency and demonstrates improved stability with respect to hyperparameters.}}
    \label{fig:f2p_vs_ot}
\end{figure}

We propose the \textit{ensemble flow filter} (EnFF), a training-free FM-based framework for efficient, scalable high-dimensional DA. Its key components are (1) a conditional probability path between latent and predictive distributions and (2) a guidance mechanism steering trajectories toward the filtering distribution (\RefSec[sec:Algorithm]). In \RefSec[sec:Theory], we show EnFF generalizes EnSF and classical filters (BPF, EnKF) through specific choices of its components. Leveraging component (1), we introduce a \textit{filtering-to-predictive} (F2P) path (\RefSec[sec:mc-approximation]) that exploits DA’s iterative Bayesian structure, yielding straighter trajectories than the conditional OT path (\RefFig[fig:ot-vs-f2p]). \rev{EnFF-F2P's core distinction from EnSF and EnFF-OT, as illustrated in \RefFig[fig:f2p_vs_ot], is the complete abandonment of the $\mathcal{N}(\vect{0},\mat{I})$ reference distribution.} Empirical results (\RefSec[sec:Numerics]) show EnFF-F2P improves robustness and efficiency while maintaining accuracy competitive with EnFF-OT and EnSF.

\subsection{Notation}
Let $\sZ^+=\{0,1,\ldots\}$, $\R=(-\infty,\infty)$, $\R^+=[0,\infty)$, and $[N]=\{1,\ldots,N\}$.
We use $\gO(\cdot)$ for asymptotic upper bounds.
For $\vect{\tau}\in\R^d$, $\delta_{\vect{\tau}}$ denotes the Dirac mass at $\vect{\tau}$, and $\gN(\vx|\vm,\mC)$ the Gaussian with mean $\vm$ and covariance $\mC$. For simplicity, we use the same notation for measures and densities, where $p(\vx)$ is the density of the measure $p(\dif\vx)$. For a measurable map $f$ and measure $\mu$, we write $f_{\#}\mu$ for the pushforward measure (Definition~\ref{def:pushforward}).

In DA, $\vx$ and $\vy$ denote state and observation, and $j$ the {\it DA timestep}. The probability path time is $t\in[0,1]$.
The discrete step count for solving the ODE/SDE is the {\it sampling timesteps} $T$. When analyzing the flow-ODE, its vector field (VF), and probability path, we omit $j$.

\section{Background}\label{sec:background}
\subsection{Flow Matching}\label{sec:flow-matching}
FM aims to learn a vector field (VF) $\vu_t\colon [0,1] \times \mathbb{R}^d \to \mathbb{R}^d$ with induced flow
$\vect{\phi}_t\colon [0,1] \times \mathbb{R}^{\EDIT{d}} \to \mathbb{R}^{\EDIT{d}}$, such that its {\em probability path}
\begin{align}\label{eq:probability-path}
    p_t := (\vect{\phi}_t)_{\#}\rho_0, \quad \forall t \in [0, 1],
\end{align}
satisfies the endpoint conditions (i) $p_0 = \rho_0$ (e.g., $\mathcal{N}({\bf 0}, \mI)$), and (ii) $p_1 = \rho_1$ (data distribution).
Since the first endpoint condition holds for any $\vect{\phi}_t$ (since $\left.\vect{\phi}_t\right|_{t=0} = id$), we focus on the second in our search for $\vu_t$.
To this end, for $\vz_1 \sim \rho_1$, FM introduces the {\em conditional probability path} $p_t(\vz|\vz_1)$ and {\em conditional VF} $\vu_t(\vz|\vz_1)$ with flow $\vect{\phi}_t\pn{\vz|\vz_1}$ such that $p_1(\vz|\vz_1)\approx\delta(\vz-\vz_1)$ and $\vect{\phi}_t(\vz|\vz_1)$ pushes $\rho_0$ to $p_t(\vz|\vz_1)$:
$$p_t(\;\cdot\;|\vz_1)=\vect{\phi}_t(\;\cdot\;|\vz_1)_{\#}\rho_0.$$
Setting $p_t(\vz):=\E_{\vz_1\sim\rho_1}[p_t(\vz|\vz_1)]$ therefore gives $p_1(\vz)\approx\rho_1(\vz)$, as desired. Moreover, the marginal and conditional VFs satisfy \cite[Theorem 1]{lipman2022flow}
\begin{align}\label{eq:marginal-vf}
    \vu_t(\vz) = \int \vu_t(\vz|\vz_1) \frac{p_t(\vz|\vz_1) \rho_1(\vz_1)}{p_t(\vz)} \wrt{\vz_1}.
\end{align}
A key insight behind FM is that, unlike the marginal probability path $p_t(\vz)$ and VF $\vu_t(\vz)$, their conditional counterparts can be defined explicitly.
One possible definition is the Gaussian conditional probability path: $p_t(\vz|\vz_1) = \mathcal{N}(\vz|\vect{\mu}_t(\vz_1), \sigma_t^2(\vz_1) \mat{I})$
where $\vect{\mu}\colon \spn{0, 1} \times \Reals^d \to \Reals^d$ and $\sigma_t\colon \spn{0, 1} \times \Reals^d \to \hrpn{0, 1}$.
When $\rho_0(\vz) = \Gaussian(\vect{0}, \mat{I})$, we set $\vect{\mu}_0(\vz_1) = \vect{0}$, $\vect{\mu}_1(\vz_1) = \vz_1$, $\sigma_0(\vz_1) = 1$, and $\sigma_1(\vz_1) = \sigmamin$ for some small $\sigmamin > 0$.
For any Gaussian conditional probability path, the corresponding conditional VF and flow are \cite[Theorem 3]{lipman2022flow}
\begin{align}\label{eq:gaussian-conditional-vf}
    \vu_t(\vz|\vz_1) = \sigma_t(\vz_1)^{-1}\delfrac[\sigma_t]{t}(\vz_1)(\vz - \vect{\mu}_t(\vz_1)) + \delfrac[\vect{\mu}_t]{t}(\vz_1), \quad \vect{\phi}_t\pn{\vz|\vz_1} = \vect{\mu}_t\pn{\vz_1} + \sigma_t\pn{\vz_1}\vz
\end{align}
By \cite[Theorem 2]{lipman2022flow}, the marginal VF $\vu_t(\vz)$ can be approximated by a neural network $\vv_t(\vz; \theta)$ by minimizing the conditional flow matching (CFM) loss:
\begin{equation}\label{eq:CFM-loss}
\gL_{\rm CFM}(\theta) := \mathbb{E}_{t\sim U([0,1]),\vz_1\sim p_1(\vz_1), \vz \sim p_t(\vz|\vz_1)}[\|\vu_t(\vz|\vz_1)-\vv_t(\vz;\theta)\|_2^2],
\end{equation}
where the expectation is approximated using Monte Carlo (MC).

We may also extend to settings beyond $\rho_0 = \Gaussian(\vect{0}, \mat{I})$ and use arbitrary reference distributions $\rho_0$ by considering probability paths and VFs conditioned on {\em both endpoints}: $p_t(\vz|\vz_0, \vz_1)$ and $\vu_t(\vz|\vz_0, \vz_1)$, where $\vz_0 \sim \rho_0$ and $\vz_1 \sim \rho_1$.
Here, the path $p_t\pn{\vz|\vz_0,\vz_1}$ satisfies
\begin{align}
    p_0(\vz | \vz_0, \vz_1) \approx \delta(\vz - \vz_0), \quad \mathrm{and} \quad p_1(\vz | \vz_0, \vz_1) \approx \delta(\vz - \vz_1),\label{eq:endpoint-cond}
\end{align}
and the flow $\vect{\phi}_t\pn{\vz|\vz_0,\vz_1}$ of $\vu_t\pn{\vz|\vz_0,\vz_1}$ satisfies $p_t(\;\cdot\; | \vz_0, \vz_1) = \vect{\phi}_t(\;\cdot\; | \vz_0, \vz_1)_{\#}\rho_0$.
For example, we can choose the probability path
    $p_t(\vz|\vz_0, \vz_1) = \mathcal{N}(\vz|t \vz_1 + (1-t) \vz_0, \sigmamin^2 \mat{I}),
    \forall t \in [0, 1]$,
which satisfies conditions (\ref{eq:endpoint-cond}) if $\sigmamin \ll 1$.
Then, letting $p_t(\vz) := \mathbb{E}_{\vz_0\sim \rho_0, \vz_1 \sim \rho_1}\left[p_t(\vz|\vz_0, \vz_1)\right]$, the endpoint conditions $p_0(\vz) = \rho_0(\vz)$ and $p_1(\vz) = \rho_1(\vz)$ hold approximately. Now, introducing an arbitrary joint distribution $\rho(\vz_0, \vz_1)$ whose marginals are $\rho_0(\vz_0)$ and $\rho_1(\vz_1)$, the following relation holds, analogous to \eqref{eq:marginal-vf} (see \cite[Theorem 3.1]{tong2023improving}):
\begin{align}
    \vu_t(\vz) &= \iint \vu_t(\vz|\vz_0, \vz_1) \frac{p_t(\vz | \vz_0, \vz_1) \rho(\vz_0, \vz_1)}{p_t(\vz)} \dif \vz_0 \dif \vz_1. \label{eq:marginal-vf-2}
\end{align}
Again, $\vu_t(\vz)$ can be approximated by a neural network $\vv_t(\vz; \theta)$ by minimizing the following {\em generalized CFM loss} (\cite[Theorem 3.2]{tong2023improving})
\begin{align}
    \label{eq:generalized-cfm-loss}
    \gL_{\rm CFM}(\theta) := \mathbb{E}_{t\sim U[0,1], (\vz_0,\vz_1) \sim \rho(\vz_0, \vz_1), \vz \sim p_t(\vz|\vz_0, \vz_1)}[\|\vu_t(\vz|\vz_0, \vz_1)-\vv_t(\vz;\theta)\|_2^2].
\end{align}

\subsection{Data Assimilation}
\label{sec:background:data-assimilation}
DA frames sequential state estimation as Bayesian inference on a partially observed stochastic system \cite{jazwinski1970stochastic,law2015data,bach2024inverse}. Let $\{\vx_j\}_{j=0}^J \subset \R^d$ be latent states forming a Markov chain, and $\vy_{1:J}=(\vy_1,\ldots,\vy_J)$ with $\vy_j\in\R^{d_y}$ be noisy observations at times $j\in[J]$.

\paragraph{Problem Setting}
We consider the following stochastic model with additive Gaussian model and observation errors:
\begin{subequations}\label{eq:da_generative}
\begin{align}
    \vx_{j} &= \vect{\psi}(\vx_{j-1}) + \vect{\xi}_{j-1},
    && \vect{\xi}_{j-1}\sim\gN(\vect{0}, \mat{\Sigma}),
    && \vx_0\sim \gN(\vm_0, \mC_0), \label{eq:dyn}\\
    \vy_j   &= \vect{h}(\vx_j) + \vect{\eta}_j,
    && \vect{\eta}_j\sim\gN(\vect{0}, \mat{\Gamma}), \label{eq:obs}
\end{align}
\end{subequations}
where $\{\vect{\xi}_j\}$ is independent of $\vx_0$ and $\{\vect{\eta}_j\}$, and $\{\vect{\eta}_j\}$ is independent of $\{\vect{\xi}_j\}$. The induced transition density and likelihood are
\begin{align}
p(\vx_j|\vx_{j-1})=\gN\!\bigl(\vx_j|\vect{\psi}(\vx_{j-1}),\mat{\Sigma}\bigr),\quad
p(\vy_j|\vx_j)=\gN\!\bigl(\vy_j| \vect{h}(\vx_j), \mat{\Gamma}\bigr).
\label{eq:densities}
\end{align}

\paragraph{Objective}
At each time $j$, the objective is to compute the {\em filtering distribution} $p(\vx_j| \vy_{1:j})$, i.e., the law of the current state conditioned on all observations up to $j$.

\paragraph{Recursive Solution}
Let $p(\vx_{j-1}| \vy_{1:j-1})$ denote the previous filtering distribution. DA advances in two steps \cite{bach2024inverse}:
\begin{enumerate}[leftmargin=6mm]
\item \textbf{Prediction:} Propagate uncertainty
to obtain the \emph{predictive} (prior) distribution:
\begin{align}
p(\vx_j | \vy_{1:j-1})
= \int p(\vx_j| \vx_{j-1})\, p(\vx_{j-1}| \vy_{1:j-1}) \wrt{\vx_{j-1}}.
\label{eq:pred_dist}
\end{align}
This step encodes the Markovian evolution and accumulates model-error covariance $\mat{\Sigma}$.

\item \textbf{Analysis:} Incorporate the new observation $\vy_j$ via Bayes’ rule:
\begin{align}
p(\vx_j| \vy_{1:j})
= \frac{p(\vy_j| \vx_j)\, p(\vx_j| \vy_{1:j-1})}
        {\int p(\vy_j| \vx_j)\, p(\vx_j| \vy_{1:j-1}) \wrt{\vx_j}}.
\label{eq:analyze}
\end{align}

\end{enumerate}

\paragraph{Empirical Approximations from Finite Ensembles}
In realistic problems, evolving the full laws in \RefEq[eq:pred_dist][eq:analyze] is intractable. Practical DA methods such as the bootstrap particle filter (BPF) and the ensemble Kalman filter (EnKF) therefore approximate the filtering and predictive distributions using \emph{empirical measures} from a finite ensemble.
That is, at time $j$, we use the analysis ensemble $\{\vx_j^{(n)}\}_{n=1}^N$ to form the approximation $\sum_{n=1}^N
\delta_{\vx_j^{(n)}}(\mathrm{d}\vx)\approx p(\dif \vx_j| \vy_{1:j})$. During prediction, ensemble members are propagated via $\vx_{j+1}^{(n)} \sim p(\vx_{j+1} | \vx_j^{(n)})$, and to perform analysis, BPF uses importance resampling, whereas EnKF uses closed-form posterior computation based on empirical Gaussian approximations of the predictive distribution.

For more details, we refer \EDIT{the reader} to Appendix~\ref{app:DA}.

\subsection{EnSF}\label{sec:ensf}
EnSF
\cite{bao2024score,bao2024ensemble} introduces a filtering method based on DMs. It represents the predictive distribution $p(\vx_j|\vy_{1:j-1})$ via its {\em score function}, learned from particles $\{\hat{\vx}_j^{(n)}\}_{n=1}^N$ obtained by propagating the previous filtering samples through the dynamics. Express the forward SDE as
\(
\dif \vz_t^{j,n} = \vf(\vz_t^{j,n},t)\dif t + g(t)\dif \vect{w}_t, \quad \vz_0^{j,n}=\hat{\vx}_j^{(n)}
\),
and estimate the score $\vs(\vz_t^{j,n},t)\approx\nabla\log p(\vz_t^{j,n}|\vy_{1:j-1})$ via score matching \cite{bao2024score} or MC \cite{bao2024ensemble}. Assuming $p(\vy_j|\vx_j)\propto e^{-J(\vx_j;\vy_j)}$ for some $J>0$, EnSF then considers the reverse-time SDE with a guidance term:
\begin{align}\label{eq:backwards-guided-sde}
    \dif \vz_t^{j, n} = \Big[\vf(\vz_t^{j, n}, t) - g(t)^2 \Big(\vs(\vz_t^{j, n}, t) -\underbrace{c(t) \nabla J(\vz_t^{j,n}; \vy_j)}_{\mathrm{guidance}}\Big)\Big] \dif t + g(t)\dif \hat{\vect{w}}_t,
\end{align}
where $c(t)$ is a monotonically decreasing function with $c(0)=1$, $c(T)=0$, and $\hat{\vect{w}}_t$ the backward Wiener process. Solving \eqref{eq:backwards-guided-sde} backwards from $t=T$ to $0$ with i.i.d.\ initial conditions $\vz_T^{j,n}\sim\gN(\vect{0},\mI)$ for $n=1,\ldots,N$ yields approximate samples $\{\vx_j^{(n)}\}_{n=1}^N$ from $p(\vx_j|\vy_{1:j})$. These particles are then propagated to $j+1$ via \eqref{eq:pred_dist}, producing $\{\hat{\vx}_{j+1}^{(n)}\}_{n=1}^N$ from $p(\vx_{j+1}|\vy_{1:j})$, and the cycle continues.

\subsection{Additional Related Works}

Other diffusion-based DA methods include FlowDAS \cite{chen2025flowdas}, which learns $p(\vx_j|\vx_{j-1})$ once offline using stochastic interpolants \cite{albergo2023stochastic} and samples from $p(\vx_j|\vx_{j-1},\vy_j)$ via guidance, and DiffDA \cite{huang2024diffda} uses \EDIT{a conditional DM to sample} $p(\vx_j|\hat{\vx}_j,\vy_j)$.
Recent extensions of EnSF include latent-space EnSF \cite{si2024latent}, and EnSF with image inpainting \cite{liang2025ensemble}, to handle sparse observations. The score-based DA (SDA) method of \cite{rozet2023score} trains unconditional DMs on local segments of spatiotemporal trajectories that are pieced together to form global trajectories, and \EDIT{incorporates} observations via training-free guidance, such as diffusion posterior sampling \cite{chung2022diffusion}.

\section{EnFF: Our Proposed Ensemble Flow Filter}\label{sec:Algorithm}
This section details our proposed EnFF method (cf.~Algorithm~\ref{alg:enff}). Applying FM to DA involves three key challenges:
\begin{itemize}[leftmargin=5mm]
\item {\bf FM-DA Challenge 1}:
Training FM at each DA step is infeasible for long trajectories.
\item {\bf FM-DA Challenge 2}:
Only predictive (prior) samples are available to approximate the VF for sampling the filtering (posterior) distribution.
\item {\bf FM-DA Challenge 3}:
Classical limitations—such as mode collapse and near-Gaussian assumptions—should be overcome.
\end{itemize}

EnFF operates by: (1) using predictive samples from $\rho_1(\vx_j):=p(\vx_j|\vy_{1:j-1})$ to construct a VF $\vu_t(\vz)$ transporting a reference distribution $\rho_0$ to $\rho_1$ (Section~\ref{sec:mc-approximation}); (2) adding a guidance term $\vg_t(\vz;\vy_j)$ so that the guided VF $\vu_t'(\vz)=\vu_t(\vz)+\vg_t(\vz;\vy_j)$ transports $\rho_0$ to the filtering distribution $p(\vx_j|\vy_{1:j})\propto p(\vy_j|\vx_j)\rho_1(\vx_j)$ (Section~\ref{sec:mc_guidance}); and (3) propagating samples from $p(\vx_j|\vy_{1:j})$ to sample $p(\vx_{j+1}|\vy_{1:j})$.
The cycle continues with $\rho_1\leftarrow p(\vx_{j+1}|\vy_{1:j})$ (Section~\ref{sec:enff_alg}).
\begin{algorithm}[t!]
\caption{\footnotesize Ensemble Flow Filter}
\begin{algorithmic}[1]\scriptsize
\State {\bf Inputs:} Number of particles $N$, transition model $p(\vx_j | \vx_{j-1})$, observation model $p(\vy_j | \vx_j)$, observations $\{\vy_1, \vy_2, \dots, \vy_J\}$, reference distribution $\rho_0$, conditional VF $\vu_t(\vz|\vz_0, \vz_1)$
\State {\bf Initialize:} Sample $\{\vx_0^{(n)}\}_{n=1}^N \sim p(\vx_0)$

\For{$j = 1$ \textbf{to} $J$}
    \For{$n = 1$ \textbf{to} $N$}
        \State Sample $\vz_0^{(n)} \sim \rho_0(\vz_0)$
        \Comment{Sample reference}
        \State Sample $\vz_1^{(n)} = \hat{\vx}_j^{(n)} \sim p(\vx_{j} | \vx_{j-1}^{(n)})$
        \Comment{Sample target}
    \EndFor
    \State (Marginal VF) \,\, $\vu_t(\vz) \approx \sum_{n=1}^{N} w_n(\vz) \vu_t(\vz | \vz_0^{(n)}, \vz_1^{(n)})$, \quad $w_n(\vz) =$ \eqref{eq:marginal-vf-mc}
    \State (Guidance VF) \, $\vg_t(\vz; \vy_j) =$ \eqref{eq:guidance-vf} $\approx$ Estimate by MC or linearization
    \State (Guided VF) \,\quad $\vu_t'(\vz; \vy_j) = \vu_t(\vz) + \vg_t(\vz; \vy_j)$
    \State Denote the flow of $\vu_t'(\vz; \vy_j)$ by $\vect{\phi}_t'(\vz; \EDIT{\vy_j})$
    \For{$n = 1$ \textbf{to} $N$}
        \State $\vx_j^{(n)} = \vect{\phi}_1'(\vz_0^{(n)}; \EDIT{\vy_j})$
        \Comment{Propagate particles with the guided VF}
    \EndFor
\EndFor
\State {\bf Output:} Particle approximation $\{\vx_j^{(n)}\}_{n=1}^N$ for each $j$
\end{algorithmic}
\label{alg:enff}
\end{algorithm}

\subsection{Training-Free Approximation of the Marginal VF}\label{sec:mc-approximation}
We first construct the marginal VF $\vu_t(\vz)$ transporting $\rho_0$ to the predictive distribution $\rho_1$. Regressing a neural network $\vv_t(\vz;\theta)$ to $\vu_t(\vz)$ via the CFM loss \eqref{eq:CFM-loss} at each timestep $j$ is infeasible, so we adopt a training-free MC approximation to $\vu_t(\vz)$, addressing {\bf FM-DA challenge 1}. Following \eqref{eq:marginal-vf-2}, the MC approximation is:
\begin{align}
    \vu_t(\vz) \approx \sum_{n=1}^N w_n(\vz) \vu_t(\vz|\vz_0^{(n)}, \vz_1^{(n)}), \quad w_n(\vz) = \frac{p_t(\vz|\vz_0^{(n)}, \vz_1^{(n)})}{\sum_{m=1}^N p_t(\vz|\vz_0^{(m)}, \vz_1^{(m)})},\label{eq:marginal-vf-mc}
\end{align}
where $(\vz_0^{(n)}, \vz_1^{(n)}) \sim \rho(\vz_0, \vz_1)$ for $n\in[N]$, i.i.d. This
parallels how the score function in EnSF is approximated by MC \cite{bao2024ensemble}.
For the conditional VF $\vu_t(\vz|\vz_0,\vz_1)$, we consider two options. As noted in Section~\ref{sec:flow-matching}, this choice is determined by the reference $\rho_0$, which then determines the conditional path $p_t(\vz_t|\vz_0,\vz_1)$.
We first examine the commonly used optimal transport (OT) conditional VF.

\paragraph{OT VF \cite{lipman2022flow}}
Let $\rho_0(\vz_0)=\gN(\vz_0|\vect{0},\mI)$ and $p_t(\vz_t|\vz_0,\vz_1)=p_t(\vz_t|\vz_1)=\gN(\vz_t|t\vz_1,(1-(1-\sigmamin)t)^2\mI)$ for some $\sigmamin>0$. This yields the {\em OT conditional VF}
$\vu_t(\vz|\vz_1)=\frac{\vz_1-(1-\sigmamin)\vz}{1-(1-\sigmamin)t}$,
with flow $\vect{\phi}_t(\vz|\vz_1)=(1-(1-\sigmamin)t)\vz+t\vz_1$.
Plugging this conditional VF into \eqref{eq:marginal-vf-mc} yields an ODE
$\frac{\dif \vz_t}{\dif t}=\vu_t(\vz_t)$, which transports $\rho_0$ to $\rho_1$, replacing the SDE used in EnSF. The ODE is more sampling-efficient, as Euler integration has error $\gO(\Delta t)$ versus $\gO(\sqrt{\Delta t})$ for Euler–Maruyama on SDEs.
However, we should note that FM with the OT conditional VF has an equivalent DM formulation using SDEs and vice-versa \cite{gao2025diffusion}; e.g., the SDE used in EnSF \cite{bao2024ensemble} can be shown to have an equivalent FM formulation with  $p_t(\vz_t|\vz_1)=\gN(\vz_t|\EDIT{(1 - (1-\epsilon_\alpha)t)\vz_1, (\epsilon_\beta+(1-\epsilon_\beta)t)\mI})$, for some $\epsilon_\alpha, \epsilon_\beta > 0$. This is only slightly different from FM using the OT conditional VF.
Thus, we claim that the key advantage of FM lies in its {\em flexibility in flow design}, allowing more general references $\rho_0$ beyond Gaussians. Next, we exploit this flexibility to design a flow tailored to DA.

\paragraph{F2P VF}\label{ex:cfm}
We bridge the filtering and predictive distributions at timesteps $j-1$ and $j$ directly to avoid an intermediate Gaussian reference distribution.
Let $\rho_0(\vz_0)$ be the filtering distribution at $j-1$, and define  $p_t(\vz_t|\vz_0,\vz_1)=\gN(\vz_t|t\vz_1+(1-t)\vz_0,\sigmamin^2\mI)$ where  $\sigmamin>0$.
The conditional VF is then $\vu\EDIT{_t}(\vz_t|\vz_0,\vz_1)=\vz_1-\vz_0$ with flow
$\vect{\phi}_t(\vz|\vz_0,\vz_1)=\vz+t(\vz_1-\vz_0)$. We call this the {\em filtering-to-predictive} (F2P) conditional VF.

With the F2P conditional VF, we further leverage the transition kernel $p(\vx_j|\vx_{j-1})$ to couple F2P's reference and target distributions.
Specifically, we take $\rho(\vz_0, \vz_1) := p(\vz_1|\vz_0)\rho_0(\vz_0)$ as our joint measure -- clearly this satisfies $\int \rho(\vz_0, \vz_1) \wrt{\vz_1} = \rho_0(\vz_0)$ and by \eqref{eq:pred_dist}, we also have $\int \rho(\vz_0, \vz_1) \wrt{\vz_0} = \rho_1(\vz_1)$, making this a sound coupling.
This ``data-dependent coupling'' \cite{albergo2024stochastic} exploits the causal link between
the reference and target distributions to produce straighter flows\footnote{\rev{Note that marginal velocities can better preserve the straightness of the conditional velocities when the source and target distributions are coupled more tightly. Intuitively, this coupling pairs each filtering sample with its dynamically propagated predictive counterpart, so the marginal VF averages conditional velocities over coherent source--target pairs rather than unrelated pairs. This idea has been exploited in works such as ReFlow
\cite{liu2023flow} and data-dependent coupling \cite{albergo2024stochastic} to improve the sampling efficiency of FM. }}, and thus reduce the number of ODE integration steps.
Figure~\ref{fig:ot-vs-f2p} compares OT and F2P flows, letting $\rho_0(\vz_0)=\gN(\vz_0|\vect{0},\mI)$ and $\rho_1(\vz_1)=\int p(\vz_1|\vz_0)\rho_0(\vz_0) \wrt{\vz_0}$, with transition kernel
\begin{align}
    p(\vz_1 | \vz_0) = \delta\Big(\vz_1 - \big(\vz_0 + \frac{6 \vz_0}{1 + \vz_0^2} + \boldsymbol{b}\big) \Big), \quad \boldsymbol{b} = (5, 5)^\Transpose\EDIT{,}
\end{align}
where $z_0 \mapsto \frac{6 z_0}{1 + z_0^2}$ is understood to be applied element-wise.
We see that F2P with data coupling produces notably straighter flows than OT. Similar effects can be achieved via {\em OT coupling} \cite{tong2023improving}, but this requires computing a transport map at each timestep $j$,
which is infeasible in our setting. In contrast, F2P avoids this overhead.

\subsection{Conditioning on Observations via Guidance}
\label{sec:mc_guidance}
Next, we construct the guidance VF $\vg_t(\vz)$ to nudge states toward observations, enabling sampling from the filtering distribution and addressing {\bf FM-DA challenges 2 and 3}. Following \cite{feng2025guidance}, we seek $\vg_t(\vz)$ such that the guided VF $\vu_t'(\vz):=\vu_t(\vz)+\vg_t(\vz)$ transports $\rho_0$ to the posterior $p(\vz_1|\vy)\propto p(\vy|\vz_1)\rho_1(\vz_1)$, assuming $\vu_t(\vz)$ transports $\rho_0$ to $\rho_1$ (the predictive distribution). Assuming $p(\vy|\vz_1)\propto e^{-J(\vz_1;\vy)}$ for some $J>0$, the expression for $\vg_t(\vz)$ is \cite[Appendix A.1]{feng2025guidance}:
\begin{align}\label{eq:guidance-vf}
    \vg_t(\vz_t; \vy) = \iint \left(\frac{e^{-J(\vz_1 ; \vy)}}{Z_t(\vz_t; \vy)} - 1 \right) \vu_t(\vz_t | \vz_0, \vz_1) p(\vz_0, \vz_1 | \vz_t)  \wrt{\vz_0}\wrt{\vz_1},
\end{align}
where $Z_t(\vz_t; \vy) := \int e^{-J(\vz_1 ; \vy)} p(\vz_1 | \vz_t) \wrt{\vz_1}$.
From \eqref{eq:marginal-vf}, we also have:
\begin{align}\label{eq:guided-vf}
    \vu_t'(\vz_t; \vy) = \iint \frac{e^{-J(\vz_1 ; \vy)}}{Z_t(\vz_t; \vy)} \vu_t(\vz_t | \vz_0, \vz_1) p(\vz_0, \vz_1 | \vz_t)  \wrt{\vz_0}\wrt{\vz_1}.
\end{align}
\RefEq*[eq:guidance-vf] and \RefEq*[eq:guided-vf] are generally intractable.
Hence, we adopt two approximation methods proposed in \cite{feng2025guidance}:

\paragraph{MC Guidance} Let $(\vz_0^{(n)}, \vz_1^{(n)})\sim \rho(\vz_0, \vz_1), n\in [N]$,
and consider the
MC approximation:
\begin{align}\label{eq:mc-guided-vf}
    \vu_t'(\vz) \approx \sum_{n=1}^N w_n'(\vz; \vy) \vu_t(\vz|\vz_0^{(n)}, \vz_1^{(n)}), \,\,\, w_n'(\vz; \vy) = \frac{e^{-J(\vz_1^{(n)} ; \vy)} p_t(\vz|\vz_0^{(n)}, \vz_1^{(n)})}{\sum_{m=1}^N e^{-J(\vz_1^{(m)} ; \vy)} p_t(\vz|\vz_0^{(m)}, \vz_1^{(m)})}.
\end{align}
In Section~\ref{sec:Theory}, we show that MC guidance yields a filtering algorithm that is approximately equivalent to BPF, inheriting its limitations—most notably, mode collapse \cite{bengtsson2008curse, snyder2008obstacles}. Therefore, we also consider the following guidance method.

\paragraph{Localized Guidance} We approximate the guidance VF \eqref{eq:guidance-vf} by linearizing the likelihood $e^{-J(\vz_1;\vy)}$ at point $\hat{\vz}_1(\vz_t) := \mathbb{E}_{\vz_1 \sim p(\vz_1 | \vz_t)}[\vz_1]$. The localized guidance becomes \cite{feng2025guidance}
\begin{equation}\label{eq:local_guidance}
    \vg_t(\vz_t;\vy)\approx -\E_{(\vz_0,\vz_1)\sim p(\vz_0,\vz_1|\vz_t)}\left[\vu_t(\vz_t|\vz_0,\vz_1)(\vz_1 - \hat{\vz}_1)^\Transpose\right]\nabla_{\hat{\vz}_1}J(\hat{\vz}_1),
\end{equation}
where $\vu_t(\vz_t|\vz_0,\vz_1)(\vz_1 - \hat{\vz}_1)^\Transpose\in \R^{d\times d}$.
Under the further assumption that the integral curve of $\vu_t(\vz_t | \vz_0, \vz_1)$ takes the form $\vz_t = \alpha_t \vz_1 + \beta_t \vz_0 + \sigma_t \vect{\varepsilon}$, i.e. the {\em affine flow assumption},
for some Gaussian noise $\vect{\varepsilon} \sim \mathcal{N}(\vect{0}, \mat{I})$ and \EDIT{time-dependent} $C^1$-curves $\alpha_t, \beta_t, \sigma_t$ such that $\sigma_t, \dot{\sigma}_t$ are sufficiently small,
one can show that (see \cite[Appendix A.10]{feng2025guidance} for the derivation)
\begin{align}
    \mathrm{\eqref{eq:local_guidance}} = -\underbrace{\frac{\dot{\alpha}_t\beta_t - \dot{\beta}_t\alpha_t}{\beta_t} \mat{\Sigma}_{1|t}}_{=: \mat{\Lambda}_t} \nabla_{\hat{\vz}_1} J(\hat{\vz}_1; \vy),
\end{align}
where $\mat{\Sigma}_{1|t} := \mathbb{E}_{\vz_1 \sim p(\vz_1|\vz_t)}\left[(\vz_1 - \hat{\vz}_1)(\vz_1 - \hat{\vz}_1)^\Transpose\right]$.
\rev{The resulting linearized localized guidance can be viewed as an FM analogue of {\em diffusion posterior sampling (DPS) \cite{chung2022diffusion}}, a commonly used guidance approximation in diffusion models: both replace an intractable posterior guidance field by a tractable likelihood-gradient correction evaluated at an estimated endpoint, here $\hat{\vz}_1(\vz_t)$. This approximation is a modeling choice for guidance and is not related to the DA setting.}
In practice, computing $\mat{\Sigma}_{1|t}$ is intractable, so we use a constant scheduler
$\mat{\Lambda}_t \approx \lambda \mat{I}$ for some constant $\lambda > 0$.
\rev{In our DA setting, the affine-flow assumption is reasonable because it is a property of the chosen conditional paths, rather than a linearity assumption on the dynamics or observation model; the marginal VF used for sampling need not be affine.}
Note that both OT and F2P conditional VFs satisfy the affine flow assumption: for OT,
we have $\vz_t = \vect{\phi}_t(\vz_0 | \vz_1) = (1-(1-\sigmamin)t) \vz_0 + t\vz_1$, which satisfies the assumption with $\alpha_t = t$, $\beta_t = 1 - (1-\sigmamin)t$, and $\sigma_t = 0$. For F2P, noting that $p_t(\vz|\vz_0,\vz_1) |_{t=0} =\gN(\vz|\vz_0,\sigmamin^2\mI)$ and therefore the initial samples are given by $\tilde{\vz}_0 = \vz_0 + \sigmamin \vect{\varepsilon}, \, \vect{\varepsilon} \sim \mathcal{N}(\vect{0}, \mat{I})$, we have
\begin{align*}
    \vz_t = \vect{\phi}_t(\tilde{\vz}_0 | \vz_0, \vz_1)
    = \tilde{\vz}_0 + t(\vz_1 - \vz_0)
    = t \vz_1 + (1-t)\vz_0 + \sigmamin \vect{\varepsilon},\ \ \vect{\varepsilon} \sim \mathcal{N}(\vect{0}, \mat{I}).
\end{align*}
This satisfies the affine flow assumption with $\alpha_t = t$, $\beta_t = 1-t$, and $\sigma_t = \sigmamin$.

\begin{remark}\label{rmk:why_guidance}
Other methods, like concatenating VFs \cite{xu2024local} (e.g., rescheduling
on $t \in [0,0.5]$ and appending another flow), can also transform a prior-sampling VF to one
for the posterior. However, concatenated VFs
risk sample degeneracy if an intermediate state approaches an empirical distribution. Guided VFs avoid this by modifying the VF over the entire
$t \in [0,1]$.
\end{remark}

\begin{remark}
Even if $\vu_t(\vz_t)$ induces straight flows, the guided VF $\vu_t'(\vz_t)$ may not. This might seem to weaken our claim about F2P’s advantage, but our experiments in \RefSec[sec:Numerics] show that F2P is still effective with a small number of sampling steps $T$.
\end{remark}

\subsection{EnFF}\label{sec:enff_alg}
We now summarize our FM-based DA algorithm in Algorithm~\ref{alg:enff}, which, like classical filtering (Section~\ref{sec:background:data-assimilation}), applies prediction and analysis steps sequentially:

\paragraph{Prediction Step}
Given particles $\{\vx^{(n)}_{j-1}\}_{n=1}^N$ from the filtering distribution $p(\vx_{j-1}|\vy_{1:j-1})$, propagate them via the Markov transition kernel: $\hat{\vx}_j^{(n)} \sim p(\vx_j | \vx_{j-1}^{(n)})$ for $n \in [N]$. Next, choose a flow design (OT or F2P, see Section \ref{sec:mc-approximation}) and compute the corresponding conditional VF $\vu_t(\vz_t|\vz_0, \vz_1)$ and the MC approximation of the marginal VF $\vu_t(\vz)$ via \eqref{eq:marginal-vf-mc}, using sample pairs $(\vz_0^{(n)}, \vz_1^{(n)}) = (\vz_0^{(n)}, \hat{\vx}_j^{(n)})$ with $\vz_0^{(n)} \sim \rho_0(\vz_0)$.

\paragraph{Analysis Step}
For the current observation $\vy_j$, construct the guidance VF $\vg_t(\vz; \vy_j)$ in \eqref{eq:guidance-vf} either by its MC or localized approximation. Using the guided VF $\vu_t'(\vz; \vy_j) = \vu_t(\vz) + \vg_t(\vz; \vy_j)$ with flow $\vect{\phi}_t'(\vz; \vy_j)$, generate particles $\vx_j^{(n)} = \vect{\phi}_1'(\vz_0^{(n)}; \vy_j)$ for $\vz_0^{(n)} \sim \rho_0$, $n \in [N]$.

\rev{
As a sanity check, \RefFig[fig:reviewer2:q5:rotation2d:cmp] shows that EnFF can track the true posterior computed by the Kalman filter in \EDIT{a} linear Gaussian system $\vect{x}_j = \mat{R}\vect{x}_{j-1}$ and $\vect{y}_j = \vect{x}_j + \vect{\eta}_j$, $ \vect{\eta}_j\sim\Gaussian\pn{\vect{0},\mat{I}}$,
where $\mat{R}$ is a rotation matrix that rotates, in this case, by about 5.9 radians.
}

\begin{figure}[!h]
    \centering
    \begin{tabular}{cc}
         \multicolumn{2}{c}{\includegraphics[scale=.4]{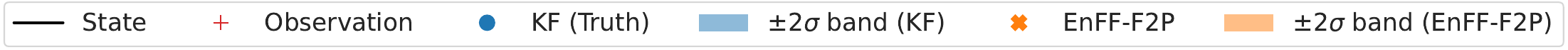}}
         \\
         \includegraphics[scale=.29]{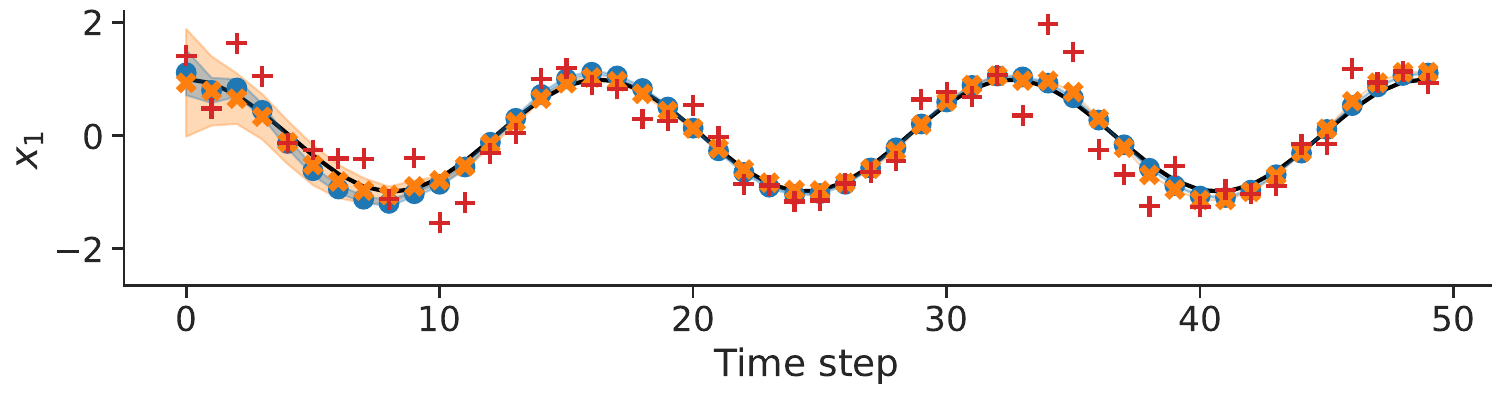}
         &
         \includegraphics[scale=.29]{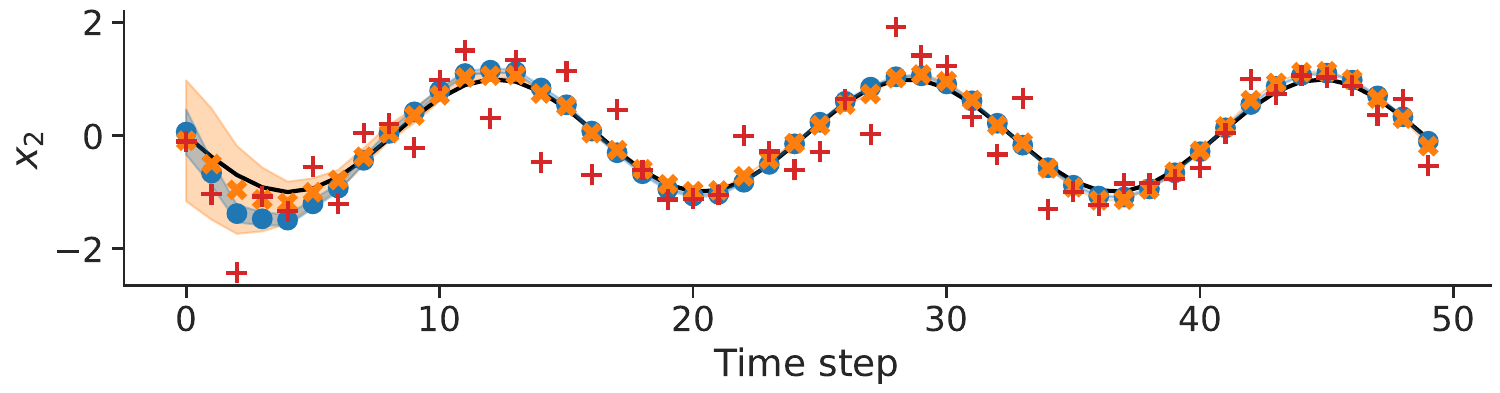}
    \end{tabular}\vspace{-0.3cm}
    \caption{\rev{
        Comparison of the Kalman filter and EnFF-F2P on a 2D linear-Gaussian DA problem.
        The dynamics are rotation by a fixed amount, and identity observations are collected every time step.
        We see that EnFF-F2P tracks the true posterior computed by the Kalman filter.
    }}
    \label{fig:reviewer2:q5:rotation2d:cmp}
\end{figure}

\section{
Theoretical Results}\label{sec:Theory}
This section provides the theoretical foundations of EnFF by analyzing the training-free MC approximation of the marginal vector field and establishing connections to classical filters. All auxiliary definitions and lemmas required for the proofs in this section are provided in Appendix~\ref{app:prelims}.

\subsection{MC Marginal VF}
\rev{This section establishes the theoretical properties of the training-free MC approximation for the marginal VF introduced in Section~\ref{sec:mc-approximation}.}

\rev{We first formalize the exact form of the VF obtained via MC approximation. Specifically, Proposition~\ref{prop:equi_CFM} demonstrates that optimizing the empirical CFM loss directly recovers the MC-approximated marginal VF defined in \eqref{eq:marginal-vf-mc}. The proof proceeds by localizing the expected loss at a fixed point $(\vz, t)$, applying the MC approximation over the sampled pairs, and analytically computing the optimal vector field by setting its gradient to zero.}

\begin{restatable}{proposition}{propEquiCFM}\label{prop:equi_CFM}
Let $\set{(\vz_0^{(n)}, \vz_1^{(n)})}_{n=1}^N \sim \rho(\vz_0, \vz_1)$ be i.i.d. samples. Then, the unique minimizer of the MC-approximated CFM loss \eqref{eq:CFM-loss} at each $(\vx, t)$ is precisely the expression given in \eqref{eq:marginal-vf-mc}, corresponding to the MC approximation of the marginal VF.
\end{restatable}
\begin{proof}
Consider the loss \eqref{eq:generalized-cfm-loss} with the conditional probability path $p_t(\vz|\vz_0,\vz_1)$ and corresponding VF $\vu_t(\vz|\vz_0,\vz_1)$:
\begin{equation}
\begin{split}
\gL_{\rm CFM}(\theta) :=& \E_{t\sim U[0,1],(\vz_0,\vz_1)\sim \rho(\vz_0,\vz_1),\vz\sim p_t(\vz|\vz_0,\vz_1)}[\|\vu_t(\vz|\vz_0,\vz_1)-\vv_t(\vz;\theta)\|_2^2]\\
=& \E_{t\sim U[0,1],(\vz_0,\vz_1)\sim p_t(\vz_0,\vz_1|\vz),\vz\sim p_t(\vz)}
[\|\vu_t(\vz|\vz_0,\vz_1)-\vv_t(\vz;\theta)\|^2_2].
\end{split}
\end{equation}
For a fixed pair $(\vz,t)$, the optimal VF $\vu_t(\vz)$ is given by minimizing
\begin{align}
    \nonumber
    \gL^{(\vz,t)}_{\mathrm{CFM}}(\vv) &:= \E_{(\vz_0,\vz_1)\sim p_t(\vz_0,\vz_1|\vz)}\bigl\|\,\vv-\vu_t(\vz|\vz_0,\vz_1)\bigr\|^2\\
    \nonumber
    &= \iint \bigl\|\,\vv-\vu_t(\vz| \vz_0,\vz_1)\bigr\|^2_2 \frac{p_t(\vz|\vz_0,\vz_1)p(\vz_0,\vz_1)}{p_t(\vz)} \wrt{\vz_0}\wrt{\vz_1}\\
    &\approx \sum_{n=1}^N \bigl\|\,\vv-\vu_t(\vz| \vz_0^{(n)},\vz_1^{(n)})\bigr\|^2_2 \frac{p_t(\vz|\vz_0^{(n)},\vz_1^{(n)})}{\sum_{m=1}^N p_t(\vz| \vz_0^{(m)},\vz_1^{(m)})}
    =: \mathcal{L}^{(\vz,t)}_\mathrm{MC-CFM}(\vv), \label{eq:MC_CFM_loss}
\end{align}
where we apply the MC approximation based on i.i.d. samples $\{(\vz_0^{(n)},\vz_1^{(n)})\}_{n=1}^N\sim \rho(\vz_0,\vz_1)$ from the third line to the fourth.

By setting the gradient of the MC-CFM loss (\eqref{eq:MC_CFM_loss}) with respect to $\vv$ to zero, i.e.,
\begin{equation}
\nabla_{\vv}\,\mathcal{L}^{(\vz,t)}_{\mathrm{MC\text{-}CFM}}(\vv)
= 2 \sum_{n=1}^N \frac{p_t(\vz | \vz_0^{(n)},\vz_1^{(n)})}{\sum_{m=1}^N p_t(\vz | \vz_0^{(m)},\vz_1^{(m)})}\bigl(\vv - \vu_t(\vz| \vz_0^{(n)},\vz_1^{(n)})\bigr) = 0,
\end{equation}
we see that the optimal VF is given by
\begin{equation*}
    \vu_t^{\rm MC}(\vz) = \argmin_{\vv} \mathcal{L}^{(\vz,t)}_\mathrm{MC-CFM}(\vv) = \sum_{n=1}^N w_n(\vz) \vu_t(\vz|\vz_0^{(n)}, \vz_1^{(n)}),\ w_n(\vz) = \tfrac{p_t(\vz|\vz_0^{(n)}, \vz_1^{(n)})}{\sum_{m=1}^N p_t(\vz|\vz_0^{(m)}, \vz_1^{(m)})}.
\end{equation*}
\end{proof}

\rev{Next, we analyze the terminal distribution induced by this MC approximation. Proposition~\ref{prop:equi_empirical} establishes that as the variance $\sigmamin \to 0$, the distribution at $t=1$ converges weakly to the empirical measure of the target samples. This property mirrors data memorization in diffusion models \cite{baptista2025memorization} and formalizes the connection between EnFF and particle filters. The result is established by verifying that the empirical mixture of the conditional probability paths satisfies the continuity equation governed by the MC marginal VF, and subsequently evaluating the terminal distribution via Lemma~\ref{lem:weak-convergence} for the weak convergence.}

\begin{restatable}{proposition}{propEquiEmpirical}\label{prop:equi_empirical}
Let $(\vz_0^{(n)}, \vz_1^{(n)}) \sim \rho(\vz_0, \vz_1)$, i.i.d. $\forall n \in [N]$. Consider the MC approximation of
$\vu_t(\vz)$ in \eqref{eq:marginal-vf-mc}, with reference measure $\rho_0$ and conditional path $p_t(\vz|\vz_0, \vz_1)$ satisfying $p_1(\vz | \vz_0, \vz_1) = \mathcal{N}(\vz|\vz_1, \sigmamin^2\mat{I})$. Then
$p_1(\vz) = \frac1N \sum_{n=1}^N\gN(\vz|\vz_1^{(n)},\sigmamin^2\mat{I})$, and we have weak convergence (Definition~\ref{def:weak_convergence}) $p_1(\dif\vz) \Longrightarrow \tfrac1N \sum_{n=1}^N \delta_{\vz_1^{(n)}}(\dif\vz)$ as $\sigmamin \to 0$.
\end{restatable}
\begin{proof}
The MC approximation of the marginal VF is
\begin{equation}
    \vu^{\rm MC}_t(\vz) = \sum_{n=1}^N w_n(\vz) \vu_t(\vz|\vz_0^{(n)}, \vz_1^{(n)}), \quad w_n(\vz) = \frac{p_t(\vz|\vz_0^{(n)}, \vz_1^{(n)})}{\sum_{m=1}^N p_t(\vz|\vz_0^{(m)}, \vz_1^{(m)})}.
\end{equation}

Since
$p_t(\vz|\vz_0^{(n)},\vz_1^{(n)})$ is generated by $\vu_t(\vz|\vz_0^{(n)},\vz_1^{(n)})$ for any $n\in [N]$, they must satisfy the continuity equation:
\begin{equation}
    \frac{\partial}{\partial t} p_t(\vz|\vz_0^{(n)},\vz_1^{(n)}) + \nabla_\vz\cdot \left[p_t(\vz|\vz_0^{(n)},\vz_1^{(n)})\vu_t(\vz|\vz_0^{(n)},\vz_1^{(n)})\right] = 0.
\end{equation}
Therefore, averaging over samples, we get
\begin{equation}
\begin{split}
    &\frac{\partial}{\partial t} \left[\frac{1}{N}\sum_{n=1}^N p_t(\vz|\vz_0^{(n)},\vz_1^{(n)})\right] + \nabla_\vz\cdot \left[\frac{1}{N}\sum_{n=1}^N p_t(\vz|\vz_0^{(n)},\vz_1^{(n)})\vu_t(\vz|\vz_0^{(n)},\vz_1^{(n)})\right] \\
    =&\frac{\partial}{\partial t} \left[\frac{1}{N}\sum_{n=1}^N p_t(\vz|\vz_0^{(n)},\vz_1^{(n)})\right] + \nabla_\vz\cdot \left[\left(\frac{1}{N}\sum_{n=1}^N p_t(\vz|\vz_0^{(n)},\vz_1^{(n)})\right)\vu^{\rm MC}_t(\vz)\right] = 0
\end{split}
\end{equation}
with initial condition $\frac{1}{N}\sum_{n=1}^N p_0(\vz|\vz_0^{(n)},\vz_1^{(n)}) = \frac{1}{N}\sum_{n=1}^N\rho_0(\vz) = \rho_0(\vz).$
This implies that $p_t(\vz) := \frac{1}{N}\sum_{n=1}^N p_t(\vz|\vz_0^{(n)},\vz_1^{(n)}) $ is the marginal probability path corresponding to the MC marginal VF $\vu^{\rm MC}_t(\vz)$ (for uniqueness, see Remark~\ref{rmk:uniqueness_of_marginal_probability_path}). In particular, at time $t=1$, we have $p_1(\vz) = \frac1N \sum_{n=1}^N \gN(\vz|\vz_1^{(n)},\sigmamin^2\mat{I}).$
Then, by Lemma \ref{lem:weak-convergence}, we establish the weak convergence
\begin{align}
    p_1(\dif \vz) \Longrightarrow \frac1N \sum_{n=1}^N \delta_{\vz_1^{(n)}}(\dif \vz),\ \text{as}\ \sigmamin\rightarrow 0.
\end{align}
\end{proof}

\begin{remark}[Uniqueness of the Marginal Probability Path]
\label{rmk:uniqueness_of_marginal_probability_path}
Under some general conditions of the conditional VF $\vu_t(\vz|\vz_0,\vz_1)\colon\R^d\to\R^d$, the uniqueness of the marginal probability path is guaranteed \cite{villani2008optimal,bogachev2015uniqueness}. That is, if $p_t^{(1)}$ and $p_t^{(2)}$ satisfy the continuity equation
\begin{equation}
    \frac{\partial}{\partial t}p_t(\vz) +\nabla_{\vz}\cdot\bigl[p_t(\vz)\,\vu_t(\vz)\bigr]=0,
\quad
p_{0}(\vz)=\rho_0(\vz), \quad \vz\in\R^d ,\quad t\in[0,T],
\end{equation}
we have $p_t^{(1)} = p_t^{(2)}
\quad\text{a.e.\ on }\R^d,\;\forall t\in[0,T].$
Our OT VF and F2P VF defined in Section~\ref{sec:Algorithm} satisfy the uniqueness condition.
\end{remark}

\subsection{Connections to Classical Filters}
Next, we connect EnFF and classical filtering algorithms. Specifically, when using MC guidance as defined in \eqref{eq:mc-guided-vf}, EnFF approximates the behavior of BPF (\EDIT{see} Appendix~\ref{app:BPF} for details).

\rev{Theorem~\ref{thm:equi_BPF} establishes that the equivalence between EnFF and BPF holds specifically in the limit where the diffusion variance vanishes ($\sigmamin \to 0$) under MC guidance.  The proof constructs the marginal probability path governed by the MC-guided vector field, showing that it incorporates the observation likelihood weights $w_n(\vy)$ directly into the flow. At the terminal time $t=1$, the induced distribution forms a Gaussian mixture centered at the predictive particles. By taking the limit as the diffusion variance $\sigmamin \to 0$, we demonstrate that this Gaussian mixture weakly converges to the discrete Dirac mixture representing the BPF filtering distribution.}

\begin{restatable}{theorem}{thmEquiBPF}\label{thm:equi_BPF}
Consider EnFF with reference measure $\rho_0$ and conditional path $p_t(\vz|\vz_0, \vz_1)$ such that
$p_1(\vz | \vz_0, \vz_1) = \mathcal{N}(\vz|\vz_1, \sigmamin^2\mat{I})$,
and let the guidance VF be approximated via MC. Then, in the limit $\sigmamin \to 0$, EnFF (Algorithm~\ref{alg:enff}) becomes equivalent to BPF.
\end{restatable}
\begin{proof}
The key difference between BPF and EnFF lies in how particles are sampled from the filtering distribution. BPF uses importance resampling, while EnFF samples via a guided flow. Thus, to prove our claim, it suffices to show that flow-based sampling in EnFF with MC guidance converges in distribution to importance resampling in BPF as $\sigmamin \rightarrow 0$.

To simplify the argument, we focus on the sampling procedure at a fixed timestep $j$. Following the notation in FM, let $\rho_1$ denote the predictive (prior) distribution,
and $\rho_0$ be an arbitrary reference distribution. The MC-guided flow \eqref{eq:mc-guided-vf} is used to sample from the filtering (posterior) distribution $\rho := p(\vx_{j} | \vy_{1:j})\propto p(\vy_j |\vx_{j})p(\vx_{j} | \vy_{1:j-1}) = p(\vy_j |\vx_{j}) \rho_1(\vx_j)$. For clarity, we omit the timestep subscript $j$ in the following proof.

Given $N$ samples $\hat{\vx}^{(n)} \sim \rho_1$, $n\in [N]$, we define the likelihood weights by
$\EDIT{w_n^L(\vy)} = \EDIT{{p(\vy | \hat{\vx}^{(n)})} / {\sum_{m=1}^N p(\vy | \hat{\vx}^{(m)})}}$.
According to Appendix \ref{app:BPF}, the importance resampling step in BPF is equivalent to sampling from the
measure $\rho_\mathrm{BPF}^N(\dif\vz; \vy) = \sum_{n=1}^N \EDIT{w_n^L}(\vy) \delta_{\hat{\vx}^{(n)}}(\dif \vz).$
Our goal is to show that, in the limit $\sigmamin \rightarrow 0$, EnFF with the MC guidance (\eqref{eq:mc-guided-vf})
samples from $\rho_\mathrm{BPF}^N$, which proves our claim.

In the FM setting, we consider sample pairs $(\vz_0^{(n)}, \vz_1^{(n)}) = (\vz_0^{(n)}, \hat{\vx}^{(n)}), \,n\in [N]$ where $\vz_0^{(n)}\sim \rho_0, \,n\in [N]$. Based on these samples, the MC guidance VF is given by
\begin{align}
\begin{split}
    &\vu_t^\mathrm{MCG}(\vz;\vy) = \sum_{n=1}^N w_n'(\vz; \vy) \vu_t(\vz|\vz_0^{(n)}, \vz_1^{(n)}), \label{eq:u_t_MCG}\\
    &\text{where} \quad w_n'(\vz; \vy) =
    \frac{p(\vy | \hat{\vx}^{(n)}) p_t(\vz|\vz_0^{(n)},
    \vz_1^{(n)})}{\sum_{m=1}^N p(\vy | \hat{\vx}^{(m)}) p_t(\vz|\vz_0^{(m)}, \EDIT{\vz_1^{(m)}})} =
    \frac{\EDIT{w_n^L}(\vy)p_t(\vz|\vz_0^{(n)},
    \vz_1^{(n)})}{\sum_{m=1}^N \EDIT{w_m^L}(\vy) p_t(\vz|\vz_0^{(m)}, \EDIT{\vz_1^{(m)}})}.
\end{split}
\end{align}
Since $\vu_t(\vz|\vz_0^{(n)}, \vz_1^{(n)})$ and $p_t(\vz|\vz_0^{(n)}, \vz_1^{(n)})$ satisfy the continuity equation,
\begin{equation}
    \frac{\partial}{\partial t} p_t(\vz|\vz_0^{(n)},\vz_1^{(n)}) + \nabla_\vz\cdot \left[p_t(\vz|\vz_0^{(n)},\vz_1^{(n)})\vu_t(\vz|\vz_0^{(n)},\vz_1^{(n)})\right] = 0,
\end{equation}
we know that
\begin{equation}\label{eq:MCG_marginal_p}
\begin{split}
&\frac{\partial}{\partial t} p^{\rm MCG}_t(\vz;\vy) + \nabla_\vz\cdot \left[p^{\rm MCG}_t(\vz;\vy)\vu_t^\mathrm{MCG}(\vz;\vy)\right] = 0,\\
&p^{\rm MCG}_t(\vz;\vy) = \sum_{n=1}^N \EDIT{w_n^L}(\vy)p_t(\vz|\vz_0^{(n)},\vz_1^{(n)}),\
p^{\rm MCG}_0(\vz;\vy) = \sum_{n=1}^N \EDIT{w_n^L}(\vy)\rho_0(\vz) = \rho_0(\vz).
\end{split}
\end{equation}
Therefore the marginal path given by the MC guidance VF $\vu_t^\mathrm{MCG}(\vz;\vy)$ is $p^{\rm MCG}_t(\vz;\vy)$ (for uniqueness, see Remark~\ref{rmk:uniqueness_of_marginal_probability_path}). At time $t=1$, since $p_1(\vz | \vz_0, \vz_1) = \mathcal{N}(\vz|\vz_1, \sigmamin^2\mat{I})$, we have $p^{\rm MCG}_1(\vz;\vy) = \sum_{n=1}^N \EDIT{w_n^L}(\vy)\mathcal{N}(\vz|\vz_1^{(n)}, \sigmamin^2\mat{I}).$
According to Lemma \ref{lem:weak-convergence}, we establish the weak convergence (denoted by $\Longrightarrow$),
\begin{equation}\label{eq:weak-convergence-marginal}
    p^{\rm MCG}_1(\dif \vz;\vy) \Longrightarrow \rho_\mathrm{BPF}^N(\dif \vz;\vy),\ \text{as}\ \sigmamin \rightarrow 0.
\end{equation}
Now let \(\vect{x}_0 \sim \rho_0\), and let \(\vect{\phi}_t^\vy(\vect{x} ; \sigmamin)\) be the flow induced by \(\vu_t^{\mathrm{MCG}}(\vz; \vy)\) (\eqref{eq:u_t_MCG}). Define \(\vect{x}_1^{\sigmamin} := \vect{\phi}_1^\vy(\vect{x}_0 ; \sigmamin)\), which has the law \(p^{\mathrm{MCG}}_1(\dif \vz; \vy) := \vect{\phi}_1^\vy\pn{\;\cdot\;;\sigmamin}_\sharp \rho_0(\dif \vz)\). Furthermore, let \(\vect{x}_{\mathrm{BPF}} \sim \rho_{\mathrm{BPF}}^N(\dif \vz; \vy)\). Since \(\vect{x}_1^{\sigmamin}\) and \(\vect{x}_{\mathrm{BPF}}\) represent sampling in EnFF and BPF, respectively, the weak convergence in \eqref{eq:weak-convergence-marginal} implies the convergence in distribution
\begin{equation}
\vect{x}_1^{\sigmamin} \stackrel{d}{\rightarrow} \vect{x}_{\mathrm{BPF}} \text{ as } \sigmamin \rightarrow 0.
\end{equation}
\end{proof}

\begin{remark}
Theorem~\ref{thm:equi_BPF} implies that for fixed \(\sigmamin > 0\), EnFF with MC guidance is equivalent to BPF with jittering, where i.i.d. Gaussian noise with variance \(\sigmamin^2\) is added to particles after resampling.
\end{remark}

Theorem~\ref{thm:equi_BPF} shows that EnFF with MC guidance offers no practical advantage; for small $\sigmamin$, Algorithm~\ref{alg:enff} behaves like BPF. Therefore, \EDIT{we focus hereafter on} \emph{localized guidance}.

\rev{Furthermore, we connect EnFF with EnKF. Theorem~\ref{thm:equi_EnKF} establishes that when restricting the system to a linear observation model and taking the limit $\sigmamin \to 0$, an appropriately designed guidance vector field allows EnFF to exactly match the EnKF filtering distribution. The proof establishes this equivalence by splitting the EnKF analysis step into deterministic and stochastic components. Specifically, an affine vector field is constructed to map the terminal particles via the deterministic Kalman update. This vector field is subsequently modified such that the intermediate target distribution is \EDIT{convolved with the covariance induced by the perturbed-observation term, namely $\mat{K}_j\mat{\Gamma}\mat{K}_j^\Transpose$}. Comparing this final induced measure with the EnKF output confirms their exact equivalence as $\sigmamin \to 0$.}

\begin{restatable}{theorem}{thmEquiEnKF}\label{thm:equi_EnKF}
Assume a linear observation operator $\vect{h}(\vx) = \mat{H}\vx$, and let EnFF use reference measure $\rho_0$ and conditional path $p_t(\vz|\vz_0, \vz_1)$ with $p_1(\vz | \vz_0, \vz_1) = \mathcal{N}(\vz|\vz_1, \sigmamin^2\mat{I})$. Then, with a specially designed guidance $\vg_t(\vz)$, EnFF (Algorithm~\ref{alg:enff}) samples from the same filtering distribution as EnKF in the limit as $\sigmamin \to 0$.
\end{restatable}

\begin{proof}
    According to Appendix~\ref{app:EnKF}, with particles from the predictive distribution $\hat{\vx}_j^{(n)}\sim p(\vx_j|\vy_{1:j-1}),n\in [N]$, the analysis step of EnKF is given by
    \begin{equation}
        \vx_j^{(n)} = \hat{\vx}_j^{(n)} + \mat{K}_j (\vy_j - \mat{H}\hat{\vx}_j^{(n)} - \vect{\eta}_j^{(n)}),\quad \vect{\eta}_j^{(n)}\sim\gN(\mat{0},\mat{\Gamma}),\quad n\in [N],
    \end{equation}
    where $\mat{K}_j\in\EDIT{\R^{d\times d_y}}$ is the Kalman gain and $\mat{H}\in\EDIT{\R^{d_y\times d}}$ is the observation matrix. We split the analysis step into a deterministic part $\tilde{\vx}_j^{(n)} = (\mat{I} - \mat{K}_j \mat{H})\hat{\vx}_j^{(n)} + \mat{K}_j \vy_j $ and a stochastic part $\mat{K}_j\vect{\eta}_j^{(n)}\sim\gN(\vect{0},\mat{K}_j\mat{\Gamma} \mat{K}_j^\Transpose).$ Therefore, the output probability measure given by EnKF is $\rho_{\rm EnKF}^N = \left(\frac{1}{N}\sum_{n=1}^N\delta_{\tilde{\vx}_j^{(n)}}\right)\ast \gN(\vect{0},\mat{K}_j\mat{\Gamma} \mat{K}_j^\Transpose) =
        \frac{1}{N}\sum_{n=1}^N \gN(\tilde{\vx}_j^{(n)}, \mat{K}_j\mat{\Gamma} \mat{K}_j^\Transpose),$
    where $\ast$
    denotes convolution of measures.

    Let $\vect{\phi}_t$ be the flow corresponding to the MC VF $\vu_t(\vz)$
    based on
    sample pairs $(\vz_0^{(n)}, \vz_1^{(n)}) = (\vz_0^{(n)}, \hat{\vx}_j^{(n)}), n\in [N]$ in \eqref{eq:marginal-vf-mc}, where $\vz_0^{(n)}\sim \rho_0, n\in [N]$.
    By Proposition~\ref{prop:equi_empirical},
    we know that $(\vect{\phi}_1)_\sharp \rho_0 = \rho_1 = \frac{1}{N}\sum_{n=1}^N \gN(\vz_1^{(n)},\sigmamin^2 \mat{I}).$
    Now define $\mat{A} = \mat{I} - \mat{K}_j \mat{H}$, $\vb = \mat{K}_j \vy_j$. Then from Lemma~\ref{lem:invertibility_I_minus_tKH}, we know that the matrix $\mat{A}_t := (1-t)\mat{I} + t\mat{A} = \mat{I} - t\mat{K}_j\mat{H}$ is invertible. Next, \EDIT{with $\vect{b}_t = t\vect{b}$,} we introduce a vector field $\tilde{\vu}_t$ defined by
    \begin{equation}\label{eq:tilde_u_t}
        \tilde{\vu}_t(\vz) = (\mat{A}-\mat{I})\mat{A}_t^{-1}(\vz - \vect{b}_t) + \mat{A}_t \vu_t(\mat{A}_t^{-1}(\vz - \vect{b}_t)) + \vect{b}.
    \end{equation}
    According to Lemma~\ref{lem:affine_flow}, we know that the flow $\tilde{\vect{\phi}}_t$ corresponding to $\tilde{\vu}_t$ has the property that
    \begin{equation}
        \tilde{\vect{\phi}}_0(\vz) = \vect{\phi}_0(\vz) = \vz,\quad \tilde{\vect{\phi}}_1(\vz) = \mat{A}\vect{\phi}_1(\vz) + \vect{b}.
    \end{equation}
    Thus, the target distribution corresponding to $\tilde{\vu}_t$ is
    \begin{equation}
    \begin{split}
        \tilde{\rho}_1 &:= (\tilde{\vect{\phi}}_1)_\sharp \rho_0 = ((\mat{A}(\cdot) + \vb)\circ \vect{\phi}_1)_\sharp\rho_0 = (\mat{A}(\cdot) + \vb)_\sharp \rho_1 \\
        &= \frac{1}{N}\sum_{n=1}^N \gN(\mat{A}\vz_1^{(n)} + \vb,\EDIT{\sigmamin^2 \mat{A}\mat{A}^\Transpose}) = \frac{1}{N}\sum_{n=1}^N \gN(\tilde{\vx}_j^{(n)},\EDIT{\sigmamin^2 \mat{A}\mat{A}^\Transpose})
    \end{split}
    \end{equation}
    According to Lemma~\ref{lem:conv_smoothed_path}, and noting that $\mat{\Pi} = \mat{K}_j\mat{\Gamma} \mat{K}_j^\Transpose$ is symmetric positive definite, we can modify the VF $\tilde{\vu}_t$ to $\vu^{\mat{\Pi}}_t$ so that
    the target distribution $\tilde{\rho}_1$ becomes
    $\rho^{\mat{\Pi}}_1 = \tilde{\rho}_1\ast \gN(\vect{0}, \mat{\Pi}) = \frac{1}{N}\sum_{n=1}^N \gN(\tilde{\vx}_j^{(n)},\EDIT{\sigmamin^2 \mat{A}\mat{A}^\Transpose} + \mat{K}_j\mat{\Gamma} \mat{K}_j^\Transpose).$
The design of $\vu^{\mat{\Pi}}_t$ depends on both the VF $\tilde{\vu}_t$ and the probability path $\tilde{p}_t := \EDIT{(\tilde{\vect{\phi}}_t)_\sharp} \rho_0$, which can be explicitly calculated.
By comparing $\rho_{\rm EnKF}^N$ and $\rho^{\mat{\Pi}}_1$, we see that as \(\sigmamin \rightarrow 0\), EnFF with the specially designed VF \(\vu_t^{\mat{\Pi}}\) yields samples from the same distribution as EnKF.
\end{proof}

\begin{remark}\label{rmk:experimental_gap}
\rev{Theorems~\ref{thm:equi_BPF} and \ref{thm:equi_EnKF} establish the theoretical generality of the guided flow framework, showing that EnFF recovers classical filters under specific guidance designs and the limit $\sigmamin \to 0$. In the experiments, however, we implement the localized guidance defined in \eqref{eq:local_guidance} in Section~\ref{sec:mc_guidance}. This practical configuration deviates deliberately from the assumptions of the equivalence results; it employs neither the pure MC guidance required to match the BPF nor the specialized affine construction required for the EnKF. The EnFF algorithm evaluated in the experiments therefore operates differently from these classical baselines.}
\end{remark}

\section{Numerical Experiments}\label{sec:Numerics}
We validate EnFF on benchmark DA tasks, demonstrating: (1) stability in high-dimensional, nonlinear settings; (2) improved efficiency over EnSF, achieving higher accuracy with fewer sampling steps; and (3) flexibility in conditional VF design.

\paragraph{Benchmark Tasks, Baseline Methods, and Evaluation Metrics}

\rev{
We evaluate EnSF and EnFF on (i) the Lorenz '63 ODE \cite{lorenz63}, (ii) the Lorenz '96 ODE \cite{lorenz96}, (iii) the 1D Kuramoto--Sivashinsky (KS) PDE \cite{kuramoto1978diffusion,michelson1977nonlinear}, and (iv) the 2D Navier--Stokes (NS) PDE \cite{temam2001navier}. \RefTable*[tab:experiment:experiment-details] provides more detail about each dynamical system and the \emph{Identity} and \emph{Arctan} observation models. We compare two EnFF variants, \emph{EnFF-OT} and \emph{EnFF-F2P}, using OT and F2P conditional VFs. Performance is measured by the root mean squared error (RMSE) between the ensemble mean and the ground truth and the Energy Score (ES) with power 1, a generalization of the continuous ranked probability score to multiple dimensions. Specifically, at DA time step $j$, let $\vx_j \in \mathbb{R}^d$ denote the true state and let $\{\vx_j^{(n)}\}_{n=1}^N$ denote the ensemble
with the mean $\bar{\vx}_j =\frac{1}{N}\sum_{n=1}^{N}\vx_j^{(n)}.$
For a single trajectory, the time-averaged RMSE and ES over $J$ DA steps are computed as \begin{align}
\operatorname{RMSE}
& =
\frac{1}{J}
\sum_{j=1}^{J}
\left(
\frac{1}{d}
\left\|
\bar{\vx}_j - \vx_j
\right\|_2^2
\right)^{1/2}
\\
\text{Energy Score}
& =
\frac{1}{J}
\sum_{j=1}^{J}
\left[
\frac{1}{N}
\sum_{n=1}^{N}
\left\|
\vx_j^{(n)} - \vx_j
\right\|_2
-
\frac{1}{2N^2}
\sum_{m=1}^{N}
\sum_{n=1}^{N}
\left\|
\vx_j^{(m)} - \vx_j^{(n)}
\right\|_2
\right].
\end{align}
}

\begin{table}[t]
    \centering
    \resizebox{0.8\textwidth}{!}{
    \begin{tabular}{lrrrr}
        \toprule
         & \textbf{Lorenz '63} & \textbf{Lorenz '96} & \textbf{KS} & \textbf{NS} \\
         \midrule
         Dimension $d$ & $3$ & $10^6$ & $256, 512, 1024$ & $3 \cdot s^2$, $s \in \set{16, 64, 256}$ \\
         Timestep size $\Delta t$ & $0.05$ & $0.01$ & $0.25$ & $10^{-4}$ \\
         Burn-in timesteps & $2000$ & $1000$ & $2150$ & $1$ \\
         Dynamics solver & RK4 & RK4 & ETD-RK4 \cite{cox2002exponential} & Chorin's projection \cite{chorin1968numerical} \\
         DA steps & $2000$ & $80$ & $1000$ & \EDIT{$200$ ($60$ for $s = 256$)} \\
         Obs timestep size & $2^i\Delta t, i \in \spn{4}$ & $10\Delta t$ & $4\Delta t$ & $100\Delta t$ \\
         True state init. & $\Gaussian\pn{\vect{0},\mat{I}}$ & $\Gaussian\pn{\vect{0},3^2\mat{I}}$ & \RefAppendix*[appendix:experimental-details:ks] & \RefAppendix*[appendix:experimental-details:ns] \\
         Ensemble init. & $\Gaussian\pn{\vect{x}_0,\mat{I}}$ & $\Gaussian\pn{\vect{0},\mat{I}}$ & $\Gaussian\pn{\vect{x}_0,\mat{I}}$ & $\Gaussian\pn{\vect{x}_0,\mat{I}}$ \\
         \bottomrule
    \end{tabular}
    }
    \caption{
        Data assimilation setting details.
        KS and NS were each solved on three spatial discretizations.
        There are two observation models: Identity, $\vy = \vx + \vect{\eta}$ where $\vect{\eta}\sim\Gaussian\pn{\vect{0},\sigma_y^2\mat{I}}$ \EDIT{with} $\sigma_y = 2$ for Lorenz '63 and $\sigma_y = 0.5$ otherwise; Arctan, $\vy = \arctan\pn{\vx} + \vect{\eta}$ where $\vect{\eta}\sim\Gaussian\pn{\vect{0},\sigma_y^2\mat{I}}$ \EDIT{with} $\sigma_y = 0.1$.
    }
    \label{tab:experiment:experiment-details}
\end{table}

\rev{In practice, we discard an initial set of DA steps and compute the above time averages only over the remaining evaluation window; in our experiments, this corresponds to averaging over the last 50 DA steps. For each method and each experimental setting, these metrics are computed independently over five trajectories, and we report the mean, min, and max across trajectories. All methods use $N = 20$ ensemble members, and their hyperparameters are tuned using Optuna \cite{akiba2019optuna}; see \RefAppendix[appendix:experimental-details] for more details. EnSF and the EnFF variants use Euler-Maruyama and Euler solvers, respectively, for sampling.}

\subsection{Comparison of EnSF and EnFF in High-Dimensional Settings}
\label{sec:numerics:ensf-vs-enff}

\rev{High-dimensional settings cause EnKF and local ensemble transform KF (LETKF) \cite{hunt2007efficient} to struggle due to their costly tuning of sensitive hyperparameters (inflation and localization) \cite{bao2024ensemble}, and high dimensionality makes BPF prone to mode collapse. In \RefFig[fig:rmse-and-crps], three high-dimensional settings with Identity and Arctan observations are used to compare EnSF and the EnFF variants, and to illustrate the performance-cost tradeoff, each method is evaluated for $T \in \set{5,10,20,50,100}$. The sublegends of each plot show the slope of each curve, computed using log-log linear regression, to approximate the convergence rate in RMSE and ES as $T$ increases. At $T =100$, EnSF performs its best for both Identity and Arctan observations, but both EnFF variants either match or improve on the plotted metrics. As $T$ decreases, the slopes show that the gap between EnSF and the EnFF variants grows, implying that the EnFF variants are always more accurate than EnSF in the tested linear and nonlinear observation settings. Visualizations of the estimated states of EnSF and EnFF-F2P for KS at $T = 5$ and NS at $T = 10$ are given in \RefFig[fig:experiments:kuramoto-sivashinsky:trajectory-and-error] and \RefFig[fig:experiments:navier-stokes:256:trajectory:pressure], respectively. EnSF's estimate captures the general pattern of KS, but overestimates the magnitudes, and for NS, it only loosely captures the ground truth's features. In contrast, EnFF-F2P's estimate has virtually negligible error. Comparing the EnFF variants, EnFF-F2P matches or further improves on EnFF-OT's metrics. In addition, \EDIT{the near-zero slopes of EnFF-F2P indicate} that EnFF-F2P frequently attains nearly the same performance when tuning its hyperparameters using $T = 5$ or $T = 100$. EnFF-F2P is orders of magnitude cheaper computationally to tune to obtain the best performance among EnSF and EnFF-OT.}

\begin{figure}[!ht]
    \centering
    \scriptsize
    \begin{tabular}{cccc}
        \multicolumn{4}{c}{
        \includegraphics[scale=.28]{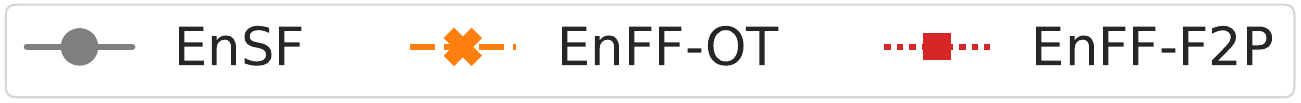}}
        \\
        \multirow{1}{*}[6em]{\rotatebox{90}{Lorenz '96}}
        \includegraphics[scale=.22]{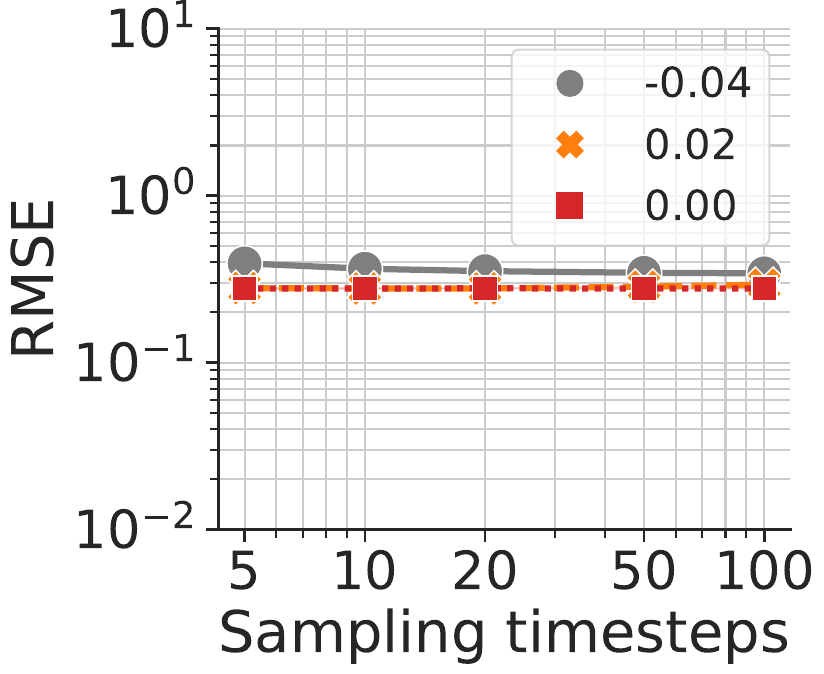}
        &
        \includegraphics[scale=.22]{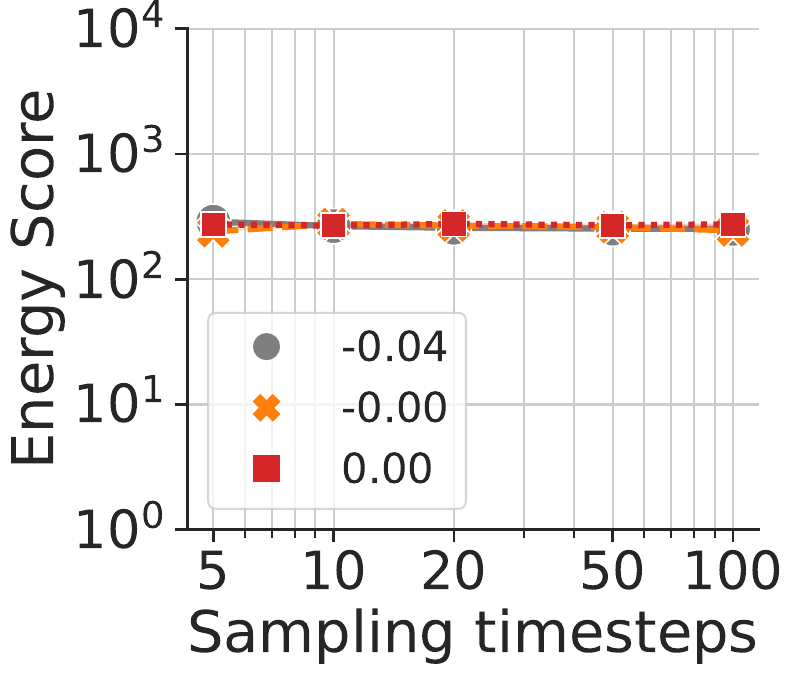}
        &
        \includegraphics[scale=.22]{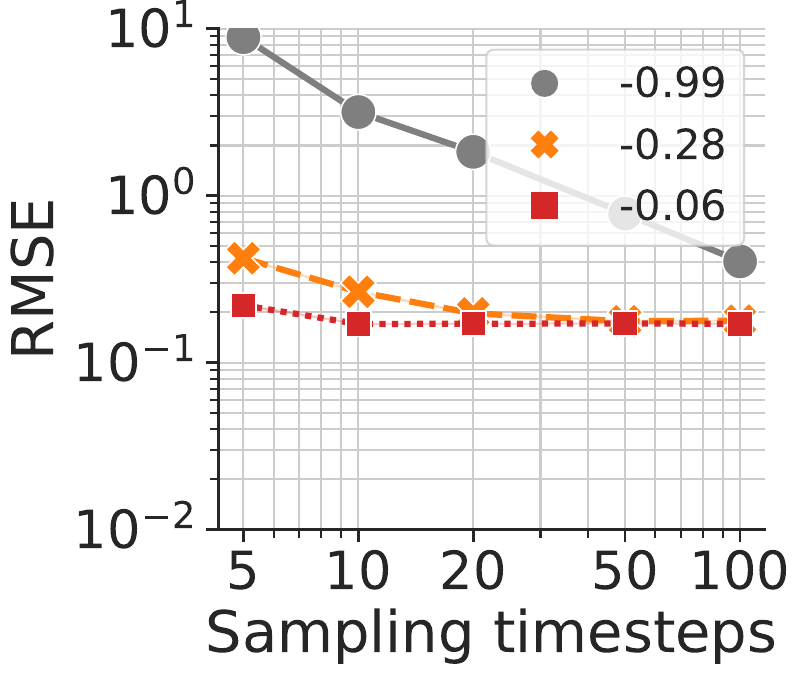}
        &
        \includegraphics[scale=.22]{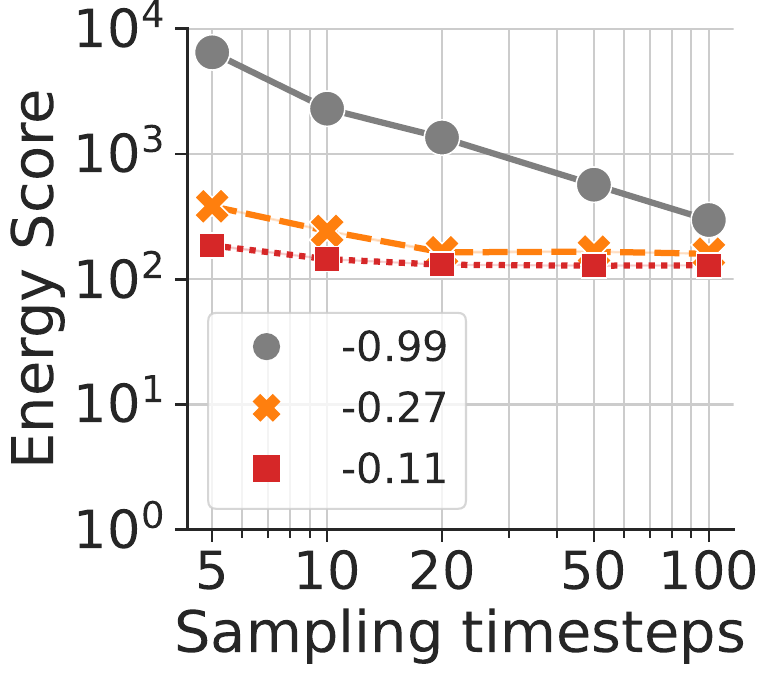}
        \\
        \multirow{1}{*}[4.5em]{\rotatebox{90}{KS}}
        \includegraphics[scale=.22]{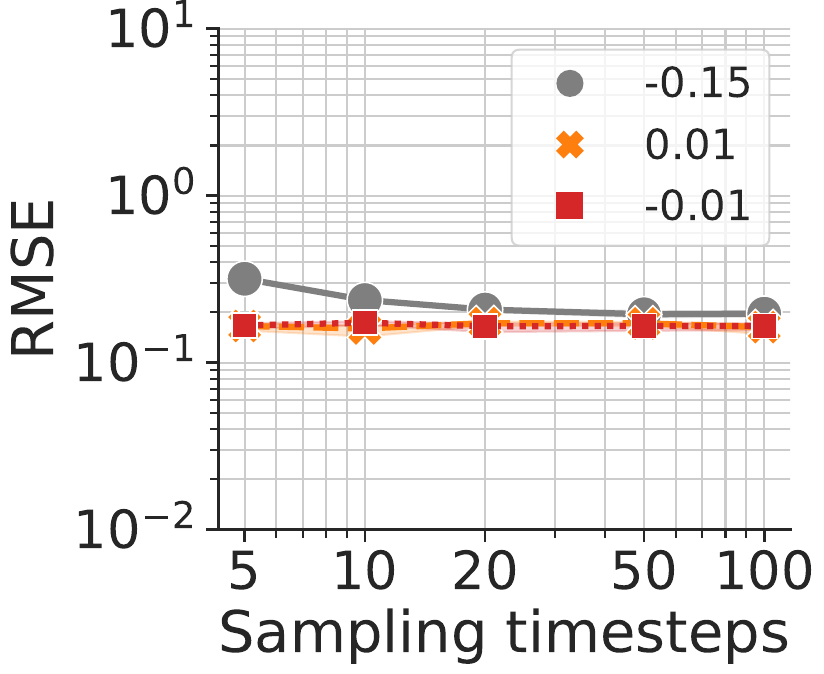}
        &
        \includegraphics[scale=.22]{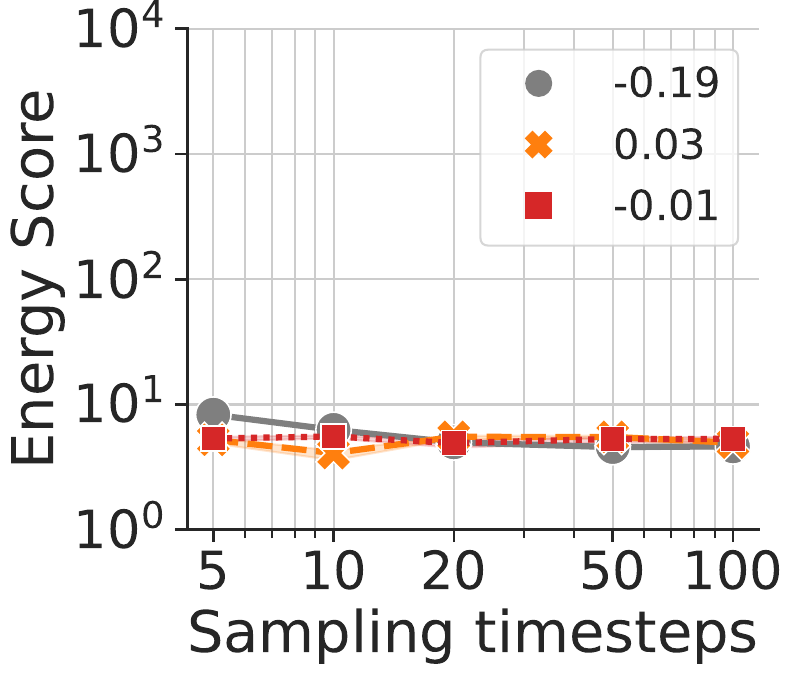}
        &
        \includegraphics[scale=.22]{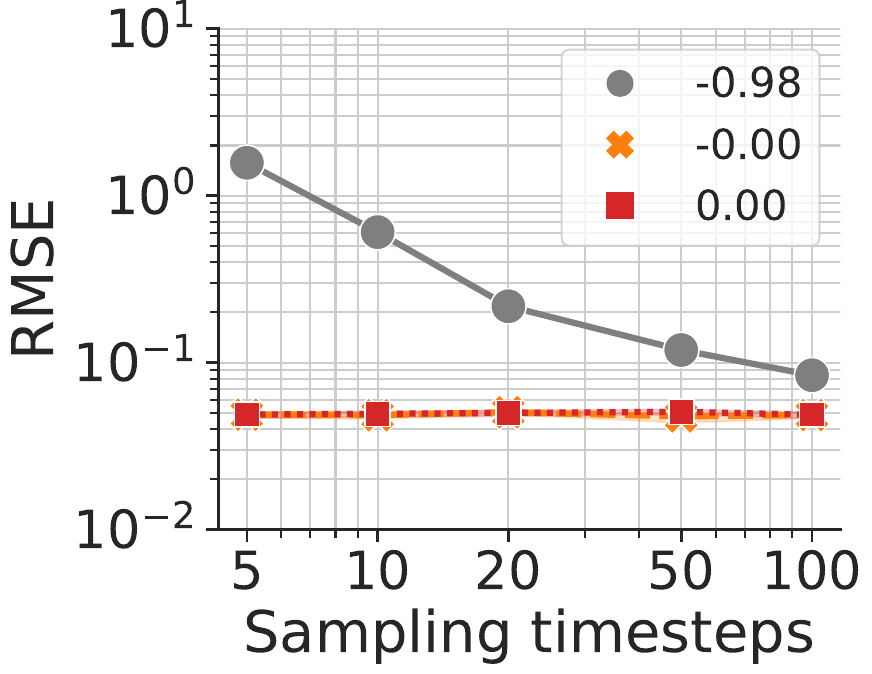}
        &
        \includegraphics[scale=.22]{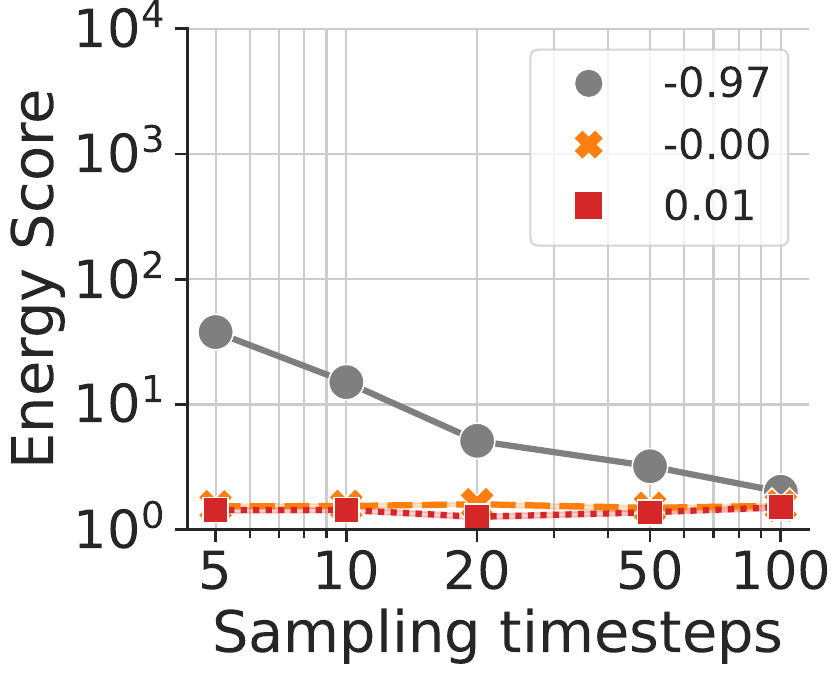}
        \\
        \multirow{1}{*}[5.5em]{\rotatebox{90}{NS-256}}
        \includegraphics[scale=.22]{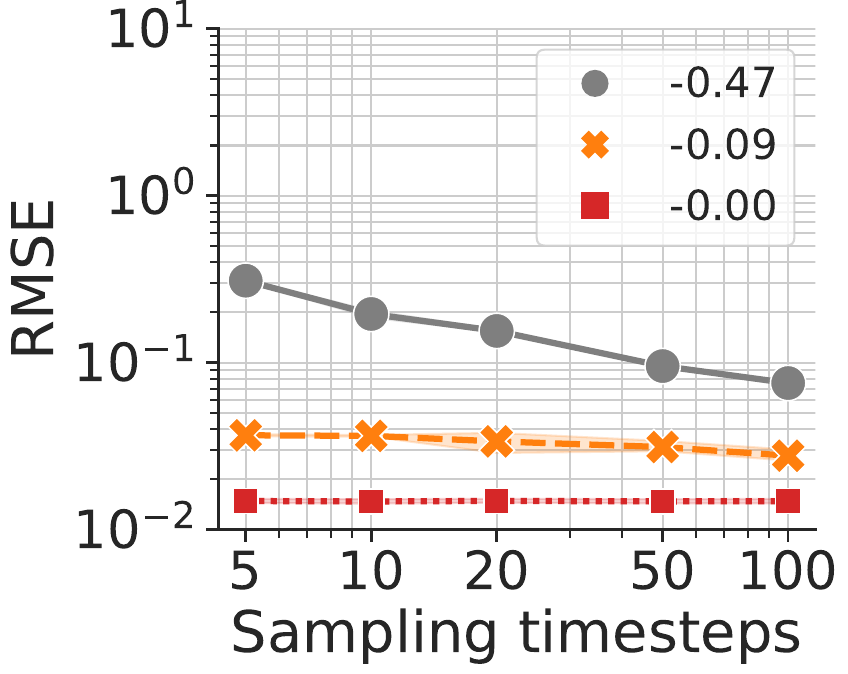}
        &
        \includegraphics[scale=.22]{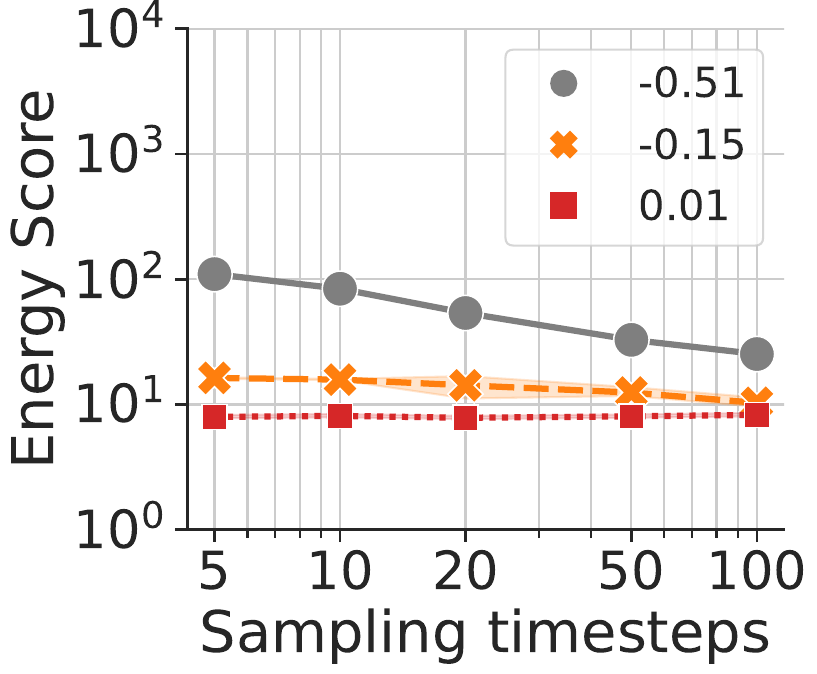}
        &
        \includegraphics[scale=.22]{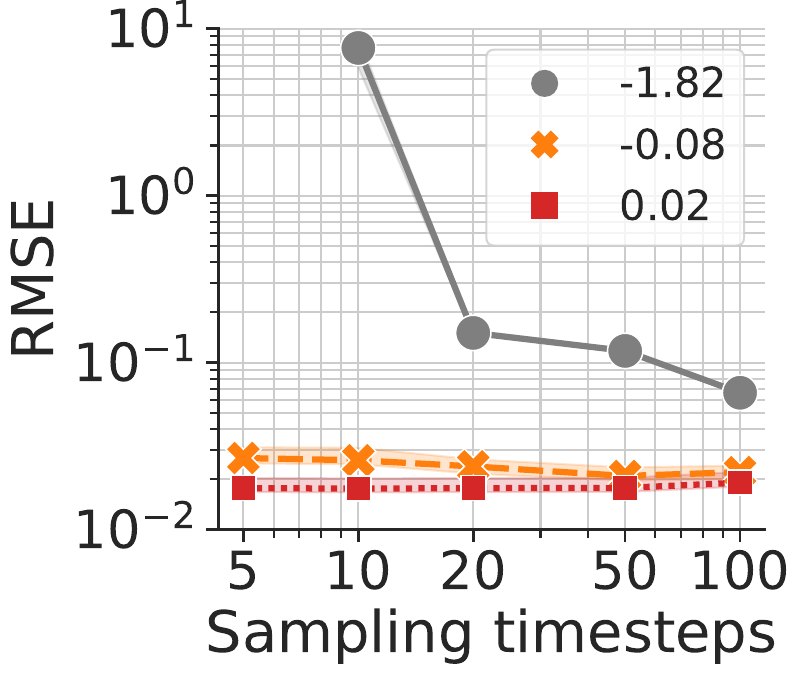}
        &
        \includegraphics[scale=.22]{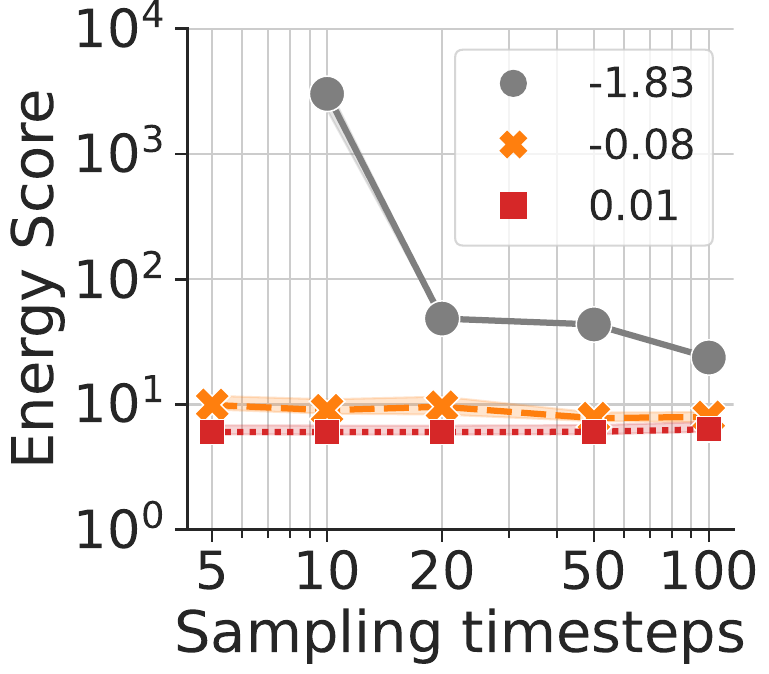}
        \\
        (a) RMSE (Identity)
         &
        (b) ES (Identity)
         &
        (c) RMSE (Arctan)
         &
        (d) ES (Arctan)
    \end{tabular}\vspace{-0.25cm}
    \caption{
        Metrics averaged over five runs varying timesteps $T$ with error bands showing the min and max.
        Hyperparameters are chosen using Optuna \cite{akiba2019optuna} to minimize RMSE for each $T$.
        The sublegends show the slopes from log-log linear regression of the plotted values.
    }
    \label{fig:rmse-and-crps}
\end{figure}

\begin{figure}[!ht]
    \centering
    \begin{tabular}{llccc}
        \multirow{2}{*}[5.27em]{\includegraphics[scale=.25]{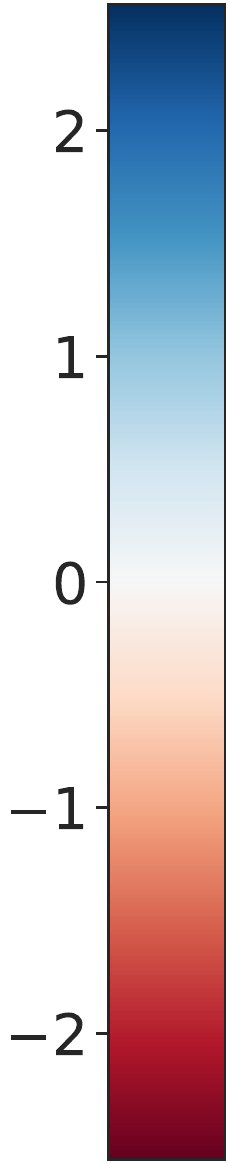}}
        &
        \multirow{2}{*}[5.40em]{\includegraphics[scale=.209]{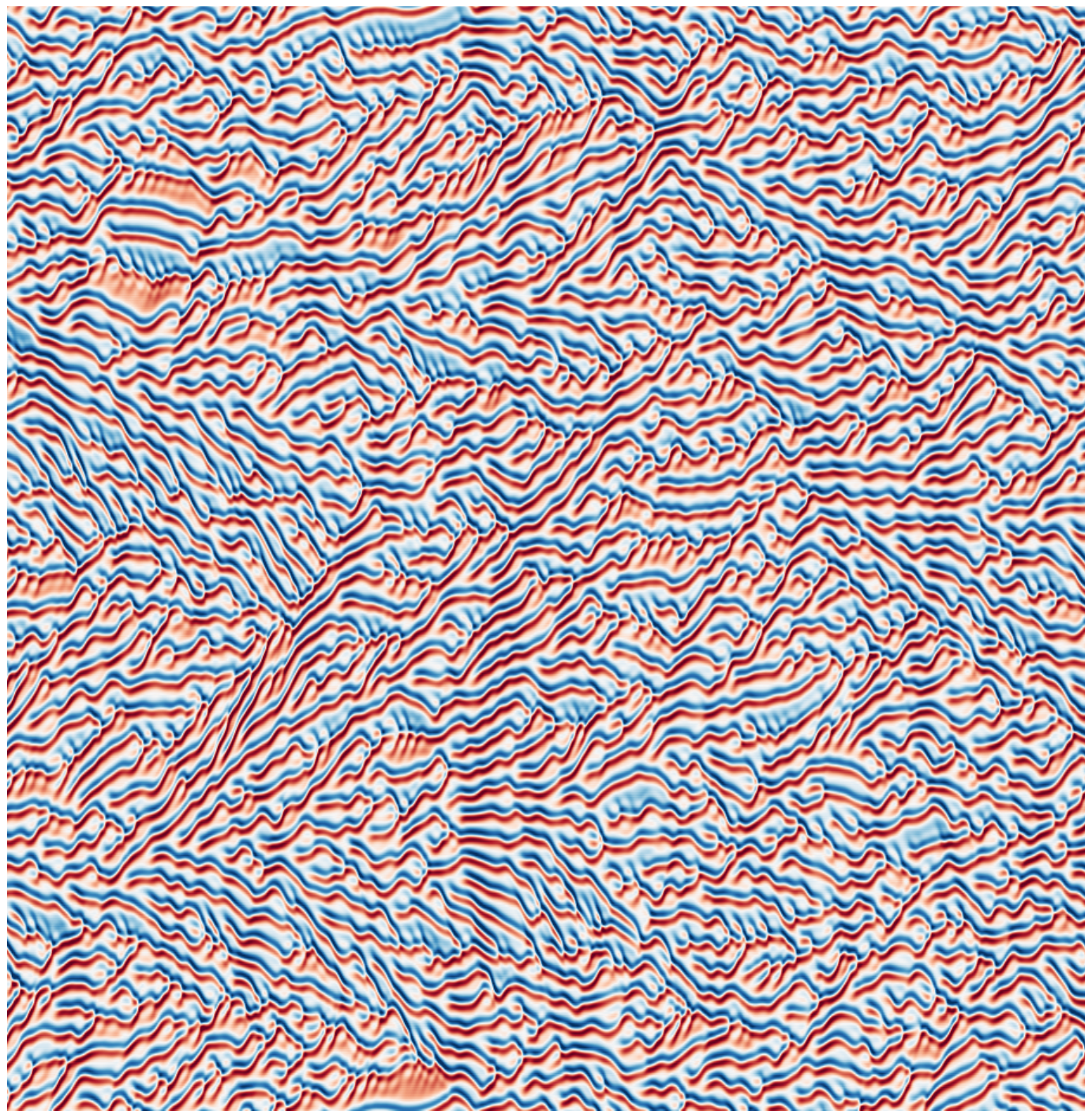}}
        &
        \multirow{1}{*}[3.5em]{\rotatebox{90}{\scriptsize EnSF \cite{bao2024ensemble}}}
        &
        \includegraphics[scale=.1]{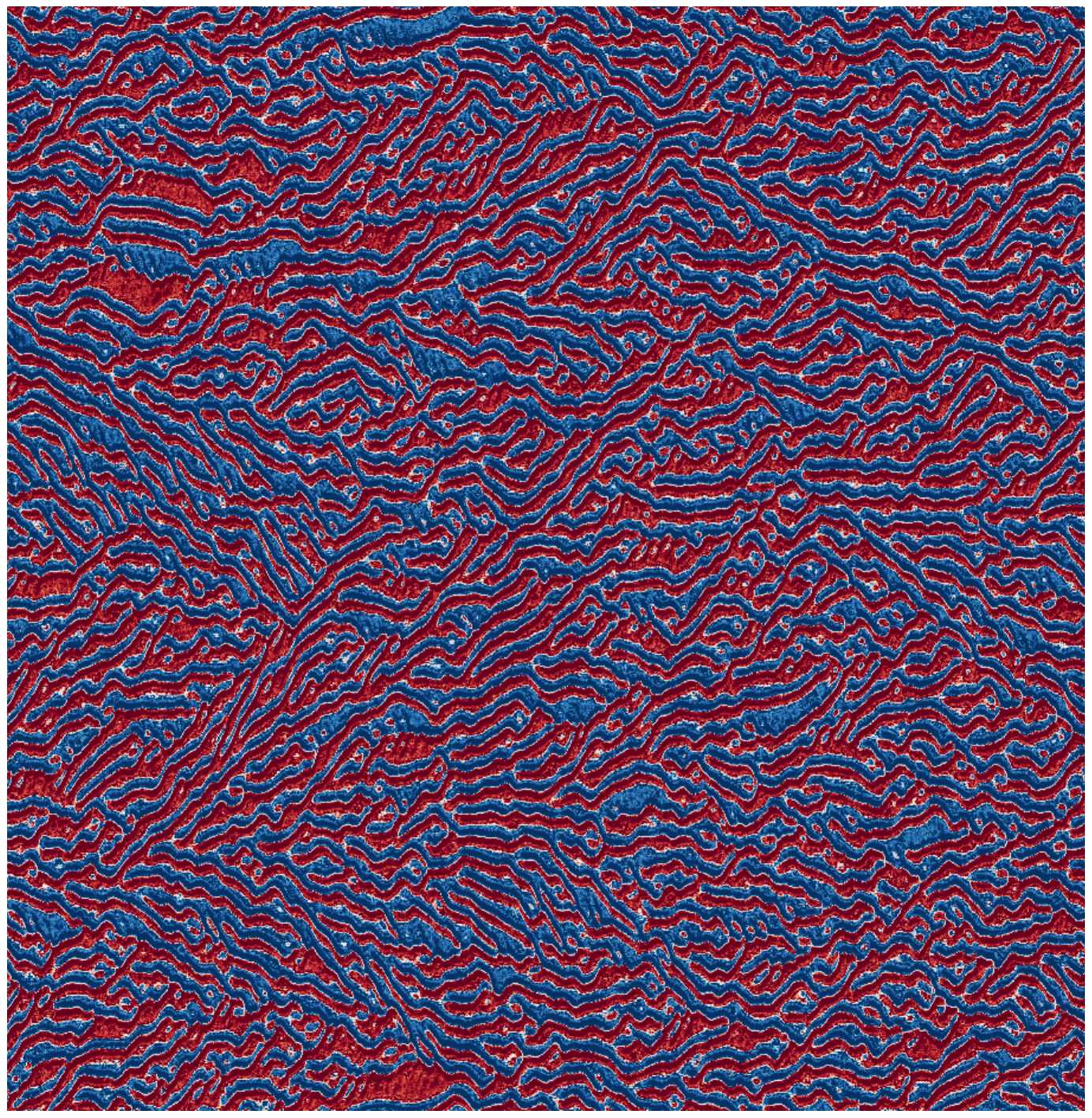}
        &
        \includegraphics[scale=.1]{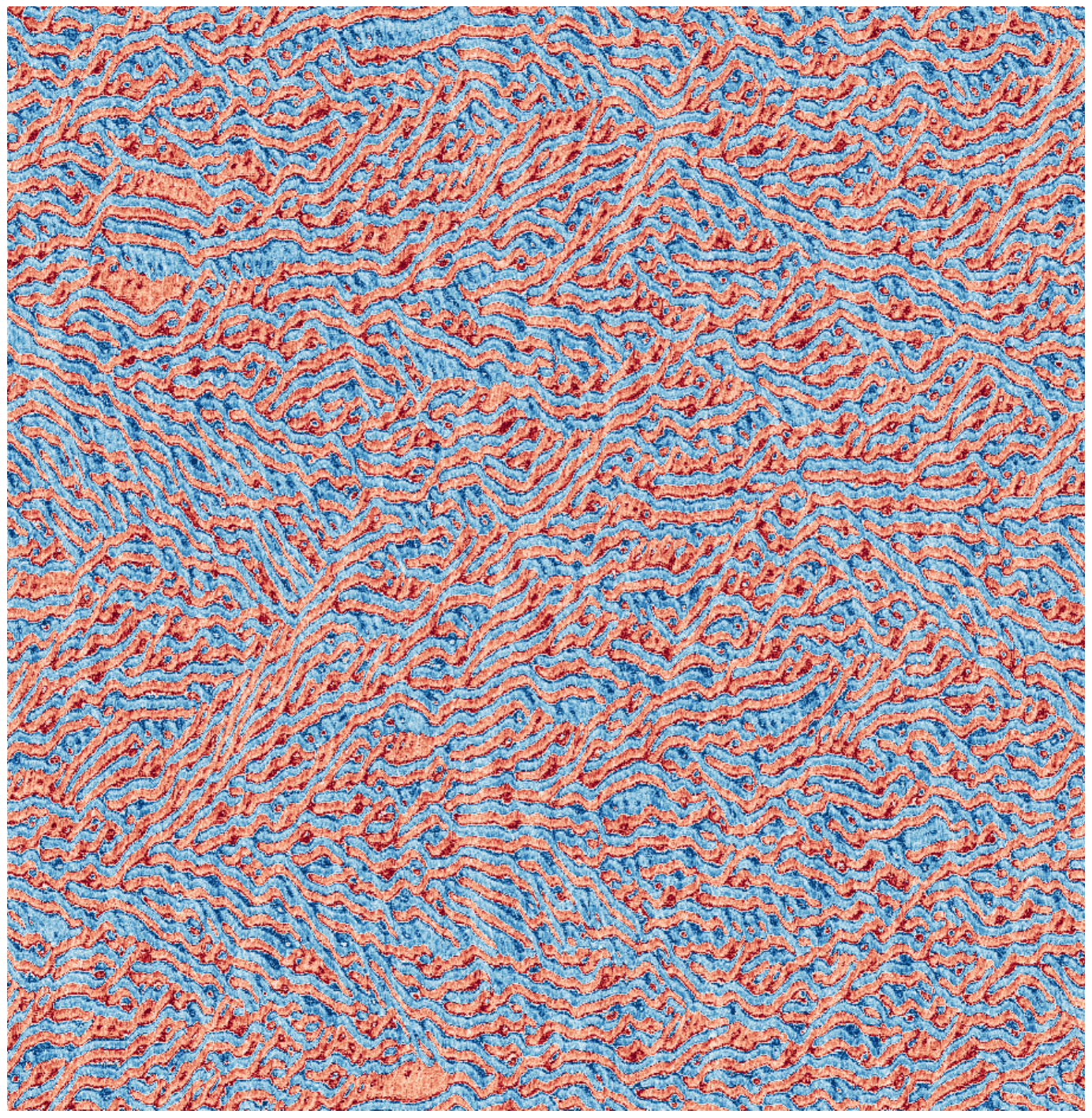}
        \\
        &
        &
        \rotatebox{90}{\scriptsize EnFF-F2P (Ours)}
        &
        \includegraphics[scale=.1]{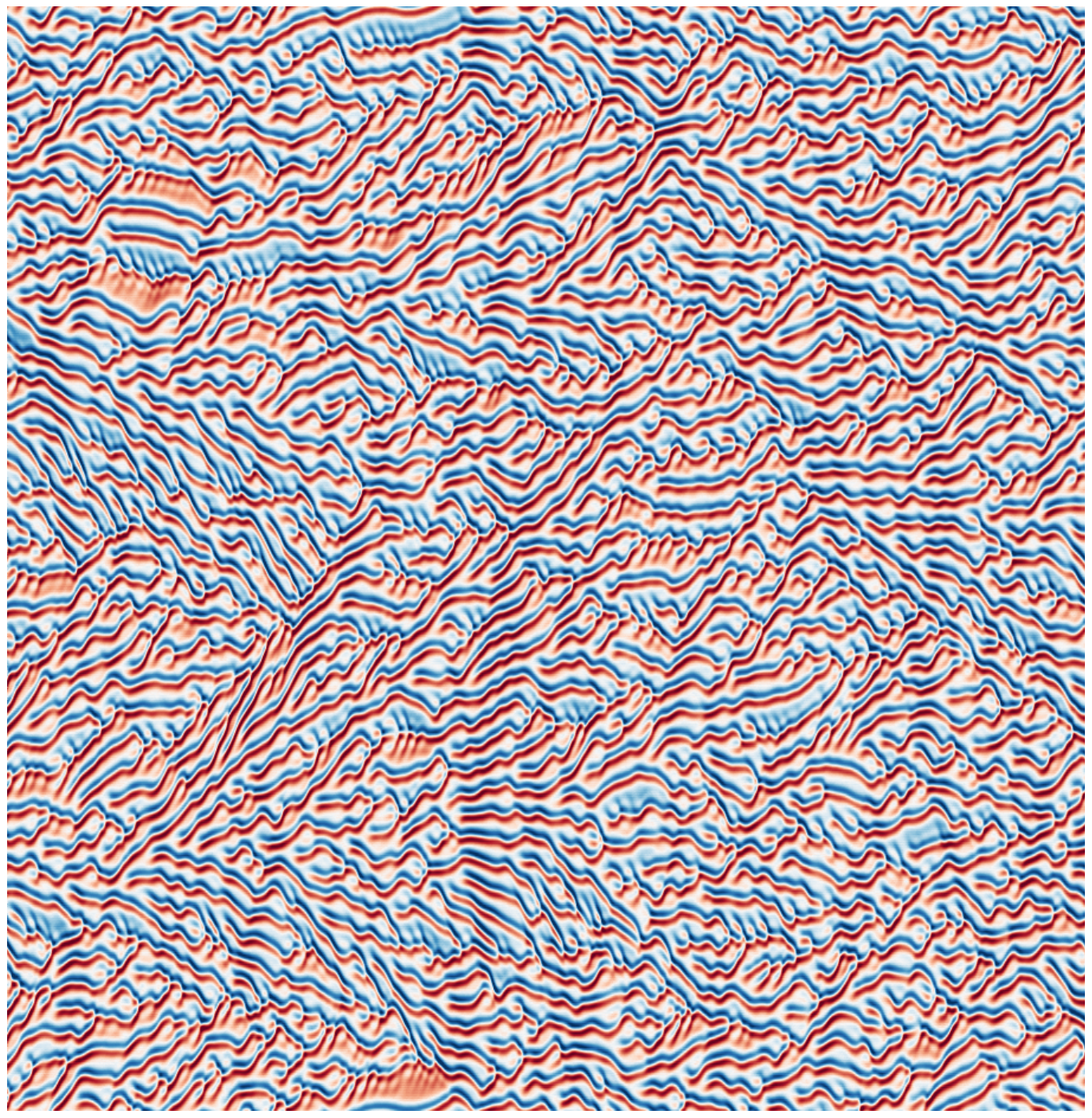}
        &
        \includegraphics[scale=.1]{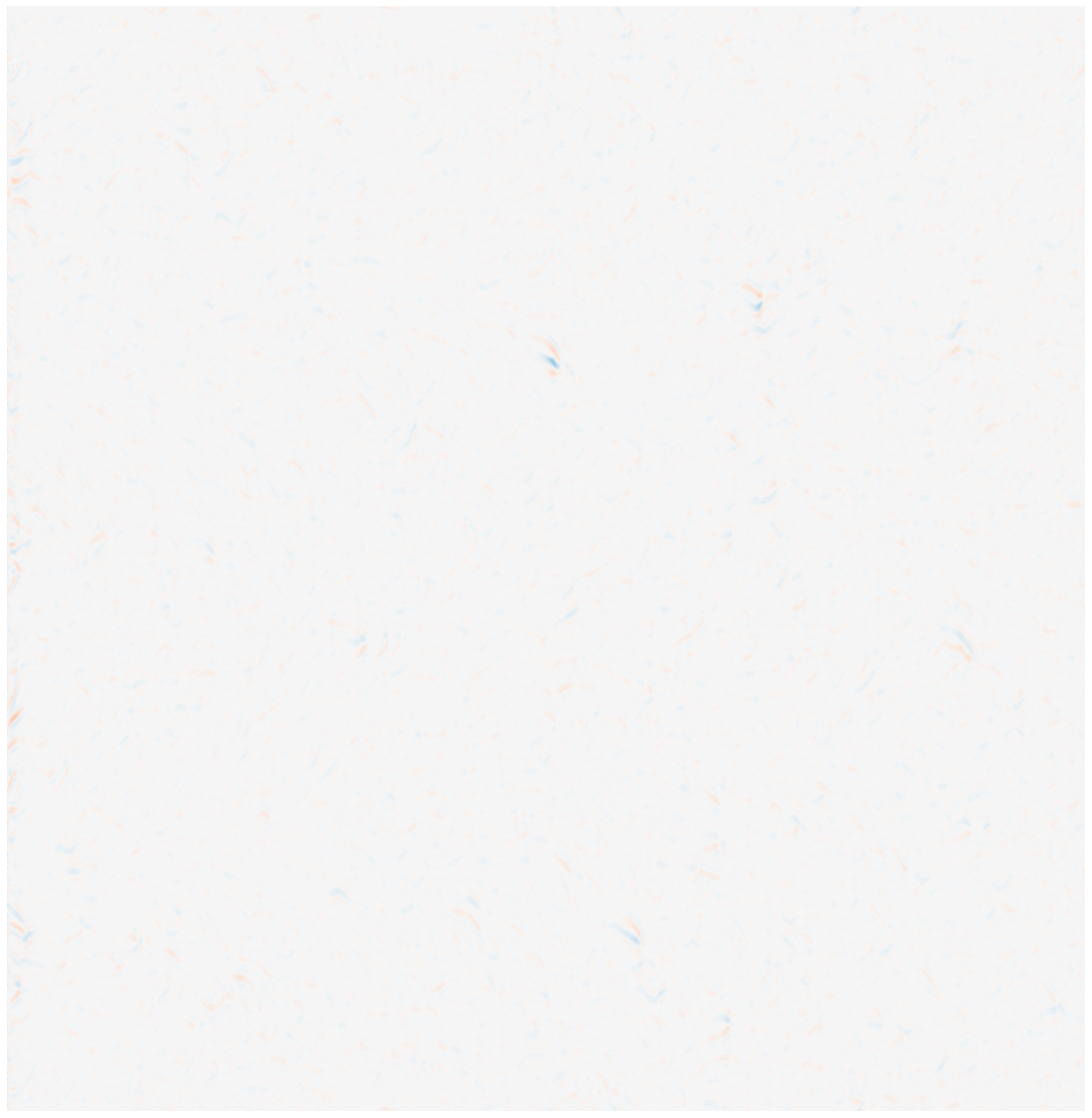}
        \\
        \multicolumn{2}{c}{{\scriptsize\hspace{4.3em} Ground truth}}
        &
        &
        {\scriptsize (a) Prediction}
        &
        {\scriptsize (b) Error}
    \end{tabular}
    \vspace{-0.3cm}
    \caption{
        \footnotesize
        Predicted states (a) and errors (b) of EnSF (top) and EnFF-F2P (bottom) for the 1,024-dimensional KS system with Arctan observations after 1,000 DA steps.
        Time is on the horizontal axis, space on the vertical.
        The filters use $T =5$ sampling steps.
        The colorbar covers the ground truth's value range.
    }
    \label{fig:experiments:kuramoto-sivashinsky:trajectory-and-error}\vspace{-0.15cm}
\end{figure}

\begin{figure}[!ht]
    \centering
    \begin{tabular}{llccc}
        \multirow{2}{*}[5.22em]{\includegraphics[scale=.245]{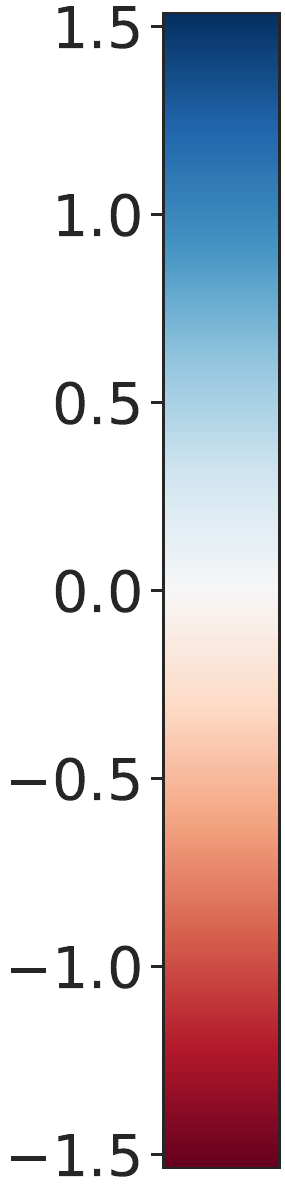}}
        &
        \multirow{2}{*}[5.27em]{\includegraphics[scale=.201]{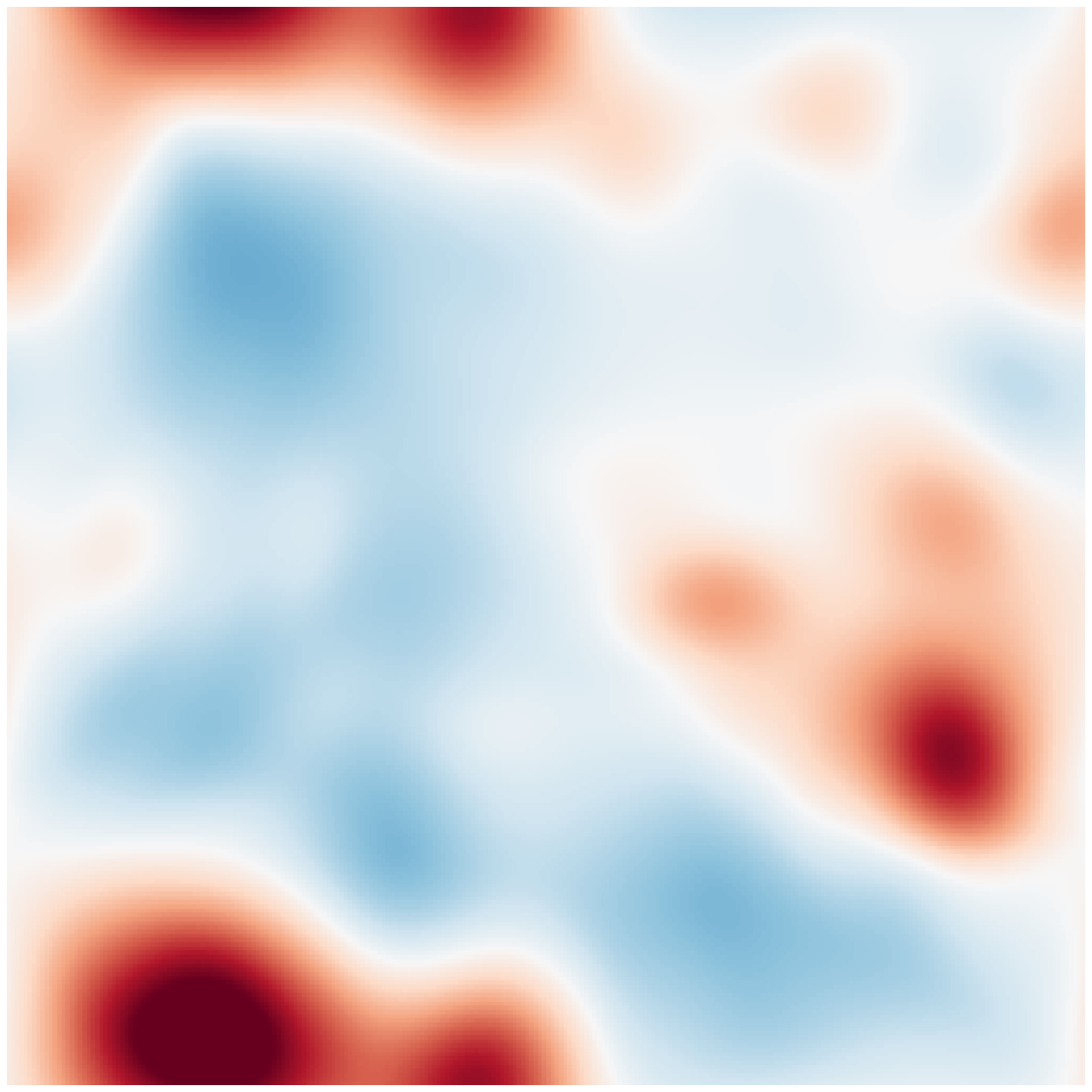}}
        &
        \multirow{1}{*}[3.5em]{\rotatebox{90}{\scriptsize EnSF \cite{bao2024ensemble}}}
        &
        \includegraphics[scale=.097]{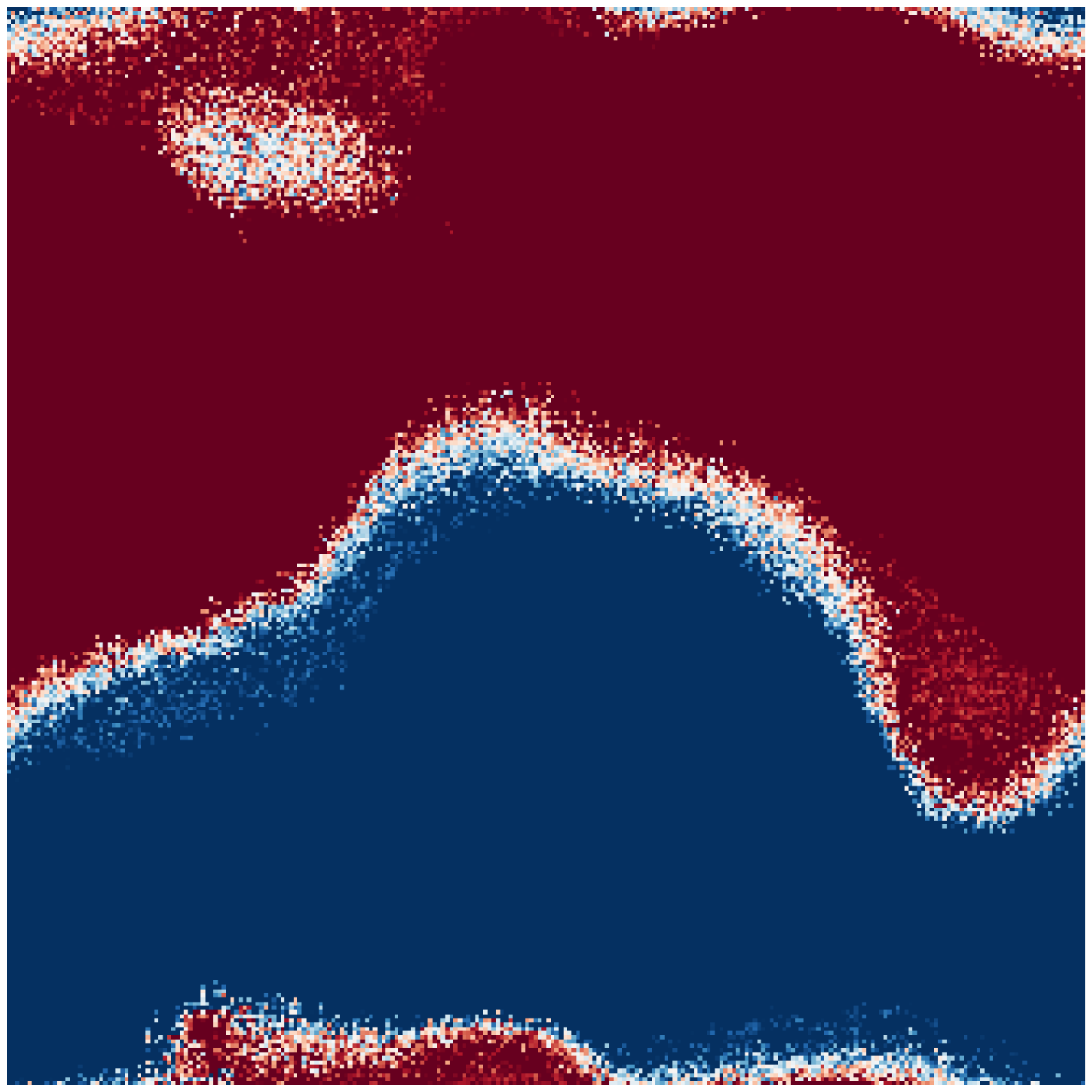}
        &
        \includegraphics[scale=.097]{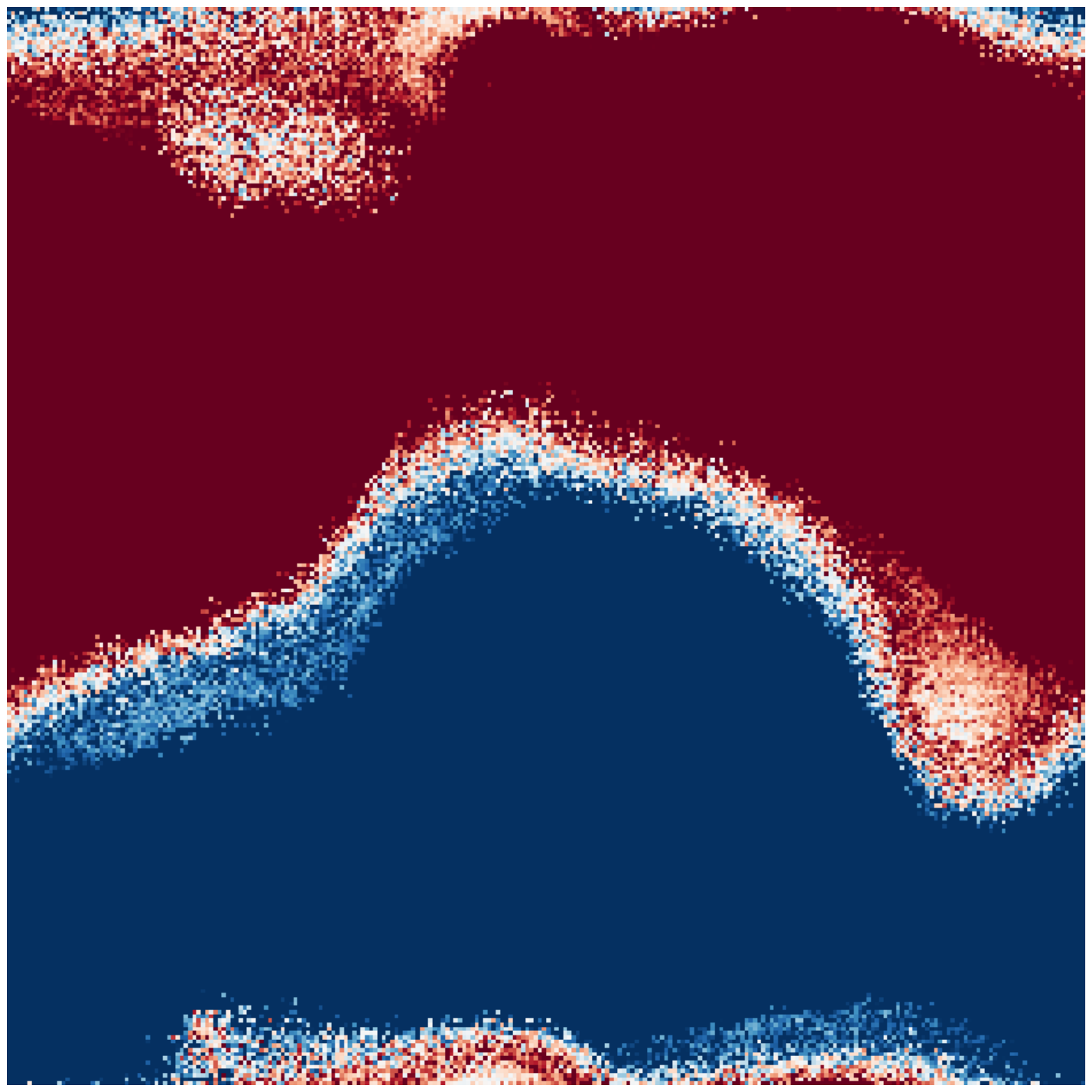}
        \\
        &
        &
        \rotatebox{90}{\scriptsize EnFF-F2P (Ours)}
        &
        \includegraphics[scale=.097]{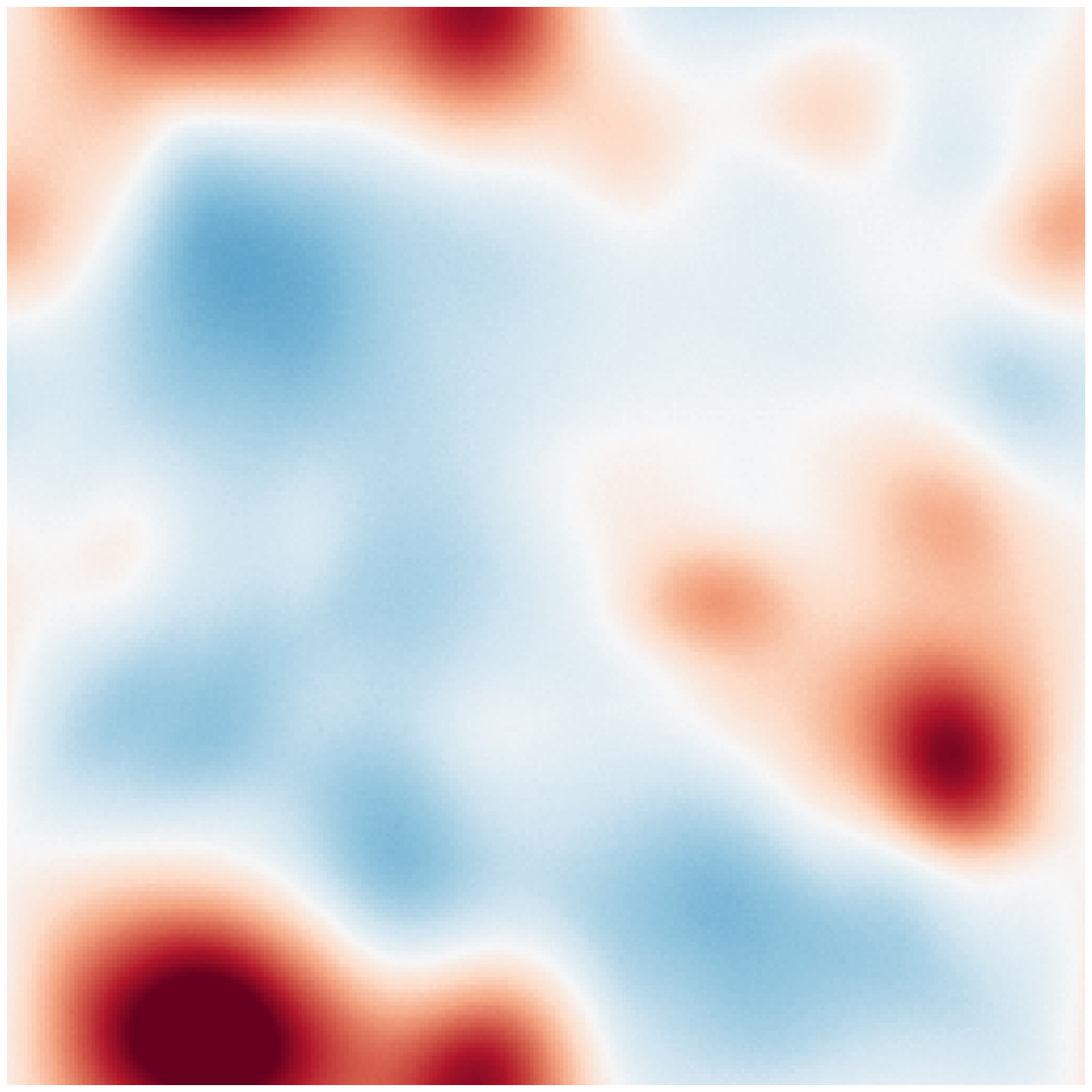}
        &
        \includegraphics[scale=.097]{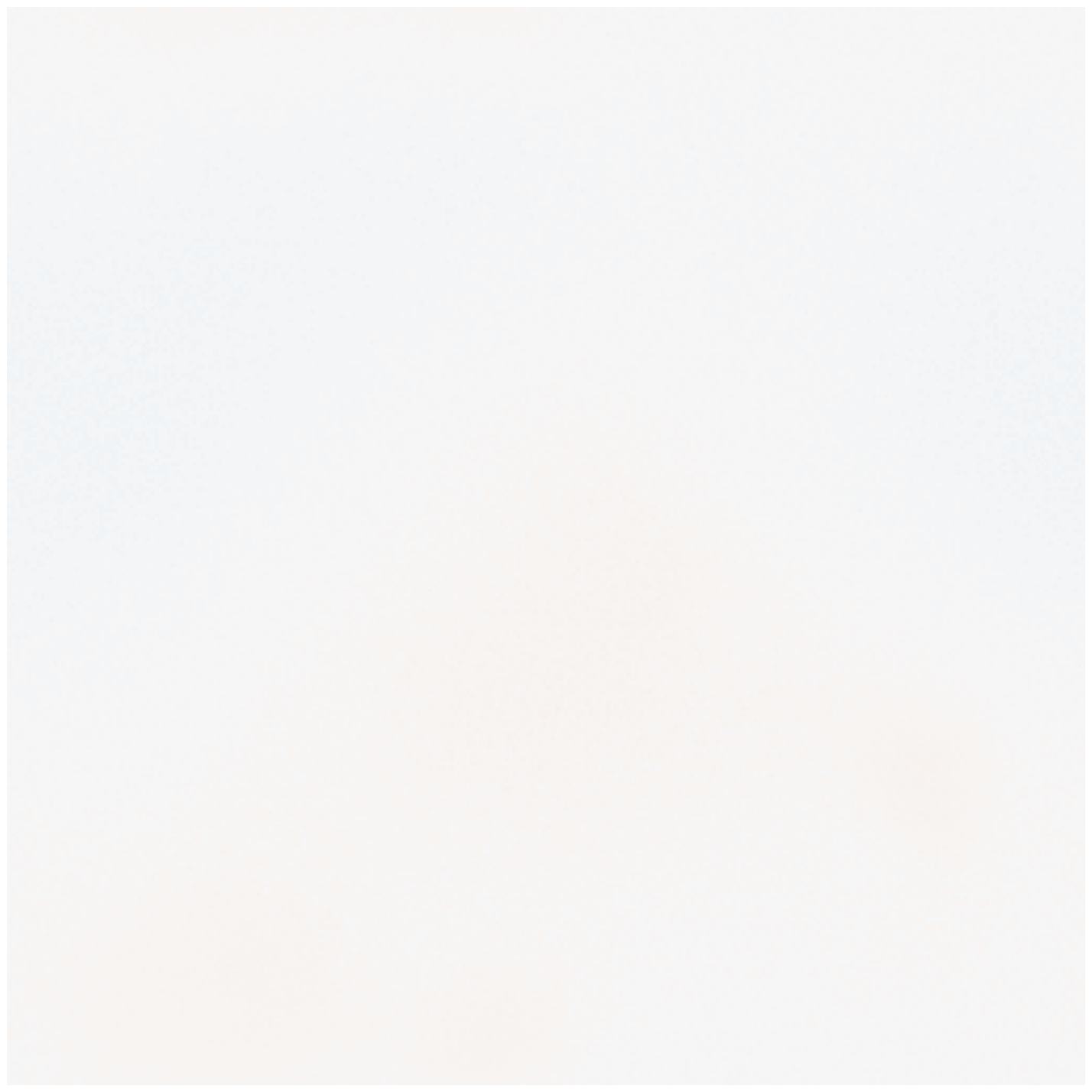}
        \\
        \multicolumn{2}{c}{{\scriptsize\hspace{4.3em} Ground truth}}
        &
        &
        {\scriptsize (a) Prediction}
        &
        {\scriptsize (b) Error}
    \end{tabular}\vspace{-0.2cm}
    \caption{\footnotesize
    Predicted pressure fields of EnSF (top) and EnFF-F2P (bottom) for NS on a $256 \times 256$ grid after 60 DA steps.
    The filters use $T = 10$ sampling steps since EnSF diverges when $T = 5$.
    The colorbar covers the ground truth's value range.
    }
    \label{fig:experiments:navier-stokes:256:trajectory:pressure}\vspace{-0.3cm}
\end{figure}

\subsection{Comparison with Classical DA Methods}
Varying the dimension of KS and NS, we compare EnSF, EnFF-OT, and EnFF-F2P with the following classical DA methods:
\begin{itemize}[leftmargin=5mm]
    \item \textbf{EnKF Perturbed Observation (EnKF-PO)} \cite{burgers1998analysis, anderson2001ensemble}:
    The classic EnKF implementation with stochastic updates, where predicted observations are perturbed by noise sampled from the observation distribution; see Appendix~\ref{app:EnKF}.
    \item \textbf{Iterative EnKF Perturbed Observation (iEnKF-PO)} (\cite{sakov2012iterative}): A variational scheme based on EnKF for strong nonlinearity, solving a time‐adjacent minimization and iteratively refining the analysis.
\item \textbf{Ensemble Square Root Filter (ESRF)} \cite{tippett2003ensemble}:
A deterministic EnKF variant that updates the ensemble via matrix square roots to match EnKF second-order statistics.
    \item \textbf{Local Ensemble Transform KF (LETKF)} \cite{hunt2007efficient}:
    An ESRF-based method that applies spatial localization via the Gaspari--Cohn (GC) function \cite{gaspari1999construction}.
\end{itemize}

\rev{All methods' hyperparameters are tuned on grids of size $512$ for KS and $64 \times 64$ for NS; tuning the classical methods on the grid sizes in \RefSec[sec:numerics:ensf-vs-enff] is prohibitively expensive. The hyperparameters of the classical methods are inflation and localization, except for BPF, iEnKF, and ESRF, where localization is not applied. \EDIT{For KS and NS, EnSF and EnFF} hyperparameters are tuned using $T = 5$ sampling timesteps.}

\subsubsection{Lorenz '63}
\rev{This benchmark provides a controlled comparison of EnSF and the EnFF variants with EnKF as localization is not needed in this low-dimensional setting. The experimental setting, based on \cite{spantiniCouplingTechniquesNonlinear2022}, is given in \RefTable[tab:experiment:experiment-details]. EnSF and the EnFF variants use $T = 10$ sampling timesteps, and EnKF does not use inflation. \RefFig[fig:rmse-crps:lorenz63spantini2022] tests each method varying the time between observations where the metrics are computed using a reference BPF with $10^6$ particles and process noise $\vect{\xi}\sim\Gaussian\pn{\vect{0},\pn{0.01}^2\mat{I}}$ as the ground truth. For Identity observations with $\sigma_y = 2$, both EnFF variants and EnKF perform better than EnSF in each case. Strong performance of EnKF when the time between observations is $0.1$ is expected; however, as the time between observations grows, EnKF and the EnFF variants perform similarly. With Arctan observations and $\sigma_y = 0.1$, EnKF has the lowest mean but diverges when the time between observations is $0.8$. EnSF and the EnFF variants are stable in this strongly nonlinear setting.}

\begin{figure}
    \centering
    \scriptsize
    \begin{tabular}{cccc}
         \multicolumn{4}{c}{\includegraphics[scale=.28]{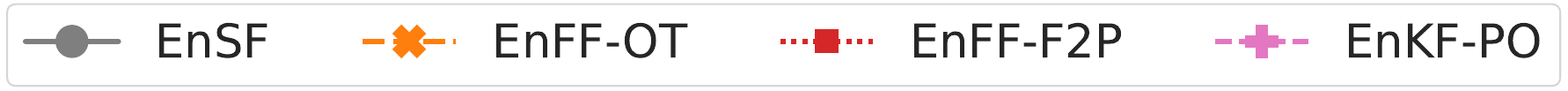}}
         \\
         \includegraphics[scale=.23]{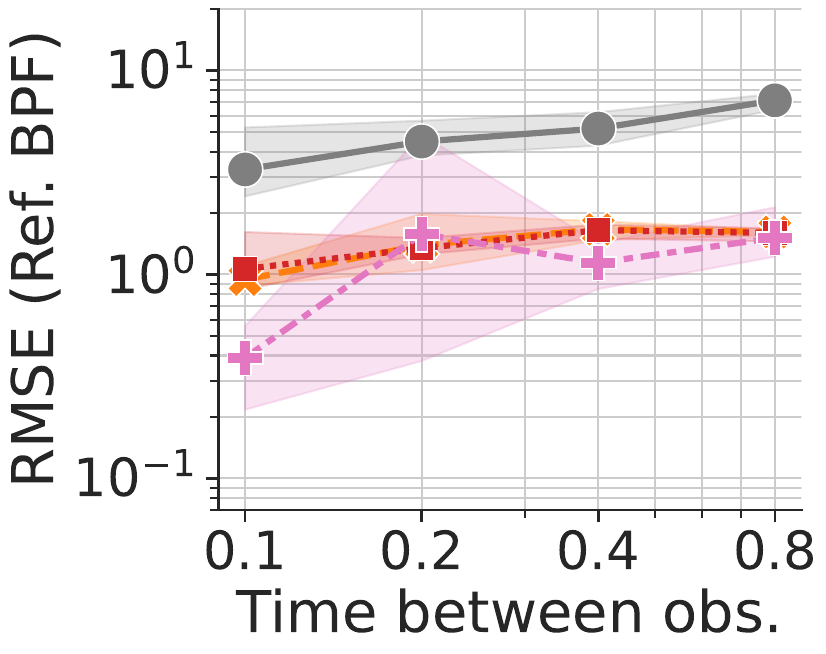}
         &
         \includegraphics[scale=.23]{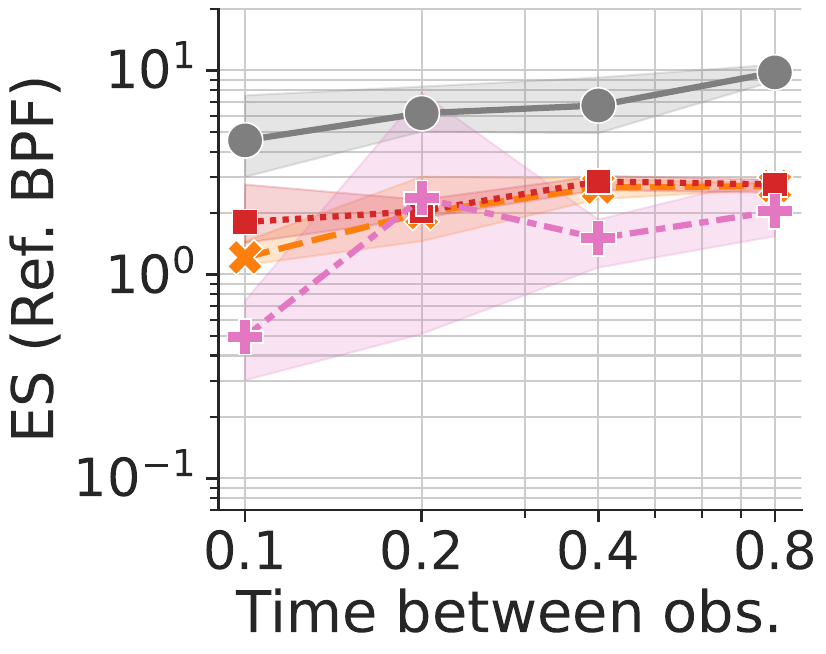}
         &
         \includegraphics[scale=.23]{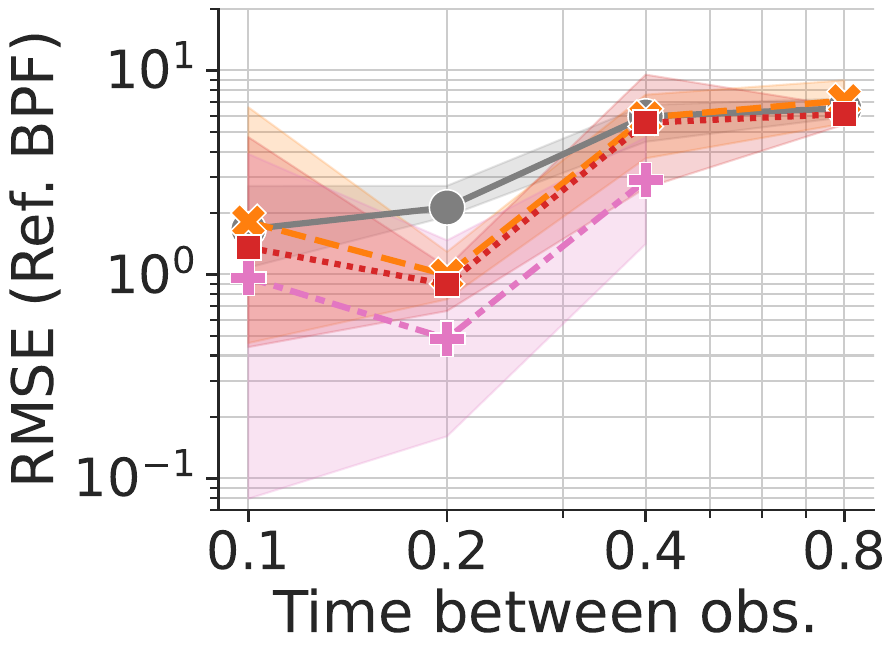}
         &
         \includegraphics[scale=.23]{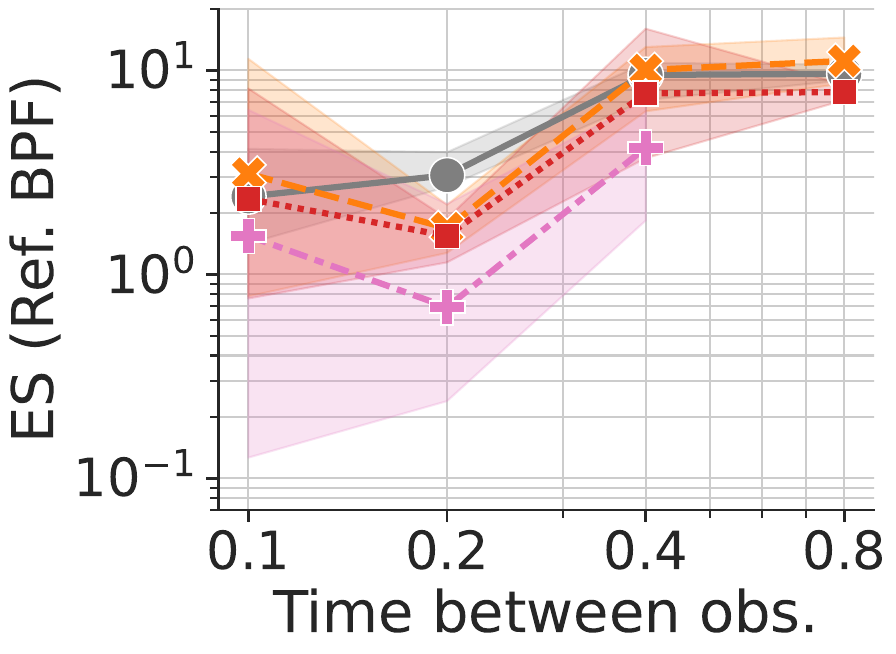}
         \\
         (a) RMSE (Identity)
         &
         (b) ES (Identity)
         &
         (c) RMSE (Arctan)
         &
         (d) ES (Arctan)
    \end{tabular}
    \caption{
        Metrics averaged over five runs varying the time between observations on Lorenz '63 with Identity and Arctan observations.
        The error bands show the min and max.
        The metrics are computed using a reference BPF with $10^6$ particles as the ground truth.
        EnSF and EnFF use $T = 10$ sampling timesteps, and EnKF does not use inflation or localization.
    }
    \label{fig:rmse-crps:lorenz63spantini2022}
\end{figure}

\subsubsection{High-Dimensional Settings}

\rev{\RefFig[fig:classical-comparison:rmse-by-dimension] compares EnSF, the EnFF variants, and the classical DA baselines on both KS and NS, across different state dimensions and for Identity and Arctan observations. Complementary to that, \RefFig[fig:classical-comparison:benchmark-timings] compares their computational times. Overall, the EnFF variants behave very similarly and consistently outperform EnSF across these experiments. Although they are not uniformly the best method on every metric and every problem setting, they provide a strong and stable accuracy--efficiency trade-off, especially as the dimension increases. Among the classical baselines, LETKF is generally the strongest performer in terms of RMSE and ES, but this comes at a substantially higher computational cost. \RefFig[fig:classical-comparison:benchmark-timings] shows that LETKF is roughly two orders of magnitude slower than the EnFF variants on two datasets for all dimensions. Moreover, the EnFF variants outperform LETKF for NS on the $256 \times 256$ grid. EnKF-PO is competitive in some settings, but its performance varies considerably with dimension, and for NS on the $256 \times 256$ grid we could not obtain a convergent run with ensemble size $N=20$. The remaining classical baselines, namely ESRF, iEnKF-PO, and BPF, are in most cases clearly worse than the EnFF variants, and in some settings no stable hyperparameter configuration could be found. Thus, while the EnFF variants are not always the single best method in raw accuracy, they deliver the most favorable overall trade-off between performance, stability, and computational cost among the methods considered here.}

\begin{figure}[!ht]
    \centering
    \scriptsize
    \includegraphics[scale=.25]{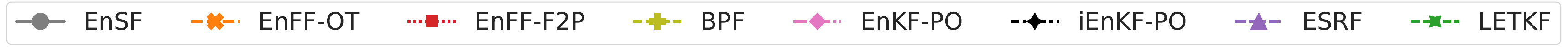}
    \begin{tabular}{cccc}
        \multirow{1}{*}[4.0em]{\rotatebox{90}{KS}}
        \includegraphics[scale=.22]{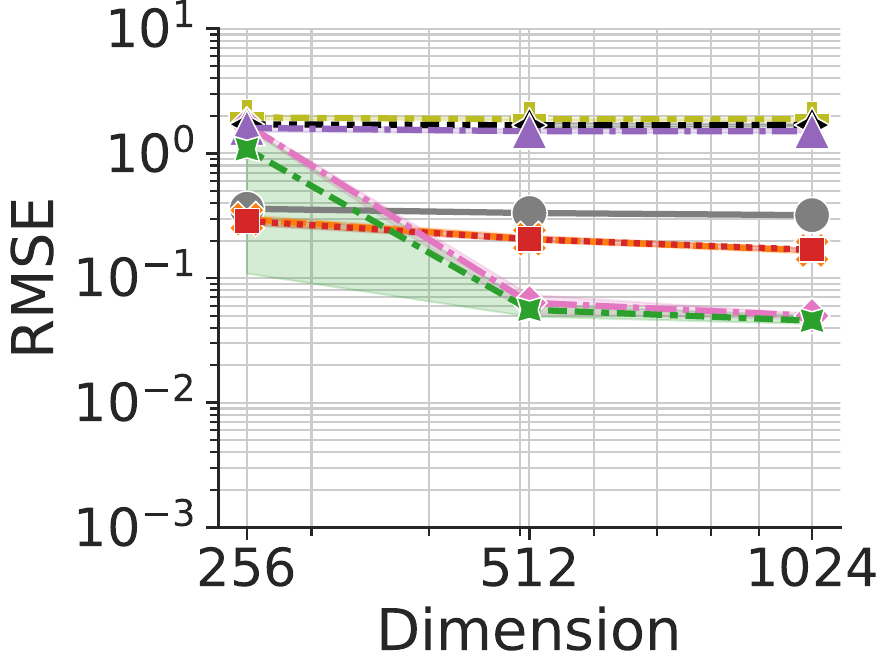}
        &
        \includegraphics[scale=.22]{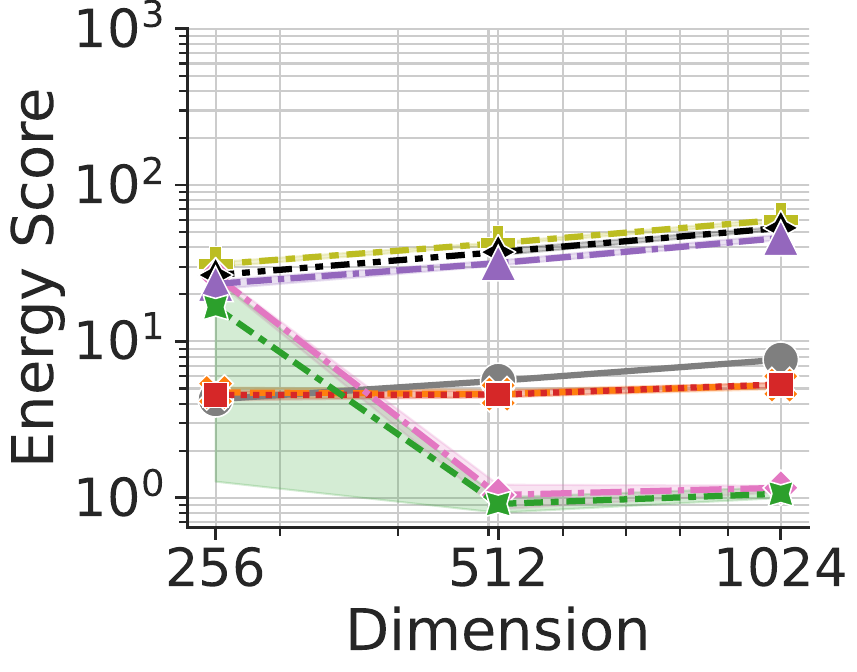}
        &
        \includegraphics[scale=.22]{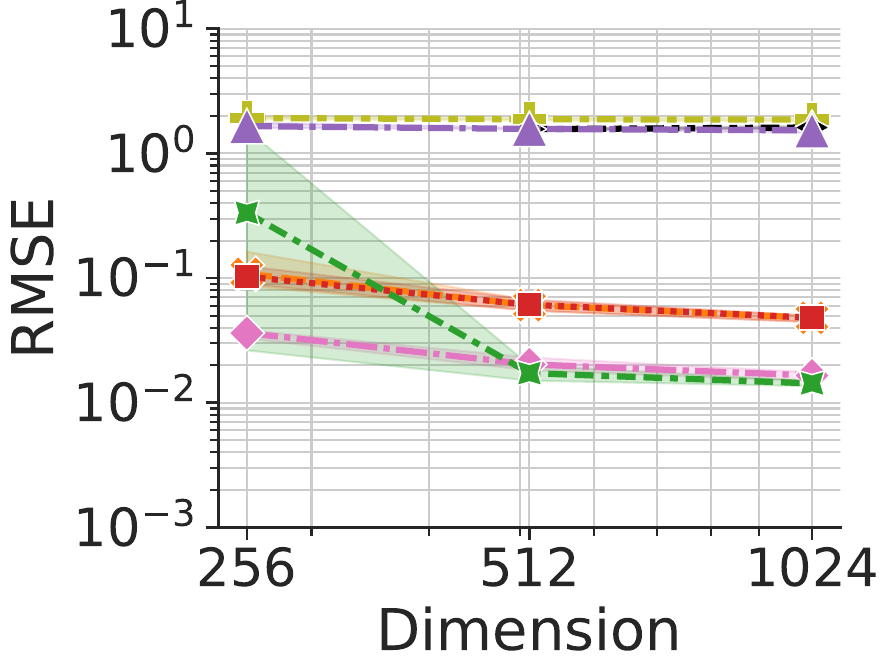}
        &
        \includegraphics[scale=.22]{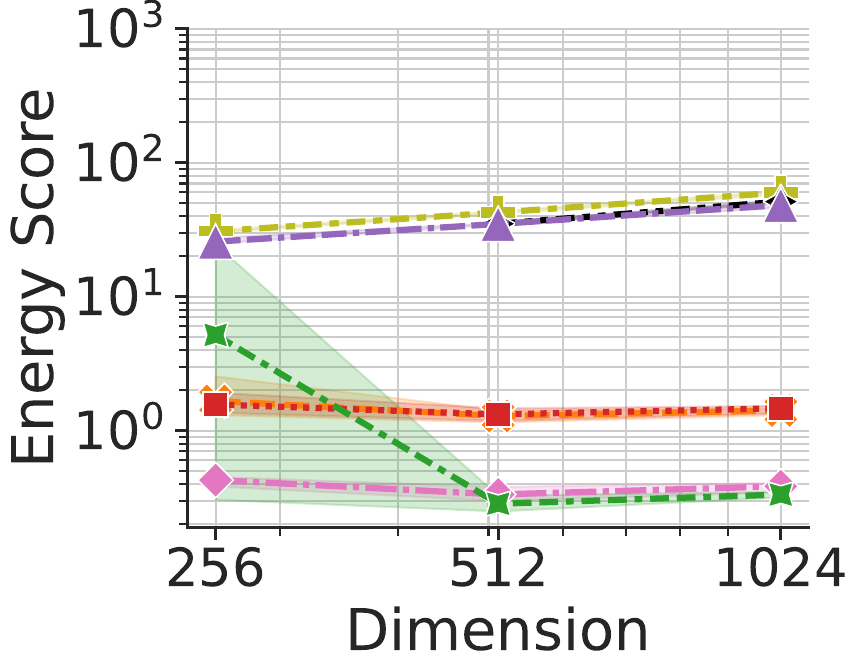}
        \\
        \multirow{1}{*}[4.0em]{\rotatebox{90}{NS}}
        \includegraphics[scale=.22]{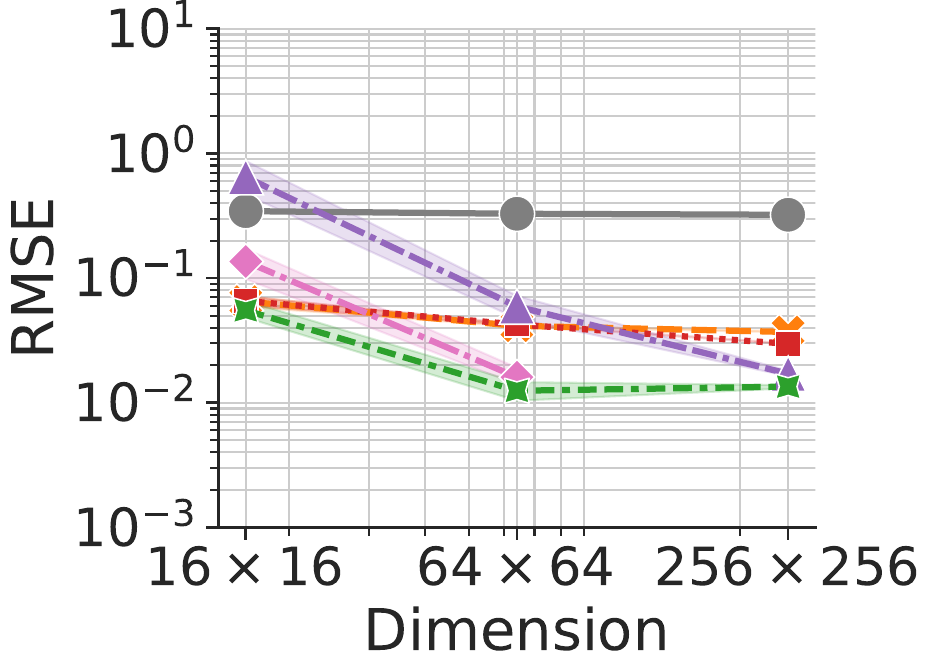}
        &
        \includegraphics[scale=.22]{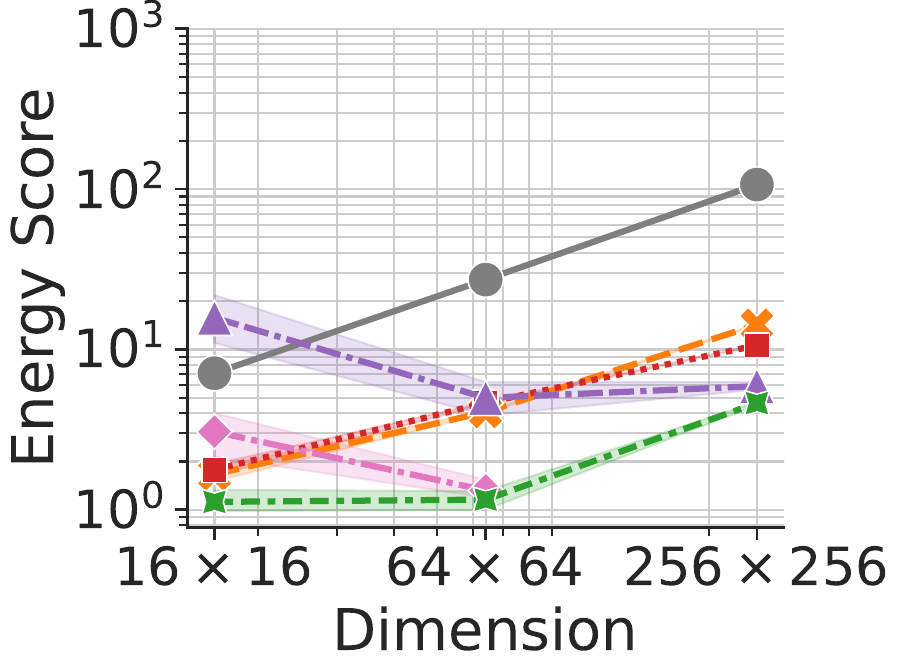}
        &
        \includegraphics[scale=.22]{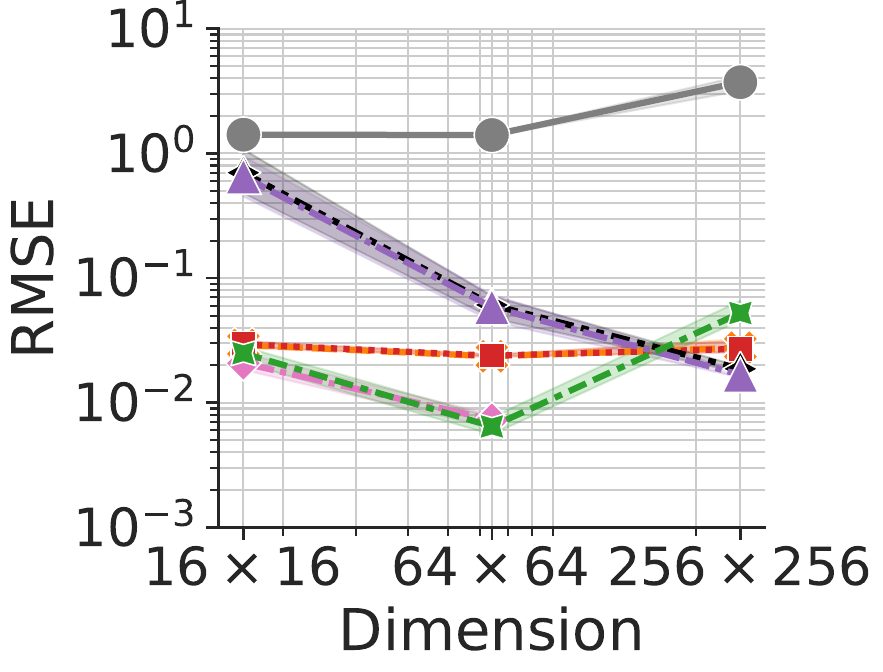}
        &
        \includegraphics[scale=.22]{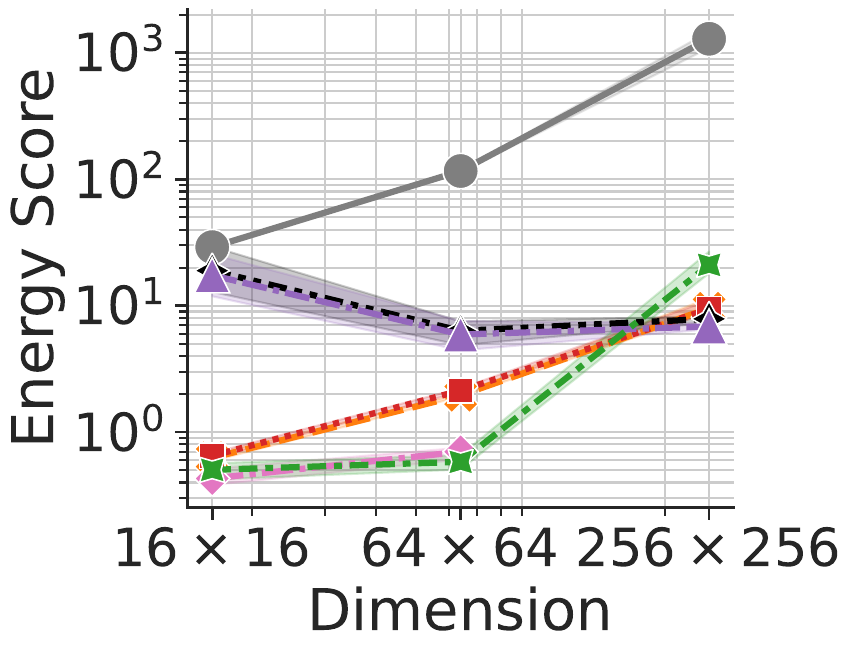}
        \\
        (a) RMSE (Identity)
         &
        (b) ES (Identity)
         &
        (c) RMSE (Arctan)
         &
        (d) ES (Arctan)
    \end{tabular}\vspace{-0.2cm}
    \caption{
    Comparison of EnSF and EnFF with classical DA baselines for KS and NS. EnSF, EnFF-OT and EnFF-F2P use 5 sampling time steps.
    The line shows the mean, and the shaded band shows the minimum and maximum \EDIT{value} observed.
    For methods not plotted, we could not find hyperparameters for which they did not diverge.
    }
    \label{fig:classical-comparison:rmse-by-dimension}
\end{figure}

\begin{figure}[!ht]
    \centering
    \scriptsize
    \includegraphics[scale=.25]{Figs/AblationSamplingTimeSteps/ClassicalComparision/KuramotoSivashinsky/ATanObs/Ablation.legend.pdf}
        \\
    \begin{tabular}{cc}
        \includegraphics[scale=.25]{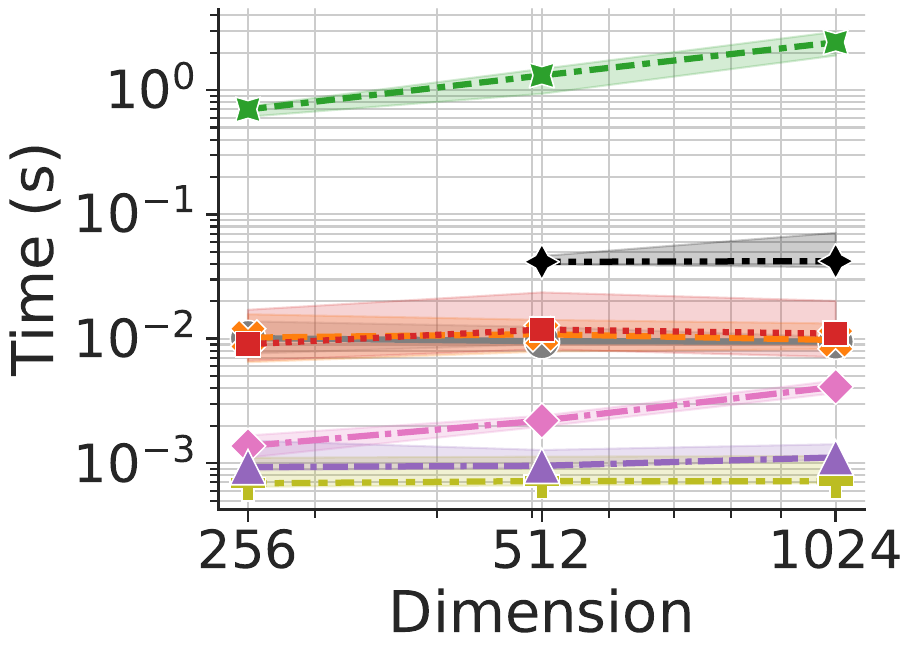}
        &
        \includegraphics[scale=.25]{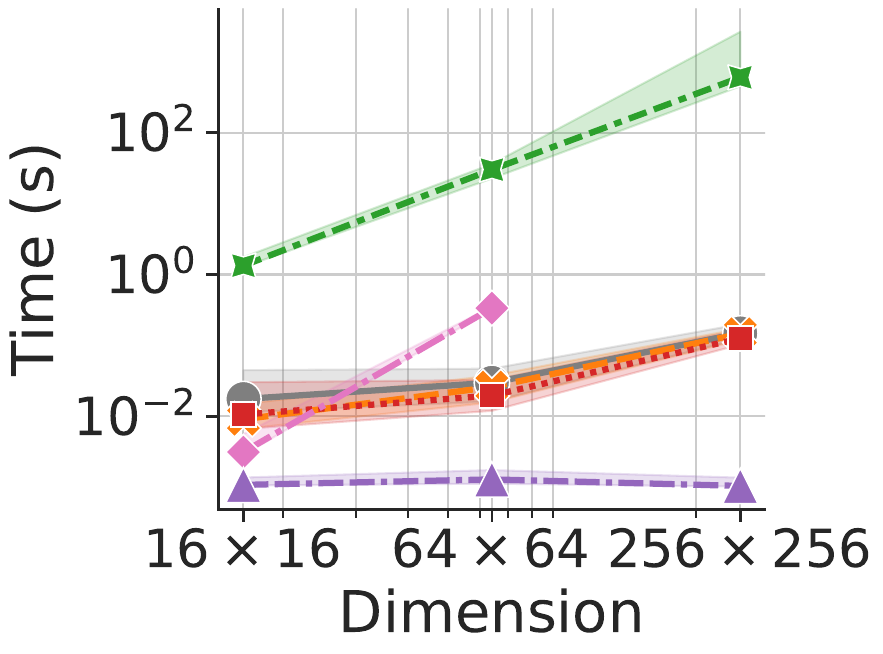}
        \\
        (a) KS
        &
        (b) NS
    \end{tabular}
    \caption{
        Time per DA step for the methods on KS and NS. EnSF and EnFF use 5 sampling time steps.
        Time is measured in seconds per DA step and is averaged over 50 DA steps. The line shows the mean, and the shaded band shows the minimum and maximum time observed.
        For methods not plotted, we could not find hyperparameters for which they did not diverge.
    }
    \label{fig:classical-comparison:benchmark-timings}
\end{figure}

\subsection{Ablation Studies}
\rev{We study the robustness of EnFF with respect to the number of sampling timesteps $T$ and its hyperparameters $\pn{\sigmamin,\lambda}$ introduced in \RefSec[sec:mc-approximation][sec:mc_guidance]. We consider two complementary views of hyperparameter robustness. First, \RefFig[fig:hyperparameter-robustness:line] shows the optimal choices of $\lambda$ and $\sigmamin$ for EnFF-OT and EnFF-F2P as functions of $T$. For each method and $T$, the optimal pair $\pn{\sigmamin,\lambda}$ is selected by minimizing the mean RMSE over the last 50 DA steps.}

\rev{The optimal hyperparameters of EnFF-F2P vary much less with $T$ than those of EnFF-OT. For Lorenz '96 and KS, the optimal values of both $\lambda$ and $\sigmamin$ for EnFF-F2P are nearly unchanged across the tested sampling timesteps. For NS, the optimal $\lambda$ of EnFF-F2P is also stable, while the optimal $\sigmamin$ exhibits some fluctuation. In contrast, the optimal hyperparameters of EnFF-OT change noticeably more as $T$ increases, especially for $\lambda$ on Lorenz '96 and KS. This indicates that EnFF-F2P has more stable hyperparameter choices across different sampling budgets. In practice, this stability suggests an efficient tuning strategy for EnFF-F2P: one can search for hyperparameters using a small number of sampling timesteps, where each trial is inexpensive, and then reuse the selected hyperparameters at larger $T$ to obtain improved performance.}

\rev{Second, \RefFig[fig:hyperparameter-robustness:contour] visualizes the hyperparameter landscape at $T = 5$ where darker regions correspond to smaller RMSE. Our hyperparameter search is implemented with Optuna \cite{akiba2019optuna}, which is not exhaustively evaluating a Cartesian grid. A larger dark region indicates that a method attains good performance for a wider range of hyperparameter configurations, and is thus more robust to hyperparameter selection. Across the tested systems, EnFF-F2P generally has broader low-RMSE regions than EnFF-OT. This difference is particularly clear for Lorenz '96, where EnFF-F2P remains accurate over a much wider portion of the hyperparameter domain. Together with the stability of the optimal hyperparameters in \RefFig[fig:hyperparameter-robustness:line], these results show that EnFF-F2P is more robust to the choice of $\sigmamin$, $\lambda$, and $T$ than EnFF-OT.}

\begin{figure}[!ht]
    \centering
    \scriptsize
    \begin{tabular}{ccc}
         \multicolumn{3}{c}{\includegraphics[scale=.3]{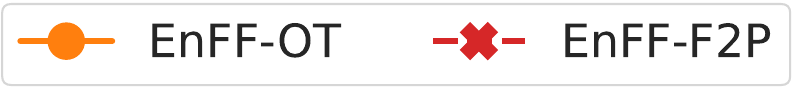}}
         \\
         \includegraphics[scale=.27]{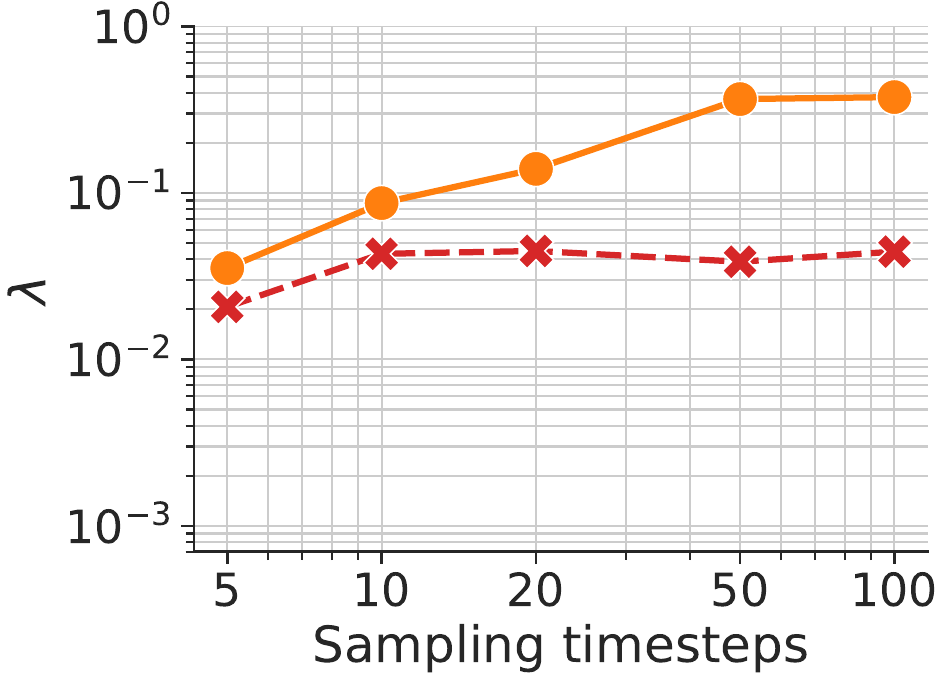}
         &
         \includegraphics[scale=.27]{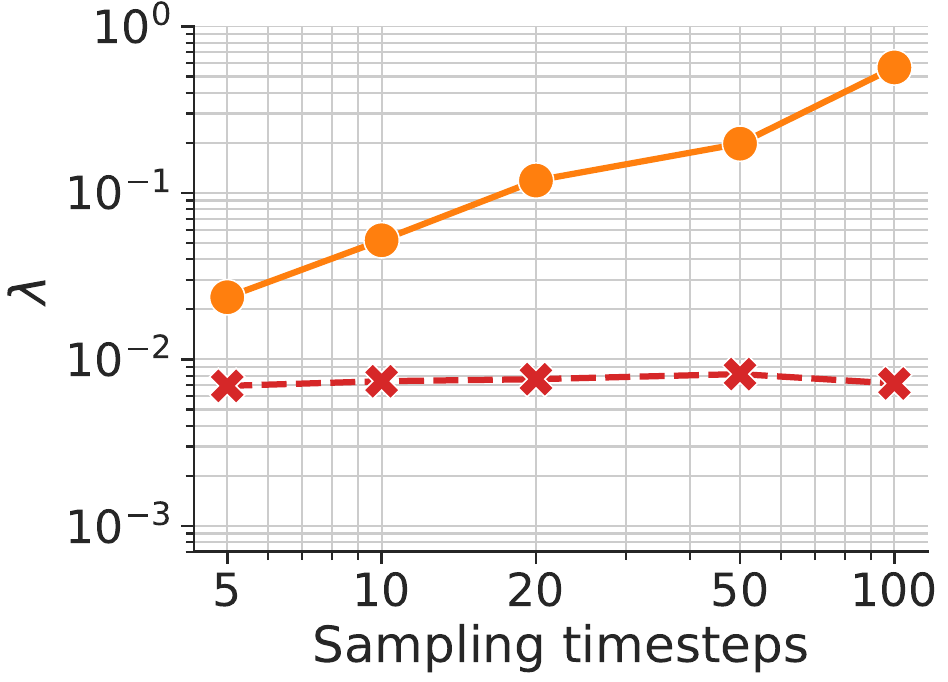}
         &
         \includegraphics[scale=.27]{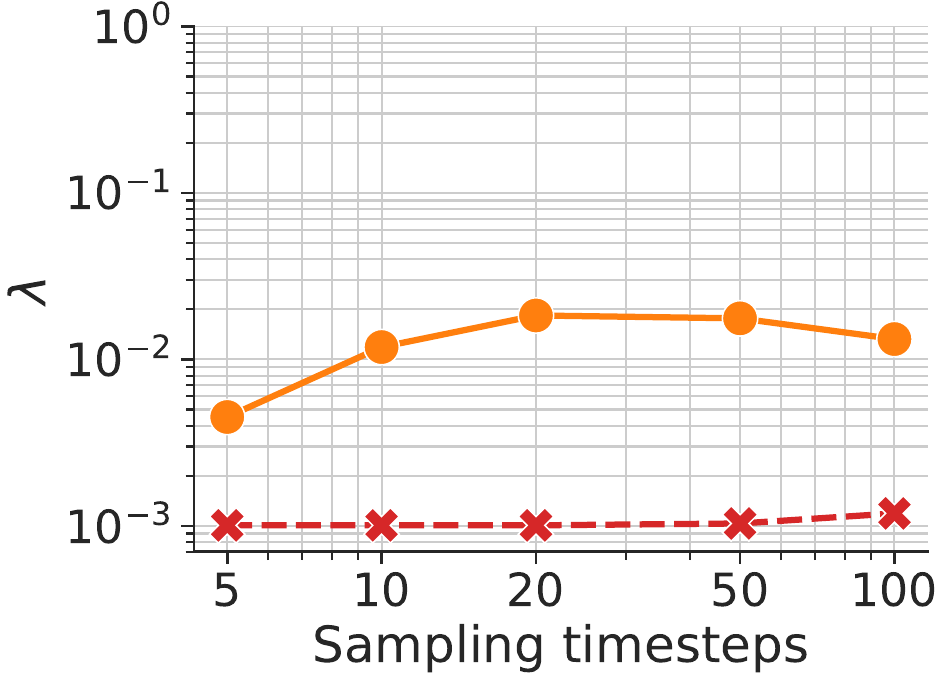}
         \\
         \includegraphics[scale=.27]{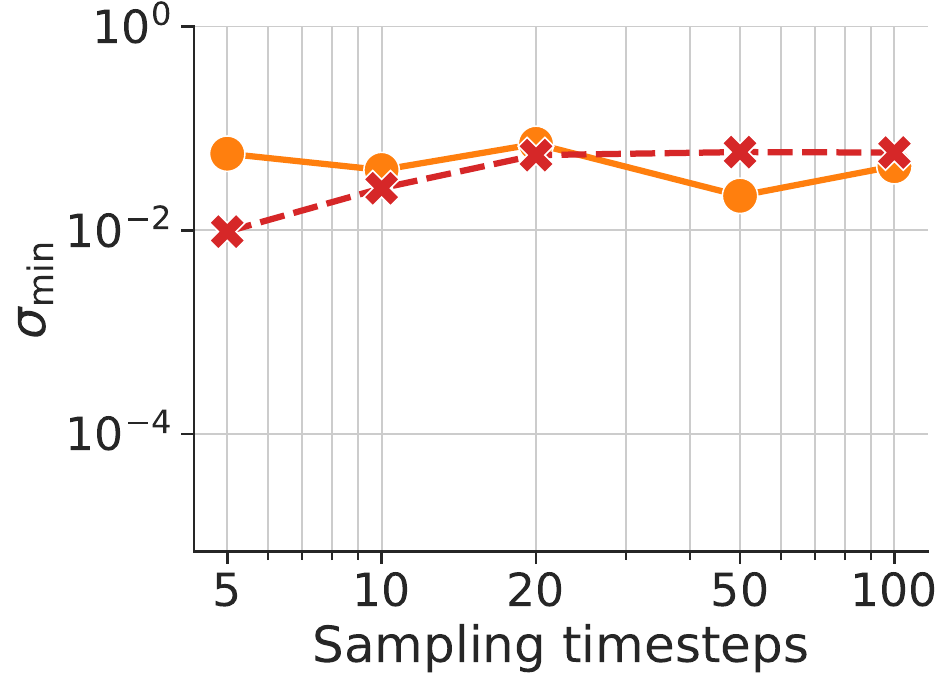}
         &
         \includegraphics[scale=.27]{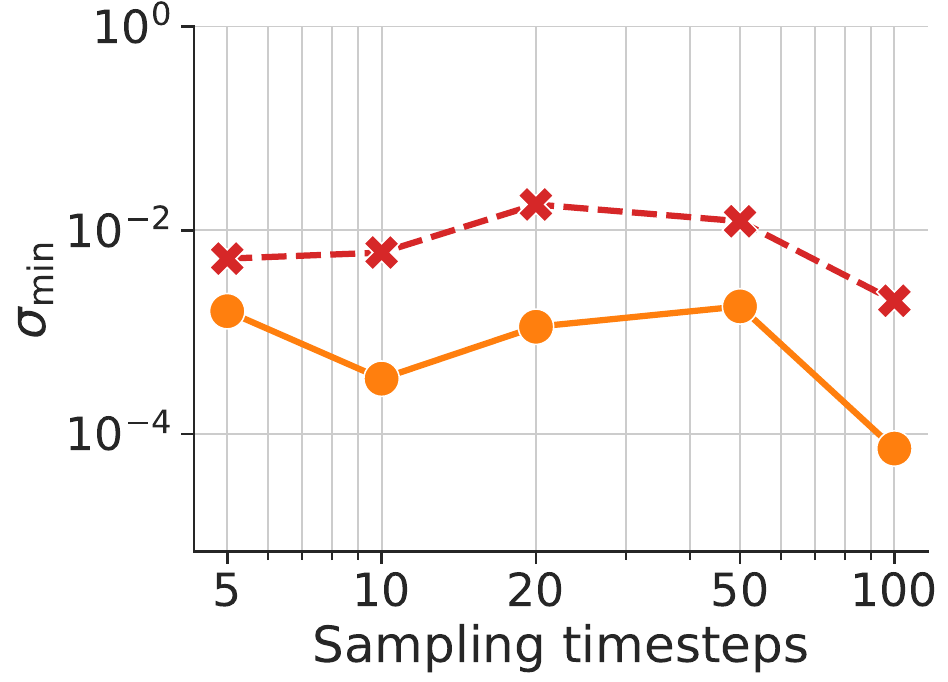}
         &
         \includegraphics[scale=.27]{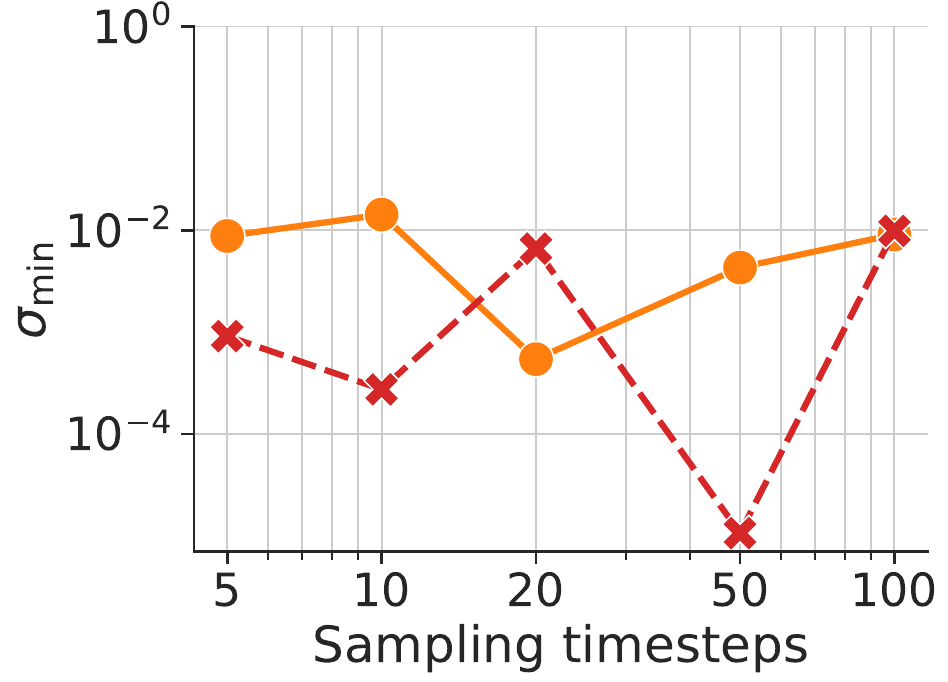}
         \\
         (a) Lorenz '96 (Arctan)
         &
         (b) KS 1,024 (Arctan)
         &
         (c) NS-256 (Arctan)
    \end{tabular}
    \caption{
        Optimal hyperparameters $\lambda$ and $\sigmamin$ of the EnFF variants varying sampling timesteps $T$.
    }
    \label{fig:hyperparameter-robustness:line}
\end{figure}

\begin{figure}[!ht]
    \centering
    \scriptsize
    \begin{tabular}{cccc}
        \multirow{1}{*}[5.5em]{\rotatebox{90}{\scriptsize EnFF-OT}}
        \includegraphics[scale=.3]{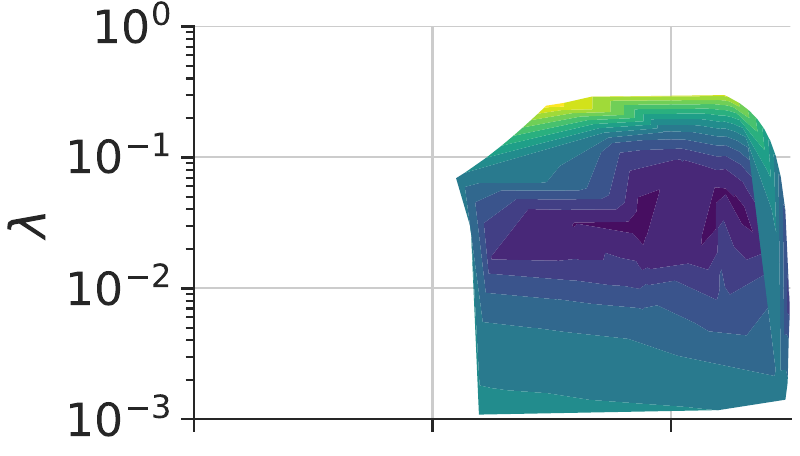}
        &
        \includegraphics[scale=.3]{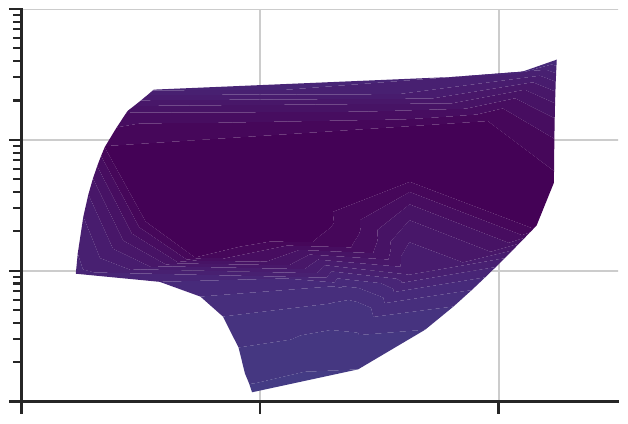}
        &
        \includegraphics[scale=.3]{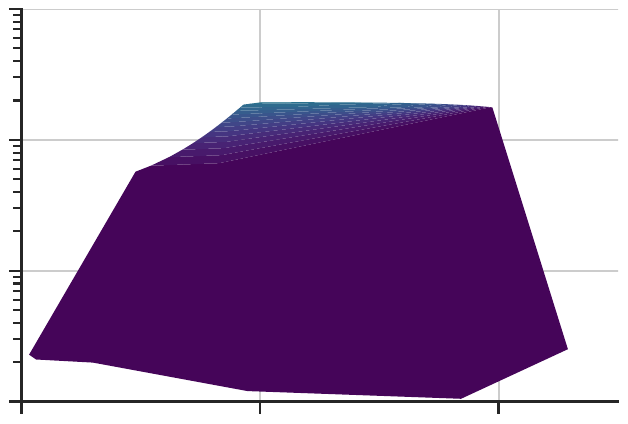}
        &
        \multirow{2}{*}[6.5em]{\includegraphics[scale=.3]{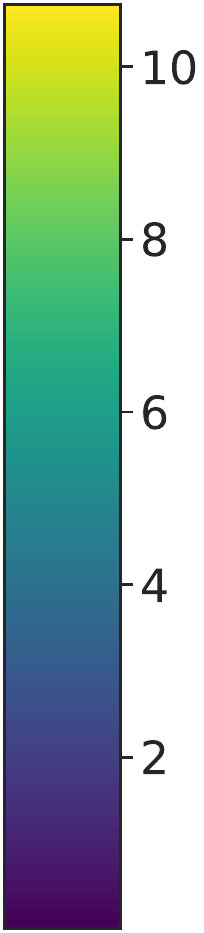}}
        \\
        \multirow{1}{*}[7.0em]{\rotatebox{90}{\scriptsize EnFF-F2P}}
        \includegraphics[scale=.3]{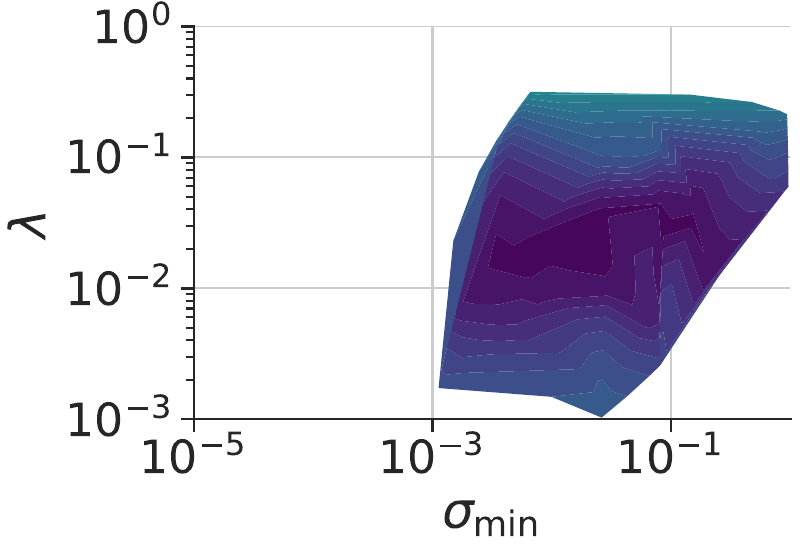}
        &
        \includegraphics[scale=.3]{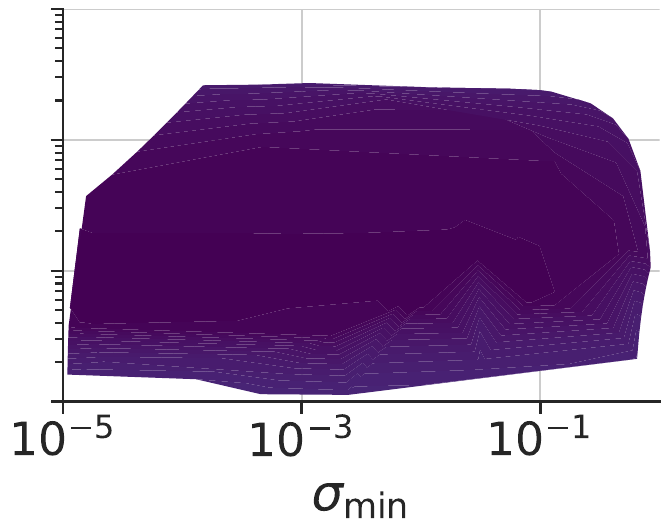}
        &
        \includegraphics[scale=.3]{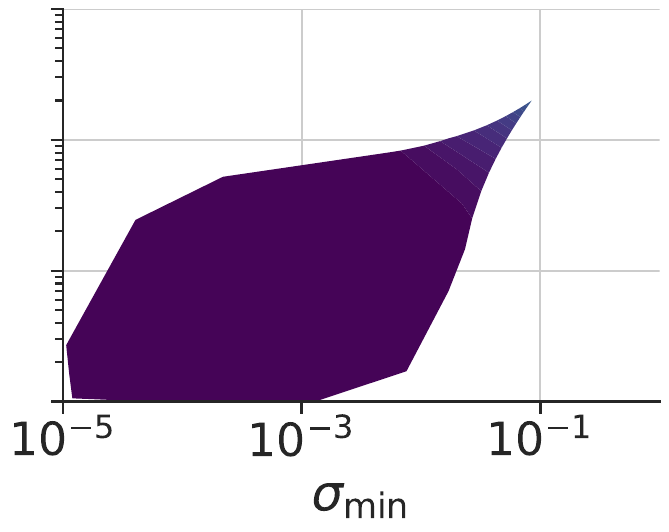}
        \\
        (a) Lorenz '96 (Arctan)
        &
        (b) KS 1,024 (Arctan)
        &
        (c) NS-256 (Arctan)
    \end{tabular}
    \caption{
        Mean RMSE over the last 50 DA steps varying $\lambda$ and $\sigmamin$ of the EnFF variants at $T = 5$.
    }
    \label{fig:hyperparameter-robustness:contour}
\end{figure}

\section{Concluding Remarks}\label{sec:Conclusion}
We introduced EnFF, a training-free DA framework based on FM principles, using MC estimators for the marginal VF and a localized guidance strategy for filtering. EnFF offers faster sampling and greater flexibility in VF design compared to score-based methods like EnSF. Theoretically, we show that it generalizes classical filters such as BPF and EnKF. Empirical results on high-dimensional benchmarks confirm EnFF’s superior accuracy-cost tradeoff, highlighting its robustness and scalability.
For future work, to further improve scalability, we plan to reduce the state dimension by performing DA in a latent space, similar to the works \cite{si2024latent, xiao2024ld}.

\bibliographystyle{siamplain}
\bibliography{references}

\appendix

\section{Additional Details on Classical DA Algorithms}
\label{app:DA}

This section provides the working details for BPF and EnKF.

\subsection{BPF}
\label{app:BPF}
BPF approximates $p(\vx_j| \vy_{1:j})$ by an empirical measure supported on a set of particles. Starting from $\{\vx_0^{(n)}\}_{n=1}^N \stackrel{\text{i.i.d.}}{\sim} \gN(\vect{m}_0,\mat{C}_0)$, it iterates:
\begin{align}
\text{Prediction:}\quad
&\hat{\vx}_j^{(n)} \sim p(\vx_j | \vx_{j-1}^{(n)})
\;\Longleftrightarrow\; \hat{\vx}_j^{(n)}=\vect{\psi}(\vx_{j-1}^{(n)})+\vect{\xi}_{j-1}^{(n)},\;\; \vect{\xi}_{j-1}^{(n)}\sim\gN(\vect{0},\mat{\Sigma}), \\
\text{Weighting:}\quad
&w_j^{(n)} \propto p(\vy_j| \hat{\vx}_j^{(n)})
\;\Longleftrightarrow\;
w_j^{(n)} \propto \exp\!\bigl(-\tfrac12\|\vy_j-\vect{h}(\hat{\vx}_j^{(n)})\|_{\mat{\Gamma}}^2\bigr), \\
\text{Resampling:}\quad
&\{\vx_j^{(n)}\}_{n=1}^N \sim \text{Resample}\bigl(\{\hat{\vx}_j^{(n)}, w_j^{(n)}\}_{n=1}^N\bigr),
\quad w_j^{(n)} \leftarrow \tfrac1N.
\end{align}

\subsection{EnKF}
\label{app:EnKF}

EnKF approximates the filtering posterior using an ensemble
$\{\vx_j^{(n)}\}_{n=1}^N$ and its sample moments. Given
$\{\vx_{j-1}^{(n)}\}_{n=1}^N$, the perturbed-observation EnKF update is as follows.

\paragraph{Prediction}
For each ensemble member,
$\vect{\xi}_{j-1}^{(n)} \stackrel{\text{i.i.d.}}{\sim} \gN(\vect{0},\mat{\Sigma})$
and
$\vect{\eta}_{j}^{(n)} \stackrel{\text{i.i.d.}}{\sim} \gN(\vect{0},\mat{\Gamma})$, then
\begin{equation}
\label{eq:enkf_prediction}
\hat{\vx}_j^{(n)}
= \vect{\psi}\bigl(\vx_{j-1}^{(n)}\bigr) + \vect{\xi}_{j-1}^{(n)}, \quad
\hat{\vy}_j^{(n)}
= \vect{h}\bigl(\hat{\vx}_j^{(n)}\bigr) + \vect{\eta}_{j}^{(n)} .
\end{equation}

\paragraph{Sample Covariances}
Let
$\bar{\hat{\vx}}_j = \frac{1}{N}\sum_{n=1}^N \hat{\vx}_j^{(n)}$
and
$\bar{\hat{\vy}}_j = \frac{1}{N}\sum_{n=1}^N \hat{\vy}_j^{(n)}$. Define
\begin{equation}
\label{eq:enkf_covs}
\hat{\mat{C}}^{xy}_j
= \frac{1}{N}\sum_{n=1}^N
\bigl(\hat{\vx}_j^{(n)} - \bar{\hat{\vx}}_j\bigr)
\bigl(\hat{\vy}_j^{(n)} - \bar{\hat{\vy}}_j\bigr)^\Transpose, \quad
\hat{\mat{C}}^{yy}_j
= \frac{1}{N}\sum_{n=1}^N
\bigl(\hat{\vy}_j^{(n)} - \bar{\hat{\vy}}_j\bigr)
\bigl(\hat{\vy}_j^{(n)} - \bar{\hat{\vy}}_j\bigr)^\Transpose .
\end{equation}

\paragraph{Analysis}
The Kalman gain and analysis update are
\begin{equation}
\label{eq:enkf_analysis}
\mat{K}_j
= \hat{\mat{C}}^{xy}_j
\bigl(\hat{\mat{C}}^{yy}_j\bigr)^{-1}, \quad
\vx_j^{(n)}
= \hat{\vx}_j^{(n)}
+ \mat{K}_j\bigl(\vy_j - \hat{\vy}_j^{(n)}\bigr),
\qquad n = 1,\ldots,N .
\end{equation}

\section{Auxiliary Materials}\label{app:prelims}
This appendix provides the auxiliary materials referenced in the main content.

\subsection{Definitions}
\begin{definition}[Pushforward Measure]\label{def:pushforward}
Let $(X,\mathcal{F})$ and $(Y,\mathcal{G})$ be measurable spaces, i.e., $X$ and $Y$ are sets and $\mathcal{F}$ and $\mathcal{G}$ are $\sigma$-algebras on $X$ and $Y$, respectively.
Let $\mu$ be a measure on $(X,\mathcal{F})$ and let $T\colon(X,\mathcal{F})\to(Y,\mathcal{G})$ be measurable.
The pushforward of $\mu$ by $T$ is the measure $T_{\#}\mu$ on $Y$ defined by
\[
T_{\#}\mu(B)\;:=\;\mu\bigl(T^{-1}(B)\bigr),\qquad B\in\mathcal{G}.
\]
If $\mu$ is a probability measure (i.e., $\mu\pn{X} = 1$) and $Z\sim\mu$, then the law of $T(Z)$ is $T_{\#}\mu$.
\end{definition}

\begin{definition}[Weak Convergence]\label{def:weak_convergence}
Let $\set{\mu_n}_{n\in\mathbb{N}}$ and $\mu$ be probability measures on $\mathbb{R}^d$. We write $\mu_n \Longrightarrow \mu$ as $n\to\infty$ if for every bounded continuous function $\varphi\colon\mathbb{R}^d\to\mathbb{R}$,
\[
\int_{\mathbb{R}^d} \varphi(\vx)\,\mu_n(\dif\vx)\;\rightarrow\;\int_{\mathbb{R}^d} \varphi(\vx)\,\mu(\dif\vx).
\]
\end{definition}

\subsection{Lemmas}
\begin{lemma}[\cite{billingsley2013convergence}]\label{lem:weak-convergence}
    Given a Gaussian mixture $P_\sigma(\dif \vz) = \sum_{n=1}^N w_n \mathcal{N}(\vz | \vz_n, \sigma^2\mat{I}) \dif \vz$ and an empirical measure $P(\dif \vz) = \sum_{n=1}^N w_n \delta_{\vz_n}(\dif \vz)$, we have the weak convergence
    \begin{align}
        P_\sigma \Longrightarrow  P\ \text{as} \ \sigma \rightarrow 0.
    \end{align}
\end{lemma}
%\begin{proof}
%See \cite{billingsley2013convergence}.
%\end{proof}

\begin{lemma}
\label{lem:invertibility_I_minus_tKH}
\EDIT{Assume} the observation operator $\vect{h}(\cdot)$ is linear, i.e. $\vect{h}(\vx) = \mat{H}\vx$. For EnKF, let $\mat{K}_j$ denote the Kalman gain at the timestep $j\in \sZ^+$. Then, for any scalar $t \in [0,1]$, the matrix $(\EDIT{\mat{I}} - t \mat{K}_j \mat{H})$ is invertible.
\end{lemma}

\begin{proof}
Firstly, for symmetric matrices $\mat{A},\mat{B}$, we say $\mat{A}\succ \mat{B}$ if $\mat{A}-\mat{B}$ is symmetric positive definite and $\mat{A}\succeq \mat{B}$ if $\mat{A}-\mat{B}$ is symmetric positive semi-definite.

Let $\mat{K}_j = \mat{C} \mat{H}^\Transpose (\mat{H} \mat{C} \mat{H}^\Transpose + \mat{\Gamma})^{-1}$ be the Kalman gain at timestep $j$, where $\mat{C} \succeq \mat{0}$ is the covariance of ensemble predictive states and $\mat{\Gamma}\succ \mat{0}$ is the observation error covariance. Let $\mat{S} = \mat{H} \mat{C} \mat{H}^\Transpose + \mat{\Gamma} \succ \mat{0}$, and $\mat{Q} = \mat{H} \mat{C} \mat{H}^\Transpose \succeq \mat{0}$.

Since $t\in[0,1]$, it suffices to show that all eigenvalues $\lambda$ of $\mat{K}_j \mat{H}$ satisfy $0 \le \lambda < 1$.
Since $\mat{K}_j \mat{H} = \mat{C} \mat{H}^\Transpose \mat{S}^{-1} \mat{H}$, the nonzero eigenvalues of $\mat{K}_j \mat{H}\in\R^{d\times d}$ are the same as the nonzero eigenvalues of $\mat{Q} \mat{S}^{-1} = \mat{H} \mat{C} \mat{H}^\Transpose \mat{S}^{-1}\in \R^{d_y\times d_y}$.

Since $\mat{Q} \mat{S}^{-1}$ is similar to $\mat{S}^{-1/2} \mat{Q} \mat{S}^{-1/2}$, and $\mat{S}^{-1/2} \mat{Q} \mat{S}^{-1/2} \succeq \mat{0}$ (as $\mat{Q} \succeq \mat{0}$ and $\mat{S} \succ \mat{0}$), the eigenvalues $\lambda$ of $\mat{Q} \mat{S}^{-1}$ are real and satisfy $\lambda \ge 0$.
We can write $\mat{S}^{-1/2} \mat{Q} \mat{S}^{-1/2} = \mat{S}^{-1/2} (\mat{S} - \mat{\Gamma}) \mat{S}^{-1/2} = \mat{I} - \tilde{\mat{\Gamma}}$ where
$\tilde{\mat{\Gamma}} = \mat{S}^{-1/2} \mat{\Gamma} \mat{S}^{-1/2}$.
Given $\mat{\Gamma} \succ \mat{0}$ (and $\mat{S} \succ \mat{0}$), it follows that $\tilde{\mat{\Gamma}} \succ \mat{0}$.
Furthermore, since $\mat{Q} \succeq \mat{0}$, we have $\mat{S} = \mat{Q} + \mat{\Gamma} \succeq \mat{\Gamma}$. As $\mat{S} \succ \mat{0}$, pre- and post-multiplying by $\mat{S}^{-1/2}$ (which is also $\succ \mat{0}$) yields $\mat{I} = \mat{S}^{-1/2} \mat{S} \mat{S}^{-1/2} \succeq \mat{S}^{-1/2} \mat{\Gamma} \mat{S}^{-1/2} = \tilde{\mat{\Gamma}}$.
Thus, we have $\mat{I} \succeq \tilde{\mat{\Gamma}} \succ \mat{0}$. This implies that for any eigenvalue $\mu$ of $\tilde{\mat{\Gamma}}$, it must satisfy $0 < \mu \le 1$.

The nonzero eigenvalues of $\mat{K}_j \mat{H}$ (that correspond to those of $\mat{Q} \mat{S}^{-1}$) are given by $1 - \mu\in [0,1)$.
Therefore, $\mat{I} - t\mat{K}_j\mat{H}$ is invertible for any $t\in[0,1]$.
\end{proof}

\begin{lemma}[VF for Composing a Flow with an Affine Transformation]
\label{lem:affine_flow}
Let $\{\vu_t\}_{t\in [0,1]}$ be a time-dependent VF on $\R^d$ generating the flow $\vect{\phi}_t$ such that $\frac{\mathrm{d}}{\mathrm{d}t}\vect{\phi}_t(\vz) = \vu_t(\vect{\phi}_t(\vz))$ and $\vect{\phi}_0(\vz)=\vz$. Let $\vect{T}(\vz) = \mat{A}\vz + \vect{b}$ be an affine transformation, where $\mat{A} \in \mathbb{R}^{d \times d}$, and $\vect{b} \in \mathbb{R}^d$.
A VF $\tilde{\vu}_t(\vz)$ generating a flow $\tilde{\vect{\phi}}_t$ with $\tilde{\vect{\phi}}_0(\vz)=\vz$ and target map $\tilde{\vect{\phi}}_1(\vz) = \vect{T}(\vect{\phi}_1(\vz))$ is given by:
\begin{equation}
    \tilde{\vu}_t(\vz) = (\mat{A}-\mat{I})\mat{A}_t^{-1}(\vz - \vect{b}_t) + \mat{A}_t \vu_t(\mat{A}_t^{-1}(\vz - \vect{b}_t)) + \vect{b}
\end{equation}
where $\mat{A}_t = (1-t)\mat{I}+t\mat{A}$ and $\vect{b}_t = t\vect{b}$, provided $\mat{A}_t$ is invertible for all $t \in [0,1]$.
\end{lemma}

\begin{proof}
Let $\vect{\phi}_t(\vz)$ be the flow generated by $\vu_t$, with $\vect{\phi}_0(\vz) = \vz$. We seek a flow $\tilde{\vect{\phi}}_t(\vz)$ such that $\tilde{\vect{\phi}}_0(\vz) = \vz$ and its endpoint is $\tilde{\vect{\phi}}_1(\vz) = \mat{A}\vect{\phi}_1(\vz) + \vect{b}$.

Define $\mat{A}_t = (1-t)\mat{I} + t\mat{A}$ and $\vect{b}_t = t\vect{b}$. Consider the flow candidate $\tilde{\vect{\phi}}_t(\vz) = \mat{A}_t \vect{\phi}_t(\vz) + \vect{b}_t$. This construction satisfies the boundary conditions:
$\tilde{\vect{\phi}}_0(\vz) = \mat{A}_0 \vect{\phi}_0(\vz) + \vect{b}_0 = \mat{I}\vz + \vect{0} = \vz$, and
$\tilde{\vect{\phi}}_1(\vz) = \mat{A}_1 \EDIT{\vect{\phi}_1}(\vz) + \vect{b}_1 = \mat{A}\vect{\phi}_1(\vz) + \vect{b}$.

The VF $\tilde{\vu}_t$ evaluated at $\tilde{\vect{\phi}}_t(\vz)$ is the time derivative of $\tilde{\vect{\phi}}_t(\vz)$:
\begin{equation}
\begin{split}
\frac{\mathrm{d}}{\mathrm{d}t}\tilde{\vect{\phi}}_t(\vz) = \frac{\mathrm{d}}{\mathrm{d}t} (\mat{A}_t \vect{\phi}_t(\vz) + \vect{b}_t) = \dot{\mat{A}}_t \vect{\phi}_t(\vz) + \mat{A}_t \dot{\vect{\phi}}_t(\vz) + \dot{\vect{b}}_t.
\end{split}
\end{equation}
Using $\dot{\mat{A}}_t = \mat{A}-\mat{I}$, $\dot{\vect{b}}_t = \vect{b}$, and the definition $\dot{\vect{\phi}}_t(\vz) = \vu_t(\vect{\phi}_t(\vz))$, this becomes:
\begin{equation}
\begin{split}
    \frac{\mathrm{d}}{\mathrm{d}t}\tilde{\vect{\phi}}_t(\vz) &= (\mat{A}-\mat{I})\vect{\phi}_t(\vz) + \mat{A}_t \vu_t(\vect{\phi}_t(\vz)) + \vect{b}\\
    &= (\mat{A}-\mat{I}) \mat{A}_t^{-1}(\EDIT{\tilde{\vect{\phi}}_t(\vz)} - \vect{b}_t) + \mat{A}_t \vu_t( \mat{A}_t^{-1}(\EDIT{\tilde{\vect{\phi}}_t(\vz)} - \vect{b}_t)) + \vect{b},
\end{split}
\end{equation}
where $\mat{A}_t$ is invertible for $t\in [0,1]$. This implies that
\begin{equation}
    \tilde{\vu}_t(\vz) = (\mat{A}-\mat{I})\mat{A}_t^{-1}(\vz - \vect{b}_t) + \mat{A}_t \vu_t(\mat{A}_t^{-1}(\vz - \vect{b}_t)) + \vect{b}.
\end{equation}
This is the expression stated in the lemma, with $\mat{A}_t=(1-t)\mat{I}+t\mat{A}$ and $\vect{b}_t=t\vect{b}$.
\end{proof}

\begin{lemma}[Convolution-Smoothed Probability Path]
\label{lem:conv_smoothed_path}
    Let $\{\vu_t\}_{t\in [0,1]}$ be a time-dependent VF on $\R^d$ with corresponding probability path $\{p_t\}_{t\in[0,1]}$ such that the continuity equation
    \begin{equation}
    \frac{\partial}{\partial t}p_t(\vz) +\nabla_{\vz}\cdot\bigl[p_t(\vz)\,\vu_t(\vz)\bigr]=0,
    \quad
    p_{0}(\vz)=\rho_0(\vz), \quad \vz\in\R^d ,\quad t\in(0,1].
    \end{equation}
    is satisfied, $\set{p_t\vu_t}_{t \in \spn{0, 1}} \subset L^1\pn{\R^d;\R^d} = \set{\vf\colon \Reals^d \to \Reals^d \mid \vx \mapsto \norm{\vf\pn{\vx}}_1 \in L^1\pn{\Reals^d;\Reals}}$ and $\set{\delfrac{\vy_i}\pn{p_t\vu_t}}_{t \in \spn{0, 1}} \subset L^1\pn{\Reals^d;\Reals}$ for $i \in [\,d\,]$.
    For a positive definite matrix $\mat{\Lambda}$, define the convolution-smoothed density by
    \begin{equation}
        \tilde{p}_t(\vz) = p_t(\vz) \ast \gN(\vz|\vect{0}, t\mat{\Lambda}),
    \end{equation}
    where $\ast$ denotes the convolution operator. Then there \EDIT{exists} a well-defined VF $\EDIT{\tilde{\vu}_t}$ such that
    \begin{equation}\label{eq:tilde_cont_equation}
    \frac{\partial}{\partial t}\tilde{p}_t(\vz) +\nabla_{\vz}\cdot\bigl[\tilde{p}_t(\vz)\,\tilde{\vu}_t(\vz)\bigr]=0,
    \quad
    \tilde{p}_{0}(\vz)=\rho_0(\vz), \quad \vz\in\R^d ,\quad t\in (0,1].
    \end{equation}
\end{lemma}
\begin{proof}
    Let $G_t(\vz) = \gN(\vz|\vect{0},t\mat{\Lambda})$, and then $\delfrac{t} G_t = \frac{1}{2}\nabla\cdot(\mat{\Lambda}\nabla G_t)$.
    Now,
    \begin{equation}\label{eq:partial_t_tilde_pt}
    \begin{split}
        &\delfrac{t}\pn{\tilde{p}_t}\pn{\vz} = \delfrac{t}\pn{p_t \ast G_t}\pn{\vz}
        = -[\nabla\cdot(p_t\vu_t)]\ast G_t + \frac{1}{2}\nabla\cdot [p_t \ast \mat{\Lambda}\nabla G_t]
    \end{split}
    \end{equation}
    Then we show that $[\nabla\cdot(p_t\vu_t)]\ast G_t = \nabla\cdot [(p_t\vu_t)\ast G_t].$
    Let $\vect{f}:=p_t \vu_t\in L^1(\R^d;\R^d)$ and $\vect{f}_i$ be the component of dimension $i$. Fix $\vz\in\R^d$, and for $i \in [\,d\,]$, we set
    \begin{equation}
    I_i(\vz):=\int_{\R^d}\spn*{\delfrac{\vy_i}\vect{f}_i(\vy)}\,G_t(\vz-\vy) \wrt{\vy}.
    \end{equation}
    Let $R > 0$ and $\chi_R \in C_c^\infty\pn{\Reals^d}$ be a cutoff function such that $\chi_R \equiv 1$ on $B_R\pn{\vect{0}}$, $\chi_R \equiv 0$ on $\Reals^d \setminus B_{2R}\pn{\vect{0}}$, $\norm{\chi_R}_\infty = 1$, and $\norm{\nabla\chi_R}_\infty \leq C/R$ for some $C > 0$.
    We define
    \begin{equation}
    I_{i,R}(\vz):=\int_{\R^d}\spn*{\delfrac{\vy_i}\vect{f}_i(\vy)}\,\chi_R(\vy)G_t(\vz-\vy) \wrt{\vy}.
    \end{equation}
    Fixing $\vz$,
    \begin{equation}
        \abs*{\spn*{\delfrac{\vy_i}\vect{f}_i(\vy)}\,\chi_R(\vy)G_t(\vz-\vy)}
        \leq \abs*{\delfrac{\vy_i}\vect{f}_i(\vy)}\norm{G_t\pn{\vz-\vy}}_\infty.
    \end{equation}
    By the dominated convergence theorem, we have $I_{i,R}(\vz)\to I_i(\vz)$ as $R\to\infty$. We integrate by parts:
    \begin{equation}
    I_{i,R}(\vz)=-\int_{\R^d}\EDIT{\vect{f}_i(\vy)}\,\delfrac{\vy_i}\bigl(\chi_R(\vy)G_t(\vz-\vy)\bigr) \wrt{\vy}
    =:I_{i,R}^{\pn{1}}(\vz)+I_{i,R}^{\pn{2}}(\vz),
    \end{equation}
    where
    \begin{equation}
    I_{i,R}^{\pn{1}}(\vz)=-\int_{\R^d} \vect{f}_i(\vy)\chi_R(\vy)\,\delfrac{\vy_i}G_t(\vz-\vy) \wrt{\vy}, \quad
    I_{i,R}^{\pn{2}}(\vz)=-\int_{\R^d} \vect{f}_i(\vy)G_t(\vz-\vy)\,\delfrac{\vy_i}\chi_R(\vy) \wrt{\vy}.
    \end{equation}
    Since $\vect{f}_i\in L^1\pn{\Reals^d;\Reals}$ and $G_t$ is a Gaussian \EDIT{kernel}, we estimate,
    \begin{equation}
    |I_{i,R}^{\pn{2}}(\vz)|\le \tfrac{C}{R}\int_{R\le|\vy|\le 2R}\abs{\vect{f}_i(\vy)}\,|G_t(\vz-\vy)| \wrt{\vy}
    \rightarrow 0, \text{ as } R\rightarrow \infty,
    \end{equation}
    By dominated convergence,
    \begin{equation}
    I_{i,R}^{\pn{1}}(\vz)\rightarrow -\int \vect{f}_i(\vy)\,\delfrac{\vy_i}G_t(\vz-\vy) \wrt{\vy}, \text{ as } R\rightarrow \infty.
    \end{equation}
    since $\abs{\vect{f}_i(\vy)\chi_R(\vy)\,\delfrac{\vy_i}G_t(\vz-\vy)} \leq \abs{\vect{f}_i\pn{\vy}}\norm{\delfrac{\vy_i}G_t\pn{\vz - \vy}}_\infty$.
    Because $I_{i,R}(\vz)\to I_i(\vz)$ as $R\to\infty$, we have
    \begin{equation}
    I_i(\EDIT{\vz})=\int_{\R^d}\spn*{\delfrac{\vy_i}\vect{f}_i(\vy)}\,G_t(\vz-\vy) \wrt{\vy}=-\int \vect{f}_i(\vy)\,\delfrac{\vy_i}G_t(\vz-\vy) \wrt{\vy}.
    \end{equation}
    In addition, since $\delfrac{\vy_i}G_t(\vz-\vy)=-\delfrac{\vz_i}G_t(\vz-\vy)$ and summing over $i$, we have
    \begin{equation}\label{eq:first_term_lemma_6}
    \bigl[\nabla\cdot (p_t\vu_t)\bigr]\ast G_t(\vz)=\nabla\cdot\bigl[(p_t\vu_t)\ast G_t\bigr](\vz).
    \end{equation}
    Since $\nabla\tilde{p}_t=\nabla(p_t\ast G_t)=p_t\ast \nabla G_t$, we have
    \begin{equation}\label{eq:second_term_lemma_6}
    \frac{1}{2}\nabla\cdot\bigl[p_t \ast (\mat{\Lambda}\nabla G_t)\bigr]
    =\frac{1}{2}\nabla\cdot\bigl[\mat{\Lambda} (p_t\ast \nabla G_t)\bigr]
    =\frac{1}{2}\nabla\cdot\bigl[\mat{\Lambda}\nabla(p_t\ast G_t)\bigr]
    =\frac{1}{2}\nabla\cdot\bigl[\mat{\Lambda}\nabla\tilde p_t\bigr].
    \end{equation}

    According to \eqref{eq:partial_t_tilde_pt}, \eqref{eq:first_term_lemma_6} and \eqref{eq:second_term_lemma_6}, we have
    \begin{equation}
        \delfrac{t}\tilde{p}_t = -\nabla\cdot\bigl[(p_t\vu_t)\ast G_t\bigr] + \frac{1}{2}\nabla\cdot\bigl[\mat{\Lambda}\nabla\tilde p_t\bigr]
    \end{equation}
    Therefore
    \begin{equation}
        \delfrac{t} \tilde{p}_t + \nabla\cdot\bigl[(p_t\vu_t)\ast G_t - \frac{1}{2}\mat{\Lambda}\nabla\tilde{p}_t\bigr] = 0.
    \end{equation}
    Since $G_t(\cdot)>0$ everywhere for $t>0$ and $p_t\ge0$ with $\int_{\R^d}p_t=1$, we have for every $\vz\in\R^d$,
    \begin{equation}
    \tilde p_t(\vz)
    = \int_{\R^d} p_t(\vy)\,G_t(\vz-\vy) \wrt{\vy} \;>\;0, \qquad t>0.
    \end{equation}
    Hence we define, for $t>0$,
    \begin{equation}\label{eq:tilde_u_correct}
    \tilde{\vu}_t(\vz)
    := \frac{\bigl((p_t\vu_t)\ast G_t\bigr)(\vz)}{\tilde p_t(\vz)}
    \;-\;\tfrac12\,\mat{\Lambda}\nabla_{\vz}\log\tilde p_t(\vz),
    \end{equation}
    which is smooth and well-defined on $(0,1]\times\R^d$. For $t=0$, we can choose $\tilde{\vu}_0(\vz) = \vu_0(\vz)$. Then \eqref{eq:tilde_cont_equation} holds. As for the initial condition at $t=0$, we have
    \begin{equation}
    \tilde p_0 := p_0\ast \mathcal N(\,\cdot\,|\,\vect{0},\mat{0}) \;=\; p_0\ast\delta_0 \;=\; p_0 \;=\; \rho_0.
    \end{equation}
\end{proof}

\section{Additional Experimental Details}
\label{appendix:experimental-details}

\rev{
The hyperparameters of all the filters are tuned using Optuna \cite{akiba2019optuna} for 50 trials to minimize RMSE on the last 50 DA steps.
EnSF and EnFF use all the ensemble members to approximate the score function or FM vector field.
}

\subsection{Lorenz '63}
\label{appendix:experimental-details:lorenz-63}
\rev{
The Lorenz '63 ODE \cite{lorenz63} is defined for $\vect{x} = [x_1,x_2,x_3]$ as
\begin{align*}
    \dot{\vect{x}} = [
        \sigma(x_2 - x_1),\;
        x_1(\rho - x_3) - x_2,\;
        x_1x_2 - \beta x_3
    ]^\Transpose
\end{align*}
where chaotic behavior occurs when $\sigma = 10$, $\rho = 28$, and $\beta = 8/3$.
}

\subsection{Lorenz '96}
\label{appendix:experimental-details:lorenz-96}

The Lorenz '96 ODE \cite{lorenz96} defines a $d\geq 4$-dimensional dynamical system:
\begin{align*}
    \diffrac[x_i]{t} = \pn{x_{i + 1} - x_{i - 2}}x_{i - 1} - x_i + b
\end{align*}
where $i \in [\,d\,]$, $x_{-1} = x_{d - 1}$, $x_0 = x_d$, $x_{d + 1} = x_1$, and $b \in \mathbb{R}$ is a constant forcing.
We use $b = 8$, inducing chaotic behavior.

\subsection{KS}
\label{appendix:experimental-details:ks}

The 1D KS PDE \cite{kuramoto1978diffusion} is a nonlinear fourth-order PDE given by
\begin{align}
    \delfrac[u]{t} + \delfrac[^2u]{x^2} + \delfrac[^4u]{x^4} + \frac{1}{2}\EDIT{\delfrac[\pn{u^2}]{x}} = 0
\end{align}
We consider periodic boundary conditions on a domain $\spn{0, L}$ with $L = 128\pi$.

We discretize the domain into a grid of size $1{,}024$, with the values of $u$ at grid points forming the state of the DA dynamical system. The initial condition $\vect{u}_0$ is generated by evolving the analytical profile $\vect{u}_0(x) = \cos(2x/L)(1 + \sin(2x/L))$ for $150$ timesteps to reach the chaotic attractor.

\subsection{2D NS}
\label{appendix:experimental-details:ns}

The 2D NS PDE \cite{temam2001navier} is given by
\begin{align*}
    \delfrac[\vu]{t} + \pn{\vu \cdot \nabla}\vu = -\frac{1}{\rho}\nabla p + \nu \nabla^2\vu + \frac{1}{\rho}f\pn{t, \vx}, \quad \nabla \cdot \vu = 0
\end{align*}
We consider periodic boundary conditions on $\spn{0, L}^2$ with $L = 2$, $\rho = 1$, and viscosity $\nu = 10^{-3}$. The forcing term is $f\pn{t, \vx} = [0.05\sin\pn{2\pi \cdot 8x_2 / L}, 0]^\Transpose$, where $\nu \geq 0$ is viscosity, $\rho$ is density, and $f\pn{t, \vx}$ is the external forcing.

We discretize the domain onto a grid, with the components of $\vu$ and pressure at grid points forming the DA state.
The initial condition $\vu_0$ is sampled from a squared exponential Gaussian process via regular Fourier feature approximation \cite{fritz2009application, hensman2018variational}:
\begin{align*}
u_0(\vx) = \vw^\Transpose_1 \vect{\psi}(\vx), \quad v_0(\vx) = \vw^\Transpose_2 \vect{\psi}(\vx)
\end{align*}
where $\vect{\psi}(\vx) = (\cos(\vect{\omega}_1^\Transpose \vx), \ldots, \cos(\vect{\omega}_M^\Transpose \vx), \sin(\vect{\omega}_1^\Transpose \vx), \ldots, \sin(\vect{\omega}_M^\Transpose \vx))$ and $\vw_1, \vw_2 \sim \mathcal{N}(\vect{0}, \mat{S})$ i.i.d., with $\mat{S} = \mathtt{diag}(s(\vect{\omega}_1), \ldots, s(\vect{\omega}_M), s(\vect{\omega}_1), \ldots, s(\vect{\omega}_M))$; here, $s(\vect{\omega}) = e^{-2\pi^2 \ell^2 |\vect{\omega}|^2}$ is the spectral density for the squared-exponential kernel (i.e., the Fourier transform of a stationary kernel) and $\vect{\omega}_m$ denotes regular grid points in the 2D spectral domain. We use $\ell = 0.2$. The pressure is initialized with zeros.

\end{document}